\documentclass[preprint, 1p]{elsarticle}
\usepackage{LaTeX/awjdean_extra_packages}

\journal{Artificial Intelligence}

\begin{document}

\begin{frontmatter}

\title{Algebras of actions in an agent's representations of the world}

\author[1]{Alexander Dean}
\ead{alexander.dean@city.ac.uk}

\author[1]{Eduardo Alonso \corref{cor1}}
\ead{e.alonso@city.ac.uk}
\cortext[cor1]{Corresponding author.}

\author[1]{Esther Mondragón}
\ead{e.mondragon@city.ac.uk}

\affiliation[1]{organization={Artificial Intelligence Research Centre (CitAI), Department of Computer Science,\\ 
City St George's, University of London},
addressline={Northampton Square},
postcode={EC1V 0HB},
city={London},
country={UK}}

\begin{abstract}
    Learning efficient representations allows robust processing of data, data that can then be generalised across different tasks and domains, and it is thus paramount in various areas of Artificial Intelligence, including computer vision, natural language processing and reinforcement learning, among others. Within the context of reinforcement learning, we propose in this paper a mathematical framework to learn representations by extracting the algebra of the transformations of worlds from the perspective of an agent. As a starting point, we use our framework to reproduce representations from the symmetry-based disentangled representation learning (SBDRL) formalism proposed by \cite{Higgins2018} and prove that, although useful, they are restricted to transformations that respond to the properties of algebraic groups. We then generalise two important results of SBDRL –the equivariance condition and the disentangling definition– from only working with group-based symmetry representations to working with representations capturing the transformation properties of worlds for any algebra, using examples common in reinforcement learning and generated by an algorithm that computes their corresponding Cayley tables. Finally, we combine our generalised equivariance condition and our generalised disentangling definition to show that disentangled sub-algebras can each have their own individual equivariance conditions, which can be treated independently, using category theory. In so doing, our framework offers a rich formal tool to represent different types of symmetry transformations in reinforcement learning, extending the scope of previous proposals and providing Artificial Intelligence developers with a sound foundation to implement efficient applications.
\end{abstract}


\end{frontmatter}

\section{Introduction}

Artificial intelligence (AI) has progressed significantly in recent years due to massive increases in available computational power that facilitates the development and training of data-intensive deep learning algorithms \citep{Amodei2018,thompson2020computational}. However, the best-performing learning algorithms often suffer from poor data efficiency and lack the levels of robustness and generalisation that are characteristic of nature-based intelligence \citep{flesch2022orthogonal,Bernardi2020,ito2022compositional,momennejad2020learning,lehnert2020reward,alonso2013associative,kokkola2019double}.
Contrarily, the brain appears to solve complex tasks by learning high-level, low-dimensional \textit{representations} that focus on aspects of the environment that are relevant 
\citep{Niv2019,mack2020ventromedial,jha2023extracting,op2001inferotemporal,shepard1987toward,edelman1997learning}.

In this paper, we propose a formal framework that aims to address how “good” representations can be learned within the context of reinforcement learning. In reinforcement learning, an agent learns an optimal policy by interacting with the environment, the world, by trial an error \citep{sutton2018reinforcement,li2017deep,arulkumaran2017deep,nian2020review}. As a result of the agent’s actions, the state of the environment changes, and the agent receives a numerical signal, a reward. The task of the agent is to learn the sequence of actions that maximises the expected cumulative reward by exploring the environment and exploiting the knowledge so acquired. For this process to be efficient, it is thus crucial that the agent learns “good” representations that guide it through the state-action space. 

But what makes a “good” state representation? Higgins et al (\cite{Higgins2018}) argue that the \textit{symmetries} of the world are important structures that should be present in the representation of that world. The study of symmetries shifts the centre of attention from studying objects directly to studying the transformations of those objects and using the information about these transformations to discover properties about the objects themselves \citep{Higgins2022}. The exploitation of symmetries has led to many successful deep-learning architectures. Examples include convolutional layers \citep{LeCun1995}, which utilise transitional symmetries to outperform humans in image recognition tasks \citep{Dai2021}, and graph neural networks \citep{Battaglia2018}, which utilise the group of permutations. Not only can symmetries provide a useful indicator of what an agent has learned, but incorporating symmetries into learning algorithms regularly reduces the size of the problem space, leading to greater learning efficiency and improved generalisation \citep{Higgins2022}. In fact, it has been shown that a large majority of neural network architectures can be described as stacking layers that deal with different symmetries \citep{Bronstein2021}. The main methods used to integrate symmetries into a representation are to build symmetries into the architecture of learning algorithms \citep{Baek2021,Batzner2022}, use data augmentation that encourages the model to learn symmetries \citep{Chen2020,Kohler2020}, or to adjust the model’s learning objective to encourage the representation to exhibit certain symmetries \citep{burgess2018understanding,Jaderberg2016}.

There are two main types of symmetries that are used in AI: \textit{invariant symmetries}, where a representation does not change when certain transformations are applied to it, and \textit{equivariant symmetries}, where the representation reflects the symmetries of the world. Historically, the learning of representations that are invariant to certain transformations has been a fruitful line of research \citep{Krizhevsky2012,Hu2018,Silver2016,Espeholt2018}. In building these invariant representations, the agent effectively learns to ignore them since such representations remain unaffected by the transformation. It has been suggested that this approach can lead to more narrow intelligence, where an agent becomes good at solving a small set of tasks but struggles with data efficiency and generalisation when tasked with new learning problems \citep{marcus2018deep,Cobbe2019}. Instead of ignoring certain transformations, the equivariant approach attempts to preserve symmetry transformations in the agent’s representations in such a way that they match the symmetry transformations of the world. It has been hypothesised that the equivariant approach is likely to produce representations that can be reused to solve a more diverse range of tasks because no transformations are discarded \cite{Higgins2022}. Equivariant symmetry approaches are commonly linked with \textit{disentangling representations} \citep{Bengio2013}, in which the agent’s representation is separated into subspaces that are invariant to different transformations. Disentangled representation learning, which aims to produce representations that separate the underlying structure of the world into disjoint parts, has been shown to improve the data efficiency of learning \citep{raffin2019decoupling,wang2022disentangled}.

Inspired by their use in physics, symmetry-based disentangled representations (SBDRs) were proposed as a formal mathematical definition of disentangled representations \citep{Higgins2018}. SBDRs are built on the assumption that the symmetries of a state of the world display important aspects of that world that need to be preserved in an agent’s internal representation (i.e., the symmetries that are present in the world state should also be present in the agent’s internal representation state). Higgins et al describe symmetries of the world state as \textit{“transformations that change only some properties of the underlying world state, while leaving all other properties invariant”} \cite [page 1]{Higgins2018}. For example, the \textit{y}-coordinate of an agent moving parallel to the \textit{x}-axis on a 2D Euclidean plane does not change. Symmetry-based disentangled representation learning (SBDRL) has gained traction in AI in recent years \citep{Park2022learning,Quessard2020learning,Miyato2022unsupervised,Wang2022surprising,Keurti2023homomorphism,Zhu2021commutative,Wang2021self,Pfau2020disentangling,Mercatali2022,Marchetti2023}. However, SBDRL only considers actions that form groups and so cannot take into account, for example, irreversible actions \citep{Higgins2018}. Besides, \citep{caselles2019symmetry} showed that a symmetry-based representation cannot be learned using only a training set composed of the observations of world states; instead, a training set composed of transitions between world states, as well as the observations of the world states, is required. In other words, SBDRL requires the agent to interact with the world. 

We agree with \citep{Higgins2018} that symmetry transformations are important structures to include in an agent’s representation, but take their work one step further: along with \citep{caselles2019symmetry} we posit that the relationships of transformations of the world due to the actions of the agent should be included in the agent’s representations of the world, but, unlike \citep{caselles2019symmetry}, we do not constraint such transformations to symmetries that rely on actions to form algebraic groups exclusively. In this paper, we show that an agent’s representation of a world would lose important information if only transformations of the actions of an agent that form groups are included, and demonstrate that there exist features that cause transformations that do not form group structures. We believe that considering these transformations in the agent’s representation of the world has the potential to build powerful learning mechanisms.

In short, we aim to help answer the question of which features should be present in a ‘good’ representation by, as suggested by \citep{Higgins2018}, looking at the transformation properties of worlds. However, while \citep{Higgins2018} only considered a limited type of symmetry transformations, those that are formalised as groups, we aim to go further and consider the full algebra of world transitions. We propose a mathematical framework to describe the transformations of the world, thereby formally characterising the features we expect to find in the representations of an artificial agent. More specifically, our contributions are as follows:

\begin{enumerate}
    \item We propose a general mathematical framework for describing the transformation of worlds due to the actions of an agent that can interact with the world.
    \item We derive the SBDRs proposed by \citep{Higgins2018} from our framework and, in doing so, identify the limitations of SBDRs in their current form.
    \item We use our framework to explore the structure of the transformations of worlds for classes of worlds containing features found in common reinforcement learning scenarios, and that go beyond symmetry groups and SBDRL. We also present the algorithm used to generate the algebra of the transformations of the world due to the actions of an agent.
    \item We generalise the equivariance condition and the definition of disentangling given by \citep{Higgins2018} to worlds that do not satisfy the conditions for SBDRs. This generalisation is performed using category theory.
\end{enumerate}

It should be emphasised that our framework is about learning representations of an agent interacting with their environment, independently of the reinforcement learning algorithm used in the process (say Q-learning, SARSA, PPO or any deep reinforcement learning algorithm). Also, the paper is formal in that it establishes a mathematical framework for representation learning. It responds to the need to formulate sound formalisms from which specific applications may follow. That is, our work is foundational.

The rest of the paper is structured as follows: In Section 2, we define our framework and then describe how it deals with generalised worlds, which consist of distinct world states connected by transitions that represent the dynamics of the world as a result of the actions of an agent. In Section 3, our framework is used to reproduce SBDRL. This is achieved by defining an equivalence relation that makes the actions of an agent equivalent if the actions produce the same outcome if performed while the world is in any given state. In Section 4, we apply our framework to worlds exhibiting common reinforcement learning scenarios that cannot be described fully using SBDRs and study the algebraic structures exhibited by the dynamics of these worlds. In Section 5, we generalise, using category theory, two important results of \citep{Higgins2018} –the equivariance condition and the disentangling definition– to worlds with transformations whose algebras do not fit into the SBDRL paradigm. We shall finish with conclusions and a discussion in Section \ref{sec:discussion and conclusion}.

\section{A mathematical framework for an agent in the world}\label{sec:Mathematical framework for an agent in an environment}

In this section, we introduce our general mathematical framework for formally describing the transformations of a world. First, we define world states and transitions. We then consider how agents are understood within our framework and how their actions are formalised as labelled transitions.

\subsection{Model of the world}

For the sake of simplicity, we consider a fully observable world consisting of a set of discrete states, which are distinguishable in some way, and a set of transitions between those states; these transitions convey the world dynamics (i.e., how the world can transform from one world state to another). This world can be represented by a directed multigraph, where the world states are the vertices of the graph and the transitions between states are arrows between the vertices. We will use this framework to reproduce the group action structure of the evolution of a world as an agent’s representation of it, as described by \citep{Higgins2018}. In so doing, we uncover the requirements for this group action structure to be present in the world.

\subsubsection{World states and world state transitions}

We believe that defining a world as a discrete set of world states with world state transitions between them is the most general definition of a world. Therefore, we take it as our starting point to define the algebra of the actions of an agent. 

\paragraph{Transitions}
We consider a directed multigraph $\mathscr{W} = (W, \hat{D}, s, t)$ where $W$ is a set of \textit{world states}, $\hat{D}$ is a set of \textit{minimum world state transitions}, and $s,t: \hat{D} \to W$; $s$ is called the \textit{source map} and $t$ is called the \textit{target map}.
For the remainder of the paper, we fix such a $(W, \hat{D}, s, t)$.
$\mathscr{W}$ is called a \textit{world}.

Minimum world state transitions are extended into a set $D$ of paths called \textit{world state transitions}: a path is a sequence of minimum world state transitions $d = \hat{d}_{n} \circ \hat{d}_{n-1} \circ ... \circ \hat{d}_{1}$ such that $t(\hat{d}_{i}) = s(\hat{d}_{i+1})$ for $i = 1, ..., n-1$.
We extend $s, t$ to $D$ as $s(d) = s(\hat{d}_{1})$, and $t(d) = t(\hat{d}_{n})$.
We also extend the composition operator $\circ$ to $D$ such that $d_{n} \circ d_{n-1} \circ ... \circ d_{1}$ is defined if $t({d}_{i}) = s({d}_{i+1})$ for $i = 1, ..., n-1$.
For $d \in D$ with $s(d) = w$ and $t(d) = w'$, we will often denote $d$ by $d: w \to w'$.

For the rest of the paper, we assume that, for each world state $w \in W$, there is a unique trivial world state transition $1_{w} \in \hat{D}$ with $s(1_{w}) = t(1_{w})$; the trivial transition $1_{w}$ is associated with the world being in state $w$ and then no change occurring due to the transition $1_{w}$.

\paragraph{Connected and disconnected worlds}
We now introduce connected and disconnected worlds.
Simply, a world $A$ is connected to a world $B$ if there is a transition from a world state in world $A$ to a world state in world $B$.
The concepts of connected and disconnected worlds are necessary for generality; we are only interested in the perspective of the agent and so only care about the world states and transitions that the agent can come into contact with.
Connected and disconnected worlds give us the language to describe and then disregard the parts of worlds that the agent will never explore and therefore are not relevant to the agent's representation.
For example, if an agent is in a maze and a section of the maze is inaccessible from the position that the agent is in, then that section of the maze would be disconnected from the section of the maze that the agent is in; if we want to study how the agent's representation evolves as it learns, it makes sense to disregard the disconnected section of the maze since the agent never comes into contact with it and so the disconnected section of the maze will not affect the agent's representation.

Formally, we first define a \textit{sub-world} $W'$ of a world $W$ as a subset $W' \subseteq W$ along with $D' = \{d \in D \mid s(d) \in W'\text{ and }t(d) \in W'\}$.
Note that a sub-world is a world.
A sub-world $W$ is \textit{connected to} a sub-world $W'$ if there exists a transition $d: w \to w'$ where $w \in W$ and $w' \in W'$; if no such transition exists, then $W$ is \textit{disconnected from} $W'$.
Similarly, a world state $w$ is \textit{connected to} a sub-world $W$ if there exists a transition $d: w \to w'$ where $w' \in W'$; if no such transition exists, then $w$ is \textit{disconnected from} $W'$.

\paragraph{Effect of transitions on world states}
We define $*$ as a partial function $D \times W \to W$ by $d * w = w'$ where $d: w \to w'$ and undefined otherwise.

\subsubsection{Example}

We consider a cyclical $2\times 2$ grid world, denoted by $\mathscr{W}_{c}$, containing an agent as shown in Figure \ref{fig:2x2-cyclical-grid-world-states}.
The transformations of $\mathscr{W}_{c}$ are due to an agent moving either up ($U$), down ($D$), left ($L$), right ($R$), or doing nothing ($1$).
The possible world states of $\mathscr{W}_{c}$ are shown in Figure \ref{fig:2x2-cyclical-grid-world-states}.
$\mathscr{W}_{c}$, and variations of it, is used as a running example to illustrate the concepts presented in this paper.

\begin{figure}
    \centering
    \begin{subfigure}[b]{0.45\linewidth}
        \centering
        \includegraphics[width=0.5\linewidth]{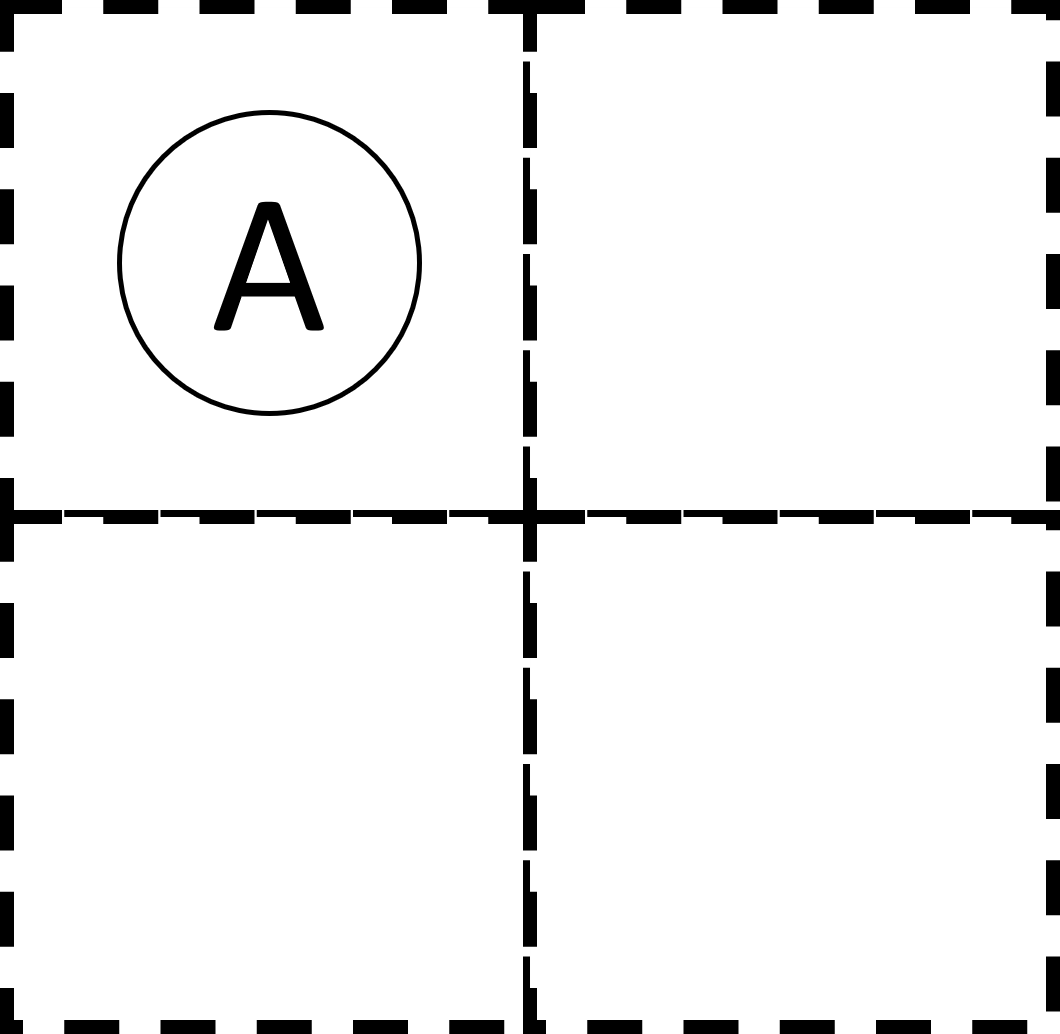}
        \caption{$w_{0}$}
        \vspace{0.25cm}
    \end{subfigure}
    \begin{subfigure}[b]{0.45\linewidth}
        \centering
        \includegraphics[width=0.5\linewidth]{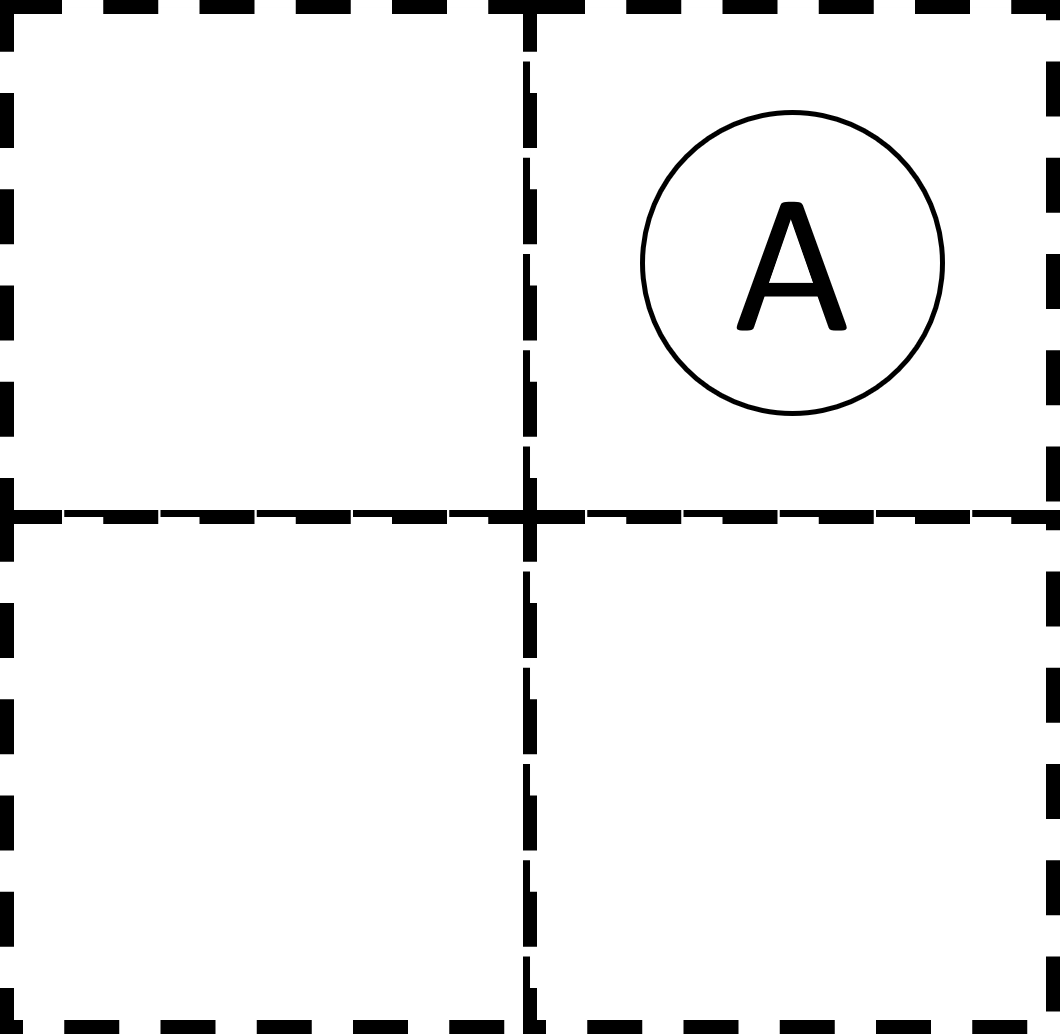}
        \caption{$w_{1}$}
        \vspace{0.25cm}
    \end{subfigure}
    \begin{subfigure}[b]{0.45\linewidth}
        \centering
        \includegraphics[width=0.5\linewidth]{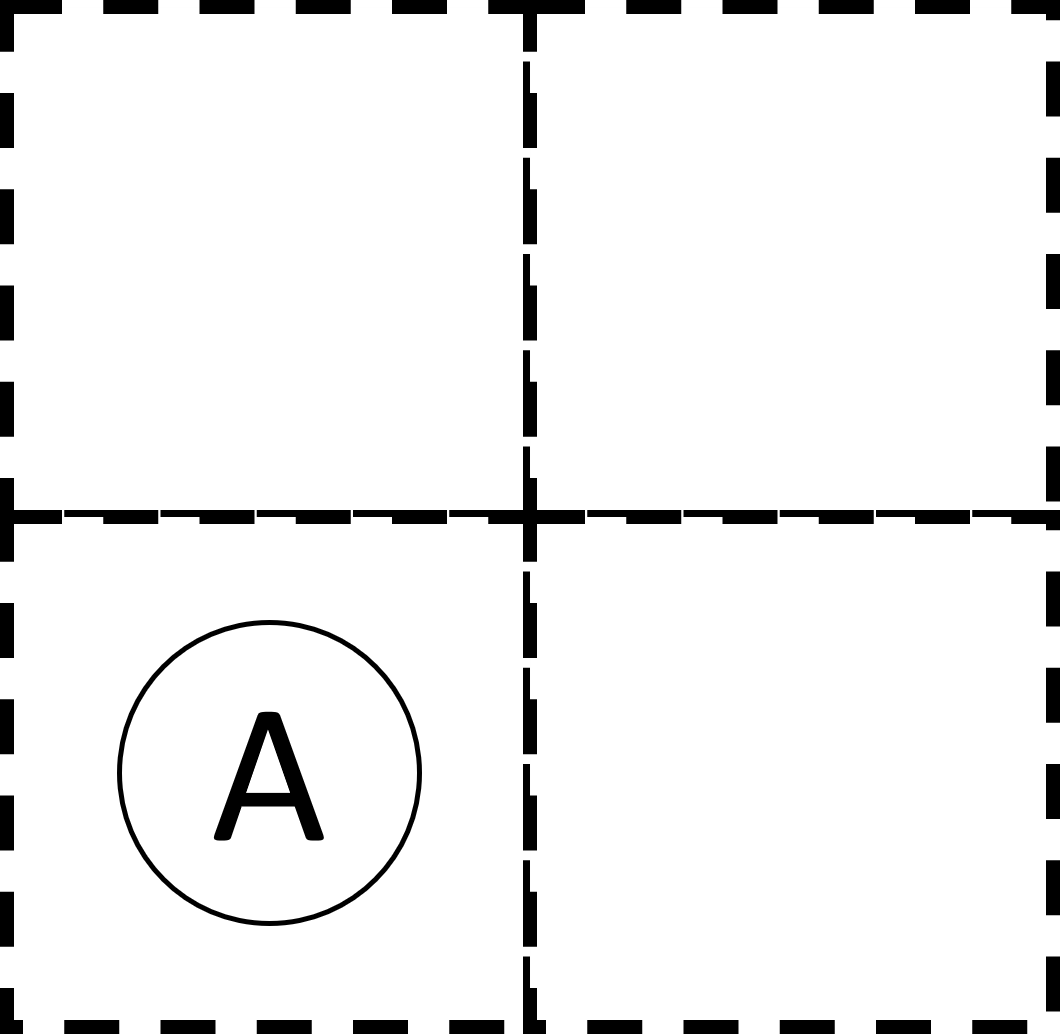}
        \caption{$w_{2}$}
    \end{subfigure}
    \begin{subfigure}[b]{0.45\linewidth}
        \centering
        \includegraphics[width=0.5\linewidth]{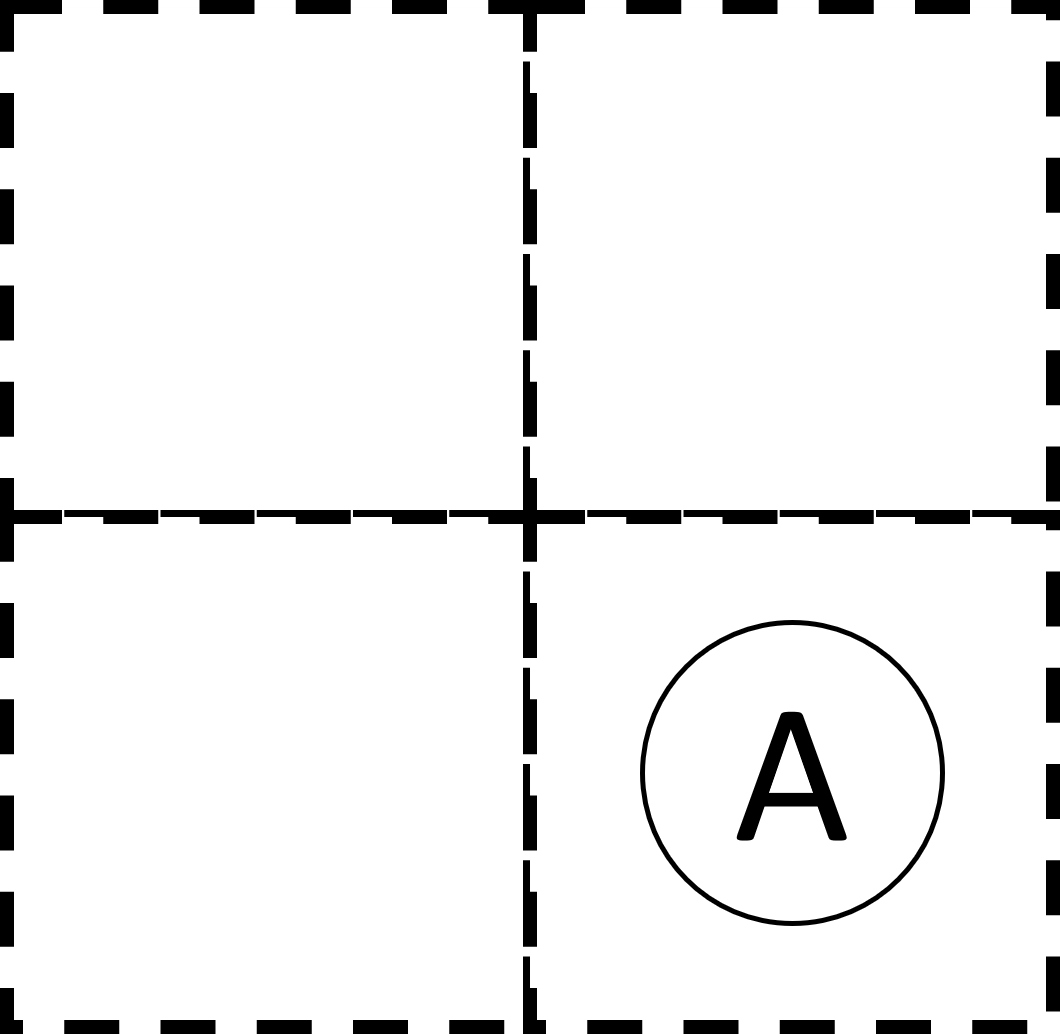}
        \caption{$w_{3}$}
    \end{subfigure}
    \caption{The world states of a cyclical $2\times 2$ grid world $W_{c}$, where changes to the world are due to an agent moving either up, down, left, or right. The position of the agent in the world is represented by the position of the circled A.}
    \label{fig:2x2-cyclical-grid-world-states}
\end{figure}

We say the world being cyclical means that if the agent performs the same action enough times, then the agent will return to its starting position; for example, for the world $\mathscr{W}_{c}$ if the agent performs the action $U$ twice when the world is in state $w_{0}$ in Figure \ref{fig:2x2-cyclical-grid-world-states} then the world will transition into the state $w_{0}$ (i.e., $U^2 * w_{0} = w_0$).
The transition due to performing each action in each state can be found in Table \ref{tab:2x2-gridworld-minimum-transitions}.

\begin{table}[H]
    \centering
    \begin{tabular}{c|c c c c c}
                &  $1$      & $U$       & $D$       & $L$       & $R$\\
         \hline
        $w_{0}$ & $w_{0}$   & $w_{2}$   & $w_{2}$   & $w_{1}$   & $w_{1}$\\
        $w_{1}$ & $w_{1}$   & $w_{3}$   & $w_{3}$   & $w_{0}$   & $w_{0}$\\
        $w_{2}$ & $w_{2}$   & $w_{0}$   & $w_{0}$   & $w_{3}$   & $w_{3}$\\
        $w_{3}$ & $w_{3}$   & $w_{1}$   & $w_{1}$   & $w_{2}$   & $w_{2}$\\
    \end{tabular}
    \caption{Each entry in this table shows the outcome state of the agent performing the action given in the column label when in the world state given by the row label.}
    \label{tab:2x2-gridworld-minimum-transitions}
\end{table}

The transitions shown in Table \ref{tab:2x2-gridworld-minimum-transitions} can be represented as the transition diagram given in Figure \ref{fig:2x2-cyclical-min-trans}.
It should be noted that, since the structure of the diagram is wholly dependent on the arrows between the world states, the positioning of the world states is an arbitrary choice.

\begin{figure}
    \centering
    \includegraphics[width=0.5\linewidth]{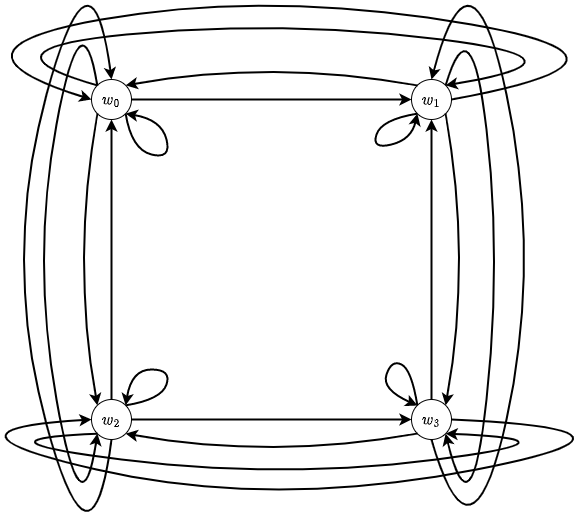}
    \caption{A transition diagram for the transitions shown in Table \ref{tab:2x2-gridworld-minimum-transitions}.}
    \label{fig:2x2-cyclical-min-trans}
\end{figure}

\subsection{Agents}

We consider worlds containing an embodied agent that is able to interact with the environment by performing actions. The end goal of the agent’s learning process is to map the useful aspects of the structure of the world to the structure of its representations; the useful aspects are those that enable the agent to complete whatever task it is programmed to achieve (e.g., in reinforcement learning, finding an optimal policy).

We use the treatment of agents adopted by \citep{Higgins2018}. The agent has an unspecified number of sensors that allow it to make observations of the state of the environment. Information about the world state that the agent is currently in is delivered to the internal state of the agent, its representation of the (state of the) world. Mathematically, the process of information propagating through the sensory states is a mapping $b : W \to O$ (the `observation process'), which produces a set of observations $o_1$ containing a single observation for each sensory state.
These observations are then used by an inference process $h : O \to Z$ to produce an internal representation.
The agent then uses some internal mechanism to select an action to perform.

It is important to note that the agent’s state representation only reflects the observations the agent makes with its sensors; in other words, the agent’s internal state is built using the information about aspects of the world state propagated through its sensory states, and not directly from the world state.

\subsubsection{Actions of an agent as labelled transitions}\label{sec:Agent actions as labelled transitions}

Consider a set $\hat{A}$ called the set of \textit{minimum actions}.
Let the set $A$ be the set of all finite sequences formed from the elements of the set $\hat{A}$; we call $A$ the set of \textit{actions}.
Consider a set $\hat{D}_{A} \subset D$, where $1_{w} \in \hat{D}_{A}$ for all $w \in W$; we call $\hat{D}_{A}$ the set of \textit{minimum action transitions}.
We consider a labelling map $\hat{l}: \hat{D}_{A} \to \hat{A}$ such that:

\begin{enumerate}
    \item Any two distinct transitions leaving the same world state are labelled with different actions.
    \begin{action_condition}\label{actcon:action-gives-single-outcome}
        For any $d,d' \in \hat{D}_{A}$ with $d\neq d'$ and $s(d)=s(d')$, $\hat{l}(d) \neq \hat{l}(d')$.
    \end{action_condition}

    \item There is an identity action that leaves any world state unchanged.
    \begin{action_condition}\label{actcon:identity-action}
        There exists an action $1 \in \hat{A}$ such that $\hat{l}(1_{w})=1$ for all $w \in W$.
        We call $1$ the \textit{identity action}.
    \end{action_condition}
    
\end{enumerate}

Given $\hat{D}_{A}$ as defined above and satisfying  action condition \ref{actcon:action-gives-single-outcome} and  action condition \ref{actcon:identity-action}, we define $Q_{D_{A}} = (W, \hat{D}_{A}, s_{A}, t_{A})$, where $s_{A}, t_{A}$ are the restrictions of $s,t$ to the set $\hat{D}_{A}$.
We now define a set $D_{A}$, the set of \textit{action transitions}, which is the set of all paths of $Q_{D_{A}}$.

We extend the map $\hat{l}$ to a map $l: D_{A} \to A$ such that if $d = \hat{d}_{n} \circ ... \circ \hat{d}_{1}$ then $l(d) = \hat{l}(\hat{d}_{n}) ... \hat{l}(\hat{d}_{1})$.
For $d \in D_{A}$ with $s(d) = w$, $t(d) = w'$ and $l(d) = a$, we will often denote $d$ by $d: w \xrightarrow{a} w'$.

If an action $a \in A$ is expressed in terms of its minimum actions as $a = \hat{a}_{n} \circ ... \circ \hat{a}_{1}$, then $a = l(d) = l(\hat{d}_{n} \circ ... \circ \hat{d}_{1}) = \hat{l}(\hat{d}_{n}) \circ ... \circ \hat{l}(\hat{d}_{1}) = \hat{a}_{n} \circ ... \circ \hat{a}_{1}$, where the $\hat{a}_{i}$ are called \textit{minimum actions}.

\begin{remark}
    For a given $w \in W$, we can label transitions in $D_{A}$ with an appropriate element of $A$ through the following: for each $d \in D_{A}$ with $s(d)=w$, express $d$ in terms of its minimum transitions in $D_{A}$ as $d = d_{n} \circ ... \circ d_{2} \circ d_{1}$; if $\hat{l}(d_{i}) = a_{i}$ then $d$ is labelled with $a_{n}...a_{2}a_{1} \in A$.
    We denote the map that performs this labelling by $l: D_{A} \to A$.
\end{remark}

Figure \ref{fig:2x2-cyclical-min-actions-standard} shows how transitions are labelled with actions in our $2 \times 2$ cyclical world example.
We only show the minimum actions for simplicity but there are actually infinite action transitions between each pair of world states; for example, the action transitions from $w_{0}$ to $w_{1}$ include those labelled by: $D \circ R$, $D \circ R \circ 1^{n}$ ($n \in \mathbb{N}$), $1^{n} \circ D \circ R$ ($n \in \mathbb{N}$), $D \circ R \circ (L \circ R)^{n}$ ($n \in \mathbb{N}$) etc...

\begin{figure}
    \centering
    \includegraphics[width=0.5\linewidth]{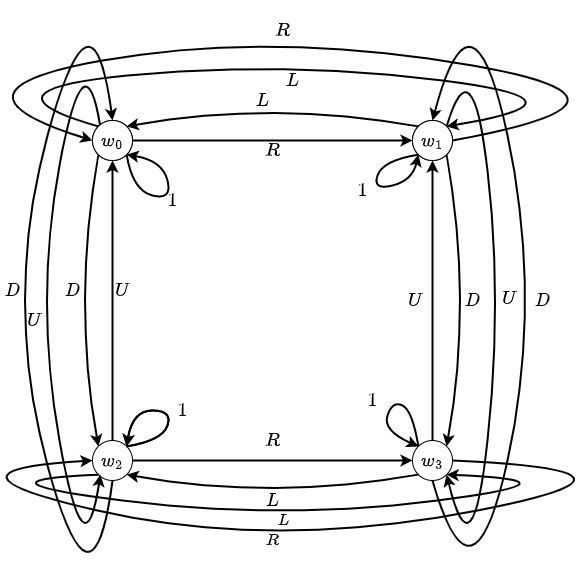}
    \caption{Labelling the transitions in Figure \ref{fig:2x2-cyclical-min-trans} with the relevant actions in $A$.}
    \label{fig:2x2-cyclical-min-actions-standard}
\end{figure}

\paragraph{Effect of actions on world states}
We define the effect of the action $a \in A$ on world state $w \in W$ as the following: if there exists $d \in D_{A}$ such that $s(d)=w$ and $l(d)=a$, then $a * w = t(d)$; if there does not exist $d \in D_{A}$ such that $s(d)=w$ and $l(d)=a$, then we say that $a * w$ is \textit{undefined}.
The effect of actions on world states is well-defined due to action condition \ref{actcon:action-gives-single-outcome}.
We can apply the minimum actions that make up an action to world states individually: if $a * w$ is defined and $a = \hat{a}_{k}...\hat{a}_{1}$ then $a * w = (\hat{a}_{k}...\hat{a}_{1}) * w = \hat{a}_{k}...\hat{a}_{2} * (\hat{a}_{1} * w)$.
Physically, the identity action $1 \in A$ corresponds to the no-op action (\textit{i.e.}, the world state does not change due to this action).

\paragraph{Actions as (partial) functions}
Consider all the transitions that are labelled by a particular action $a \in A$.
Together these transitions form a partial function $f_{a}: W \to W$ because for any $w \in W$ either $a * w$ is undefined or $a * w$ is defined and there is a unique world state $w' \in W$ for which $a * w = w'$ (due to condition 1).
$f_{a}$ is not generally surjective because for a given $w \in W$ there is not necessarily a transition $d \in D$ with $l(d) = a$ and $t(d) = w$.
$f_{a}$ is not generally injective because it is possible to have an environment where $f_{a}(w)=f_{a}(w')$ for some $w \in W$ different from $w' \in W$.
We can also reproduce these functions using the formalism given by \cite{caselles2019symmetry}, which describes the dynamics of the world in terms of a multivariate function $f: A \times W \to W$.
If we let $f: A \times W \to W$ be the dynamics of the environment then the transition caused by an action $a \in A$ on a world state $w \in W$ (where $a * w$ is defined) is given by $(a,w) \mapsto f(a,w) = a * w$.
Mathematically, we curry the function $f: A \times W \to W$ to give a collection $\{f_{a}\}$ of partial functions with a partial function $\hat{f}(a)=f_{a}: W \to W$ for each action $a \in A$ as $\textit{Curry}: (f: A \times W \to W) \to (f_{a}: W \to W)$.

\medskip
Once we have introduced our framework to formalise word states and their transitions through the agent’s intervention, we are investigating next its representational power by showing that it can generate the same transitions as SBDR, that is, those that form symmetry groups over disentangled representations, in Section \ref{sec:Reproducing SBDRL}, and other types of representations, that do not fit in the constraints imposed by SBDR, and that are standard in reinforcement learning scenarios in Section \ref{sec:Beyond SBDRs}.

\section{Reproducing SBDRL}\label{sec:Reproducing SBDRL}

 We now use the framework set out in the previous section to reproduce SBDRL and illustrate it using worlds that are similar to those given by \cite{Higgins2018} and \cite{caselles2019symmetry}. We choose to begin by reproducing symmetry-based representations because (1) symmetry-based representations describe transformations of the world that form relatively simple and well-understood algebraic structures (groups), (2) groups, and the symmetries they describe, are gaining increasing prominence in Artificial Intelligence research, (3) it shows how our framework encompasses previous work in formalising the structure of transformations of a world, and (4) it provides a more rigorous description of SBDRL, which should aid future analysis and development of the concept.

Section \ref{sec:Symmetry-based disentangled representation learning} provides a description of SBDRL, and Section \ref{sec:SBDRL through equivalence} shows how to obtain SBDRL using an equivalence relation on the actions of the agent.
Section \ref{sec:Algorithmic exploration of world structures} details the algorithmic exploration of world structures performed on example worlds and goes through a worked example.
Finally, Section \ref{sec:World conditions} shows the conditions of the world that are required for the actions of an agent to be fully described by SBDRs.

\subsection{Symmetry-based disentangled representation learning}\label{sec:Symmetry-based disentangled representation learning}

We proceed to present a more detailed description of the SBDRL formalism, how to extract representations from world states, and how to build symmetry-based and symmetry-based disentangled representations, before stating their limitations and thus, of SBDRL.

\paragraph{From world states to representation states}
The world state is an element of a set $W$ of all possible world states.
The observations of a particular world state made by the agent's sensors are elements of the set $O$ of all possible observations.
The agent's internal state representation of the world state is an element of a set $Z$ of all possible internal state representations.
There exists a composite mapping $f = h \circ b: W \to Z$ that maps world states to states of the agent's representation ($w \mapsto z$); this composite mapping is made up of the mapping of an observation process $b: W \to O$ that maps world states to observations ($w \mapsto o$) and the mapping of an inference process $h: O \to Z$ that maps observations to the agent's internal state representation ($o \mapsto z$) (see Figure \ref{fig:observation-maps}).

\begin{figure}[t]
    \centering
    \includegraphics[scale = 0.35]{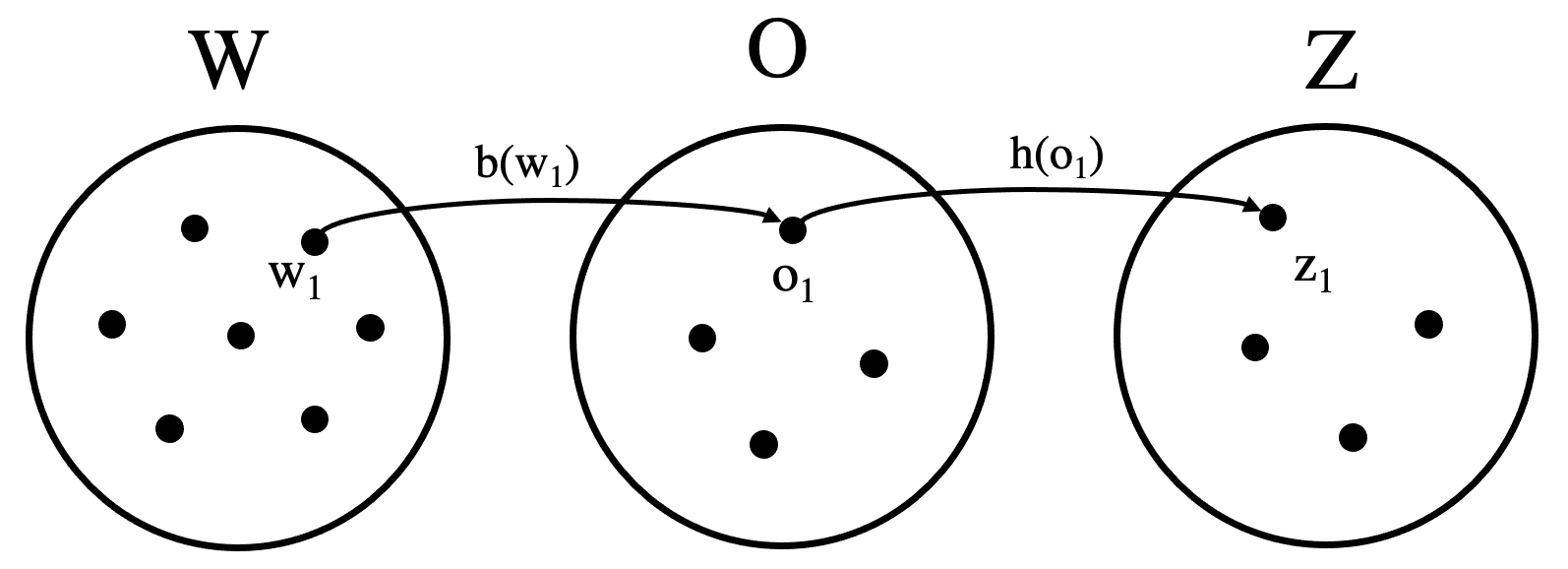}
    \caption{The composite mapping from the set $W$ of world states to the set $Z$ of state representations via the set $O$ of observations.}
    \label{fig:observation-maps}
\end{figure}

Given that the symmetries represented by SBDRL form algebraic groups, we first define the notions of groups and group action formally, as follows:

\paragraph{Groups and symmetries}
\begin{definition}[Group]
    A group $G$ is a set with a binary operation $G \times G \to G$, $(g, g') \mapsto g \circ g'$ called the \textit{composition} of group elements that satisfies the following properties:
\begin{enumerate}
    \item \textit{Closure.}
    $g \circ g'$ is defined for all $g, g' \in G$.
    \item \textit{Associative.}
    $(g \circ g') \circ g'' = g \circ (g' \circ g'')$ for all $g, g', g'' \in G$.
    \item \textit{Identity.}
    There exists a unique identity element $1 \in G$ such that $1 \circ g = g \circ 1 = g$ for all $g \in G$.
    \item \textit{Inverse.}
    For any $g \in G$, there exists $g^{-1} \in G$ such that $g \circ g^{-1} = g^{-1} \circ g = 1$.
\end{enumerate}
\end{definition}

Applying symmetries to objects is mathematically defined as a \textit{group action}.

\begin{definition}[Group action]
    Given a group $G$ and a set $X$, a group action of $G$ on $X$ is a map $G \times X \to X$, $(g,x) \mapsto g * x$ that satisfies the following properties:
\begin{enumerate}
    \item \textit{Compatibility with composition.}
    The composition of group elements and the group action are compatible: $g' \circ (g * x) = (g' \circ g) * x$ for $g,g' \in G$ and $x \in X$.
    \item \textit{Identity.}
    The group identity $1 \in G$ leaves the elements of $X$ unchanged: $1 * x = x$ for all $x \in X$.
\end{enumerate}
\end{definition}

Another important property of groups is commutation.
Two elements of a group \textit{commute} if the order they are composed does not matter: $g \circ g' = g' \circ g$.
If all elements in a group commute with each other then the group is called \textit{commutative}.
Subgroups of a group might commute with each other.

\paragraph{Symmetry-based representations}
The set $W$ of world states has a set of symmetries that are described by the group $G$.
This group $G$ acts on the set $W$ of world states via a group action $\cdot_{W}: G \times W \to W$.
For the agent's representations $z_i \in Z$ to be symmetry-based representations, a corresponding group action $\cdot_{Z}: G \times Z \to Z$ must be found so that the symmetries of the agent's representations reflect the symmetries of the world states.
The mathematical condition for this is that, for all $w \in W$ and all $g \in G$, applying the action $g \cdot_W$ to $w$ and then applying the mapping $f$ gives the same result as first applying the mapping $f$ to $w$ to give $f(w)$ and then applying the action $g \cdot_Z$ to $f(w)$.
Mathematically, this is $f(g \cdot_W w) = g \cdot_Z f(w)$.
If this condition is satisfied, then $f$ is called a \textit{group-equivariant map}.

\paragraph{Symmetry-based disentangled representations}
To go from symmetry-based representations to symmetry-based disentangled representations, suppose the group of symmetries $G$ of the set $W$ of world states decomposes as a direct product $G = G_1 \times \hdots \times G_i \times \hdots \times G_n$.
The group action $\cdot_Z : G \times Z \to Z$ and the set $Z$ are disentangled with respect to the decomposition of $G$, if there is a decomposition $Z = Z_1 \times \hdots \times Z_i \times \hdots \times Z_n$ and actions $\cdot_{Z_i}: G_i \times Z_i \to Z_i, i \in \{1, \hdots, n\}$ such that $(g_{G_1}, g_{G_2}) \cdot_Z (z_{Z_1}, z_{Z_2}) = (g_{G_1} \cdot_{Z_1} z_{Z_1}, g_{G_2} \cdot_{Z_2} z_{Z_2})$, where $g_{G_i} \in G_i$ and $z_{Z_i} \in Z_i$.
In other words, each subspace $Z_i$ is invariant to the action of all the $G_{j \neq i}$ and only affected by $G_i$.

The representations in $Z$ are symmetry-based disentangled with respect to the decomposition $G = G_1 \times \hdots \times G_i \times \hdots \times G_n$, where each $G_i$ acts on a disjoint part of $Z$, if:
\begin{enumerate}
    \item There exists a group action $\cdot_{W}: G \times W \to W$ and a corresponding group action $\cdot_{Z}: G \times Z \to Z$;
    \item The map $f : W \to Z$ is group-equivariant between the group actions on $W$ and $Z$: $g \cdot_{Z} f(w) = f(g \cdot_{W} w)$. In other words, the diagram
\[\begin{tikzcd}
	w && {g \cdot_{W} w} \\
	\\
	{f(w)} && {g \cdot_{Z} f(w) = f(g \cdot_{W} w)}
	\arrow["{g \cdot_{W}}", from=1-1, to=1-3]
	\arrow["f"', from=1-1, to=3-1]
	\arrow["f"', from=1-3, to=3-3]
	\arrow["{g \cdot_{Z}}", from=3-1, to=3-3]
\end{tikzcd}\]
    commutes.    

    \item There exists a decomposition of the representation $Z = Z_1 \times \hdots \times Z_n$ such that each subspace $Z_i$ is unaffected by the action for all $G_{j \neq i}$ and is only affected by $G_i$.
\end{enumerate}

\paragraph{Limitations of SBDRL}
Both \cite{Higgins2018} and \cite{caselles2019symmetry} suggest that these group actions can be used to describe some types of real-world actions.
However, it is important to note that they do not believe that all actions can be described by their formalism: \textit{``It is important to mention that not all actions are symmetries, for instance, the action of eating a collectible item in the environment is not part of any group of symmetries of the environment because it might be irreversible.''} \cite[page 4]{caselles2019symmetry}.

\subsection{SBDRL through equivalence}\label{sec:SBDRL through equivalence}

For the algebra of the actions of our agent to form a group, we need some sense of actions being the same so that the algebra can satisfy the group properties (\textit{e.g.}, for the identity property we need an element $1$ in the algebra $A$ such that $1a = a1 = a$ for any $a \in A$).
We define an equivalence relation on the elements of $A$ that says two actions are equivalent (our sense of the actions being the same) if they lead to the same end world state when performed in any initial world state.
This equivalence relation is based on our mathematical interpretation of the implication given by \cite{Higgins2018} that transformations of the world are the same if they have the same effect, which is used to achieve the group structure for SBDRL.
Our use of an equivalence relation was inspired by \cite{caselles2020sensory}, which uses a similar equivalence relation to equate action sequences that cause the same final observation state after each action sequence is performed from an initial observation state.
We then derive some properties of the equivalence classes created by $\sim$ that will be used to show that the actions of an agent form the group action described by \cite{Higgins2018} under the equivalence relations we define and for worlds satisfying certain conditions.

\medskip
\begin{definition}[Equivalence of actions under $\sim$]
    Given two actions $a, a' \in A$, we denote $a \sim a'$ if $a * w = a' * w$ for all $w \in W$.

\end{definition}
\medskip
\begin{remark}
    If $a \sim a'$, then either for each $w \in W$ (1) there exists transitions $d: w \xrightarrow{a} t(d)$ and $d': w \xrightarrow{a'} t(d)$ or (2) there exists no transitions $d: w \xrightarrow{a} t(d)$ or $d': w \xrightarrow{a'} t(d)$.
\end{remark}
\medskip
\begin{proposition}
    $\sim$ is an equivalence relation.
\end{proposition}

\begin{proof} We demonstrate the three properties of an equivalence relation, namely, reflexive, transitive, and symmetric.\hfill

    \textbf{Reflexive.}
    If $a \sim a'$ then $a * w = a' * w$ for all $w \in W$.

    \textbf{Transitive.}
    If $a \sim a'$ and $a' \sim a''$, then $a * w = a' * w$ for all $w \in W$ and $a' * w = a'' * w$ for all $w \in W$.
    Therefore, $a * w = a'' * w$ for all $w \in W$ and so $a \sim a''$.

    \textbf{Symmetric.}
    If $a \sim a'$, then $a * w = a' * w$ for all $w \in W$.
    Therefore $a' * w = a * w$ for all $w \in W$, and so $a' \sim a$.
\end{proof}

Figure \ref{fig:2x2-cyclical-min-act-equivalence} shows the effect of applying the equivalence relations to our $2 \times 2$ cyclical example world $\mathscr{W}_{c}$.

\begin{figure}
    \centering
    \includegraphics[width=0.5\linewidth]{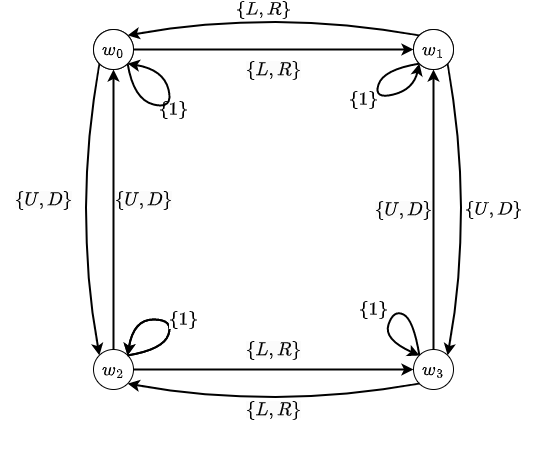}
    \caption{Action equivalence classes in $A/\sim$ for the actions show in Figure \ref{fig:2x2-cyclical-min-actions-standard}.}
    \label{fig:2x2-cyclical-min-act-equivalence}
\end{figure}

We define the canonical projection map $\pi_{A}: A \to A/\sim$ that sends actions in $A$ to their equivalence classes under $\sim$ in the set $A/\sim$.
We denote the equivalence class of $a$ by $[a]_{\sim}$.
Sometimes we will drop the $[a]_{\sim}$ in favour of $a \in A/\sim$ for ease.

\paragraph{Composition of actions}
We define the composition of elements in $A/\sim$ as $\circ: (A/\sim) \times (A/\sim) \to (A/\sim)$ such that $[a']_{\sim} \circ [a]_{\sim} = [a' \circ a]_{\sim}$ for $a,a' \in A$.
\medskip
\begin{proposition}
    $[a']_{\sim} \circ [a]_{\sim} = [a' \circ a]_{\sim}$ is well-defined for all $a, a' \in A$.
\end{proposition}
\begin{proof}
    We need to show that the choice of $a,a'$ doesn't matter: if $a_{1} \sim a_{2}$ and $a_{3} \sim a_{4}$ for $a_{1}, a_{2}, a_{3}, a_{4} \in A$, then $[a_{3} \circ a_{1}]_{\sim} = [a_{4} \circ a_{2}]_{\sim}$.
    $a_{1} \sim a_{2}$ means there exists $d_{i}: s(d_{1}) \xrightarrow{a_{i}} t(d_{1})$ for $i=1,2$.
    Since actions are unrestricted in $W$, for any world state and any action there is a transition with a source at that world state that is labelled by that action.
    $a_{3} \sim a_{4}$ means there exists $d_{j}: s(d_{3}) \xrightarrow{a_{j}} t(d_{3})$ for $j=3,4$, and so there exists $d_{j}: t(d_{1}) \xrightarrow{a_{j}} t(d_{3})$ for $j=3,4$.
    $\implies$ there exists $(d_{3} \circ d_{1}): s(d_{1}) \xrightarrow{a_{3} \circ a_{1}} t(d_{3})$ and $(d_{4} \circ d_{2}): s(d_{1}) \xrightarrow{a_{4} \circ a_{2}} t(d_{3})$.
    $\implies$ $s(d_{3} \circ d_{1}) = s(d_{4} \circ d_{2})$ and $t(d_{3} \circ d_{1}) = s(d_{4} \circ d_{2})$.
    $\implies$ $(a_{3} \circ a_{1}) \sim (a_{4} \circ a_{2})$.
    $\implies$ $[a_{3} \circ a_{1}]_{\sim} = [a_{4} \circ a_{2}]_{\sim}$.
\end{proof}

\paragraph{Effect of equivalent actions on world states}
We define the effect of an element of  $A/\sim$ on world states as $*: (A/\sim) \times W \to W$ such that $[a]_{\sim} * w = a * w$.
Note that this is only defined if there exists $d: w \xrightarrow{a} t(d)$ for $d \in D_{A}$.; if not, then $[a]_{\sim} * w$ is called $\textit{undefined}$.
\medskip
\begin{proposition}
    $[a]_{\sim} * w$ is well-defined for all $a \in A$ and for all $w \in W$.
\end{proposition}
\begin{proof}
    We need to show that $a_{1} * w = a_{2} * w$ if $[a_{1}]_{\sim} = [a_{2}]_{\sim}$ for $a_{1},a_{2} \in A$ and $w \in W$.
    If $[a_{1}]_{\sim} = [a_{2}]_{\sim}$, then $a_{1} \sim a_{2}$.
    Since actions are unrestricted in $W$, for any world state and any action there is a transition with a source at that world state labelled by that action.
    $\implies$ there exists $d_{i}: w \xrightarrow{a_{i}} t(d_{1})$ for $i=1,2$.
    $\implies$ $a_{1} * w = d_{1} * w = t(d_{1})$ and $a_{2} * w = d_{2} * w = t(d_{1})$.
\end{proof}

\paragraph{Reversible actions}
An action $a \in A$ is called \textit{reversible} in a given state $w \in W$ if $a \in A_{R w}$ where $A_{Rw} = \{a \in A \mid\textit{ there exists }a' \in A\textit{ such that }a' \circ a * w = w \}$.
An action $a \in A$ is called \textit{reversible} if it is reversible in every $w \in W$.
An action that is not reversible is called \textit{irreversible}.

\paragraph{Properties of the quotient set $A/\sim$}

\begin{proposition}\label{prp:Asim-identity}
    $(A/\sim, \circ)$ has an identity element.
\end{proposition}
\begin{proof}
    To show that $(A/\sim, \circ)$ has an identity element we can show that there is an element $e \in A$ which satisfies (a) $[a]_{\sim} \circ [e]_{\sim} = [a]_{\sim}$ and (b) $[e]_{\sim} \circ [a]_{\sim} = [a]_{\sim}$ for all $a \in A$.
    We will prove that the identity action $1 \in A$ satisfies the above condition.
    Consider any transition $d: s(d) \xrightarrow{a} t(d)$ labelled by any action $a \in A$.
    
    (a) There exists a transition $1_{s(d)}: s(d) \xrightarrow{1} s(d)$ due to action condition \ref{actcon:identity-action}.
    $t(1_{s(d)})=s(d)$ $\implies$ $a \circ 1$ is defined for $d$.
    $s(a \circ 1) = s(1) = s(d) = s(a)$ and $t(a \circ 1) = t(a)$.
    $\implies$ $a \circ 1 \sim a$.
    $\implies$ $[a \circ 1]_{\sim} = [a]_{\sim}$.
    $\implies$ $[a]_{\sim} \circ [1]_{\sim} = [a]_{\sim}$.
    
    (b) There exists a transition $1_{t(d)}: t(d) \xrightarrow{1} t(d)$ due to action condition \ref{actcon:identity-action}.
    $t(1_{t(d)})=t(d)$ $\implies$ $1 \circ a$ is defined for $d$.
    $s(1 \circ a)=s(a)$ and $t(1 \circ a) = t(1) = t(a)$.
    $\implies$ $1 \circ a \sim a$.
    $\implies$ $[1 \circ a]_{\sim} = [a]_{\sim}$.
    $\implies$ $[1]_{\sim} \circ [a]_{\sim} = [a]_{\sim}$.
    Therefore $1 \in A$ satisfies the conditions for $[1]_{\sim}$ being an identity element in $(A/\sim, \circ)$.
\end{proof}

\begin{proposition}\label{prp:Asim-associative}
    $\circ$ is associative with respect to $(A/\sim, \circ)$.
\end{proposition}
\begin{proof}
    For $\circ$ to be associative we need $[a_{1}]_{\sim} \circ ([a_{2}]_{\sim} \circ [a_{3}]_{\sim}) = ([a_{1}]_{\sim} \circ [a_{2}]_{\sim}) \circ [a_{3}]_{\sim}$ for any $a_{1},a_{2},a_{3} \in A$.
    We have $a_{1} \circ (a_{2} \circ a_{3}) = (a_{1} \circ a_{2}) \circ a_{3}$ from the associativity of $\circ$ with respect to $(A, \circ)$, and $[a']_{\sim} \circ [a]_{\sim} = [a' \circ a]_{\sim}$ for any $a,a' \in A$ by definition of $\circ$ on $A/\sim$.
    $\implies$ $[a_{1}]_{\sim} \circ ([a_{2}]_{\sim} \circ [a_{3}]_{\sim}) = [a_{1} \circ (a_{2} \circ a_{3})]_{\sim} = [(a_{1} \circ a_{2}) \circ a_{3}]_{\sim} = [(a_{1} \circ a_{2})]_{\sim} \circ [a_{3}]_{\sim} = ([a_{1}]_{\sim} \circ [a_{2}]_{\sim}) \circ [a_{3}]_{\sim}$.
\end{proof}

In summary, we have $(A, \circ, *)$, which is a set $A$ along with two operators $\circ: A \times A \to A$ and $*: A \times W \to W$, and we have $(A/\sim, \circ, *)$, which is a set $A/\sim$ along with two operators $\circ: (A/\sim) \times (A/\sim) \to (A/\sim)$ and $*: (A/\sim) \times W \to W$.
We have shown that $\circ$ is associative with respect to $(A/\sim, \circ)$, and that $(A/\sim, \circ)$ has an identity element by action condition \ref{actcon:identity-action}.

\subsection{Algorithmic exploration of world structures}\label{sec:Algorithmic exploration of world structures}

To gain an intuition of the structure of different worlds and to illustrate our theoretical work with examples, we developed an algorithm that uses an agent’s minimum actions to generate the algebraic structure of the transformations of a world. We display this structure as a generalised Cayley table (a multiplication table for the distinct elements of the algebra). 

The main algorithm generates what we call a state Cayley table (Algorithm \ref{alg:Generate state Cayley table}). The elements of this state Cayley table are the world states reached when the row action and then the column action are performed in succession from an initial world state  $w$ (\textit{i.e.}, $column\_label * (row\_label * w)$). Once the state Cayley table has been generated, we can use it to generate the action Cayley table, in which the elements of the table are the equivalent elements in the algebra if the agent performs the row action followed by the column action (Algorithm \ref{alg:Generate action Cayley table}). For the sake of readability, Algorithms \ref{alg:AddElementToStateCayleyTable} to \ref{alg:SearchForBrokenEquivalenceClasses}, which are internal loops of Algorithm \ref{alg:Generate state Cayley table}, can be found in the Appendix. The code for their implementation of all the algorithms is available at \url{github.com/awjdean/CayleyTableGeneration}. 

It is worth noting the usefulness of these algorithms: independently of their use in this paper, to show to generate both SBDRL groups (Section \ref{sec:Reproducing SBDRL}) and symmetries that fall beyond SBDRL’s scope (Section \ref{sec:Beyond SBDRs}), researchers in representation learning and reinforcement learning are provided with an automatic mechanism to generate transitions of any arbitrary structure.

\begin{algorithm}[H]
    \caption{Generate state Cayley table.}\label{alg:Generate state Cayley table}
    \begin{algorithmic}[1]
        \Require $minimum\_actions$: a list of minimum actions, $w$: initial world state.
        \State $state\_cayley\_table \gets$ an empty square matrix with dimensions $len(minimum\_actions) \times len(minimum\_actions)$, with rows and columns labelled by $minimum\_actions$.

        \For{$a$ in $minimum\_actions$}
            \State Create an equivalence class for $a$.
            \State $state\_cayley\_table \gets$ AddElementToStateCayleyTable($state\_cayley\_table$, $w$, $a$).\Comment{See Algorithm \ref{alg:AddElementToStateCayleyTable}.}
        \EndFor

        \For{$row\_label$ in $state\_cayley\_table$}
            \State $equivalents\_found \gets$ SearchForEquivalents($state\_cayley\_table$, $w$, $row\_label$). \Comment{See Algorithm \ref{alg:SearchForEquivalents}.}
            \If{$len(equivalents\_found) \neq 0$}
                \State Merge the equivalence classes of equivalent minimum actions.
                \State Delete the row and column from $state\_cayley\_table$ for the minimum actions not labelling the merged equivalence class.
            \EndIf
        \EndFor
        
    \State Initialize an empty list $candidate\_cayley\_table\_elements$.
    \State $candidate\_cayley\_table\_elements \gets$ SearchForNewCandidates($state\_cayley\_table$, $w$, $candidate\_cayley\_table\_elements$).\Comment{See Algorithm \ref{alg:SearchForNewCandidates}.}
    
    \While{$len(candidate\_cayley\_table\_elements) > 0$}
        \State $a_{C} \gets$ pop an element from $candidate\_cayley\_table\_elements$.
        \State $equivalents\_found \gets$ SearchForEquivalents($state\_cayley\_table$, $w$, $a_{C}$). \Comment{See Algorithm \ref{alg:SearchForEquivalents}.}
        \If{$len(equivalents\_found) \neq 0$}
            \State Add $a_{C}$ to the relevant equivalence class.
            \State Continue to the next iteration of the while loop.
        \Else
            \State Check if $a_{C}$ breaks any of the existing equivalence classes.

            \State $broken\_equivalence\_classes \gets$ SearchForBrokenEquivalenceClasses($state\_cayley\_table$, $w$, $a_{C}$).\Comment{See Algorithm \ref{alg:SearchForBrokenEquivalenceClasses}.}
            
            \If{$len(broken\_equivalence\_classes) \neq 0$}
                \For{each new equivalence class}
                    \State $state\_cayley\_table \gets$ AddElementToStateCayleyTable($state\_cayley\_table$, $w$, $\textit{element labelling new equivalence class}$).\Comment{See Algorithm \ref{alg:AddElementToStateCayleyTable}.}
               \EndFor
            \EndIf
            \State Create new equivalence class for $a_{C}$.
            \State $state\_cayley\_table \gets$ AddElementToStateCayleyTable($state\_cayley\_table$, $w$, $a_{C}$).\Comment{See Algorithm \ref{alg:AddElementToStateCayleyTable}.}
        \EndIf
        \State $candidate\_cayley\_table\_elements \gets$ SearchForNewCandidates($state\_cayley\_table$, $w$, $candidate\_cayley\_table\_elements$).\Comment{See Algorithm \ref{alg:SearchForNewCandidates}.}
    \EndWhile
    \State \textbf{return} $state\_cayley\_table$
    \end{algorithmic}
\end{algorithm}

\begin{algorithm}[H]
\caption{Generate action Cayley table.}\label{alg:Generate action Cayley table}
\begin{algorithmic}[1]
    \Require $state\_cayley\_table$.
    \State $action\_cayley\_table \gets$ an empty square matrix with the dimensions of $state\_cayley\_table$, with rows and columns labelled by the rows and columns of $state\_cayley\_table$.
    \For{$row\_label$ in $action\_cayley\_table$}
        \For{$column\_label$ in $action\_cayley\_table$}
            \State $a_{C} \gets column\_label \circ row\_label$.
            \State $ec\_label \gets$ label of equivalence class containing $a_{C}$.
            \State $state\_cayley\_table[row\_label][column\_label] = ec_label$.
        \EndFor
    \EndFor
    \State \textbf{return:} $state\_cayley\_table$.
\end{algorithmic}
\end{algorithm}

\paragraph{Displaying the algebra}
We display the algebra in two ways:
(1) a \textit{$w$-state Cayley table}, which shows the resulting state of applying the row element to $w$ followed by the column element (\textit{i.e.,} $w\textit{-state Cayley table value} = \textit{column label} * (\textit{row label} * w$)), and (2) an \textit{action Cayley table}, which shows the resulting element of the algebra when the column element is applied to the left of the row element (\textit{i.e.}, $\textit{action Cayley table value} = \textit{column element} \circ \textit{row element}$).

\paragraph{Algebra properties}
We also check the following properties of the algebra algorithmically: (1) the presence of identity, including the presence of left and right identity elements separately, (2) the presence of inverses, including the presence of left and right inverses for each element, (3) associativity, (4) commutativity, and (5) the order of each element in the algebra.
For our algorithm to successfully generate the algebra of a world, the world must contain a finite number of states, the agent must have a finite number of minimum actions, and all the transformations of the world must be due to the actions of the agent.

\subsubsection{Example}\label{sec:Example}

For our example world $\mathscr{W}_{c}$, the equivalence classes shown in Figure \ref{fig:2x2-cyclical-min-act-equivalence} - those labelled by $1$, $R$, and $U$ - are the only equivalence classes in $A/\sim$.
The $w$-state Cayley table in Table \ref{tab:2x2-gridworld-no-walls-state-cayley} shows the final world state reached after the following operation: $\text{table entry} = \text{column element} * (\text{row element} * w)$.

The $w$-action Cayley table in Table \ref{tab:2x2-gridworld-no-walls-action-cayley} shows the equivalent action in $A/\sim$ for the same operation as the $w$-state Cayley table: $[\text{table entry}] * w = \text{column element} * (\text{row element} * w)$ for all $w \in W$.

The choice of the equivalence class label in Table \ref{tab:2x2-gridworld-no-walls-equivalence-classes} is arbitrary; it is better to think of each equivalence class as a distinct element as shown in the Cayley table in Table \ref{tab:2x2-gridworld-no-walls-action-cayley-abstract}.

There are four elements in the action algebra, therefore, if the agent learns the relations between these four elements, and then it has complete knowledge of the transformations of our example world.

\begin{table}
    \centering
    \begin{tabular}{c|c c c c c}
        $A/\sim$    &  $1$      & $D$       & $L$       & $RU$\\
         \hline
        $1$         & $w_{0}$   & $w_{2}$   & $w_{1}$   & $w_{3}$\\
        $D$         & $w_{2}$   & $w_{0}$   & $w_{3}$   & $w_{1}$\\
        $L$         & $w_{1}$   & $w_{3}$   & $w_{0}$   & $w_{2}$\\
        $RU$        & $w_{3}$   & $w_{1}$   & $w_{2}$   & $w_{0}$\\
    \end{tabular}
    \caption{$w_{0}$ state Cayley table for $A/\sim$.}
    \label{tab:2x2-gridworld-no-walls-state-cayley}
\end{table}

\begin{table}
    \centering
    \begin{tabular}{c|c c c c c}
        $A/\sim$    &  $1$      & $D$       & $L$       & $RU$\\
        \hline
        $1$         & $1$       & $D$       & $L$       & $RU$\\
        $D$         & $D$       & $1$       & $RU$      & $L$\\
        $L$         & $L$       & $RU$      & $1$       & $D$\\
        $RU$        & $RU$      & $L$       & $D$       & $1$\\
    \end{tabular}
    \caption{Action Cayley table for $A/\sim$.}
    \label{tab:2x2-gridworld-no-walls-action-cayley}
\end{table}

\begin{table}[H]
    \centering
    \begin{tabular}{c|l}
        $\sim$ equivalence class label & $\sim$ equivalence class elements\\
        \hline
        $1$         & $1, 11, DD, LL, RURU, ...$\\
        $D$         & $D, D1, 1D, RUL, LRU, ...$\\
        $L$         & $L, L1, RUD, 1L, DRU, ...$\\
        $RU$        & $RU, RU1, LD, DL, 1RU, ...$
    \end{tabular}
    \caption{Action Cayley table equivalence classes.}
    \label{tab:2x2-gridworld-no-walls-equivalence-classes}
\end{table}

\begin{table}[H]
    \centering
    \begin{tabular}{c|c c c c c}
        $A/\sim$    &  $1$      & $2$       & $3$       & $4$\\
        \hline
        $1$         & $1$       & $2$       & $3$       & $4$\\
        $2$         & $2$       & $1$       & $4$      & $3$\\
        $3$         & $3$       & $4$      & $1$       & $2$\\
        $4$        & $4$      & $3$       & $2$       & $1$\\
    \end{tabular}
    \caption{Abstract action Cayley table for $A/\sim$.}
    \label{tab:2x2-gridworld-no-walls-action-cayley-abstract}
\end{table}

\paragraph{Properties of $A/\sim$ algebra}

\begin{table}
    \centering
    \begin{tabular}{c|c}
        \textbf{Property}   & \textbf{Present?} \\
        \hline
        Totality            & Y\\
        Identity            & Y\\
        Inverse             & Y\\
        Associative         & Y\\
        Commutative         & Y
    \end{tabular}
    \caption{Properties of the $A/\sim$ algebra.}
    \label{tab:2x2-gridworld-no-walls-algebra-properties}
\end{table}

The properties of the $A/\sim$ algebra are displayed in Table \ref{tab:2x2-gridworld-no-walls-algebra-properties} and show that $A/\sim$ is a commutative group, where the no-op action is the identity, and all elements are their own inverses.
Since the action algebra of our example world is a group, it can be described by SBDRL.
The order of each element is given by Table \ref{tab:2x2-gridworld-no-walls-element orders}.

\begin{table}
    \centering
    \begin{tabular}{c|c}
        \textbf{Element}   & \textbf{Order} \\
        \hline
        $1$     & 1\\
        $D$     & 2\\
        $L$     & 2\\
        $RU$    & 2
    \end{tabular}
    \caption{Order of elements in $A/\sim$.}
    \label{tab:2x2-gridworld-no-walls-element orders}
\end{table}

\subsection{Conditions for SBDRL to apply}\label{sec:World conditions}

To simplify the problem, we only consider worlds where the transformations of the world are only due to the actions of an agent for the remainder of this paper unless otherwise stated. Therefore, we will only consider worlds with $D = D_{A}$.

To be a group, $A/\sim$ must satisfy the properties of (1) identity, (2) associativity, (3) closure, and (4) inverse.
The identity and associative properties are satisfied by Proposition \ref{prp:Asim-identity} and Proposition \ref{prp:Asim-associative} respectively. For the closure property to be satisfied, the following condition is sufficient:

\begin{world_condition}[Unrestricted actions]\label{wldcon:unrestricted-actions}
    For any action $a \in A$ and for any world state $w \in W$, there exists a transition $d \in D_{A}$ with $d: w \xrightarrow{a} t(d)$.
    In other words, $a * w$ is defined for all $a \in A$ and all $w \in W$.

\end{world_condition}

The inverse property of a group (4) is stricter than our current definition of the reversibility of an action. For the structure $A/\sim$ to have the inverse property, each element (action) in $A/\sim$ must not only be reversible from each starting state, but additionally, the inverse of a given element in $A/\sim$ must be the same for each starting state; for example, if $a'$ is the inverse of $a$ from a state $w \in W$ ($a' \circ a * w = w$), then $a'$ must be the inverse of $a$ from all states in $W$.

\begin{world_condition}[Inverse actions]\label{wldcon:inverse-actions}
    For each $a \in A/\sim$, there exists an $a' \in A/\sim$ such that $a' \circ a * w = w$ and $a \circ a' * w = w$ for all $w \in W$.
\end{world_condition}

\begin{proposition}\label{prp:Asim-group}
    If the world satisfies world conditions \ref{wldcon:unrestricted-actions} and \ref{wldcon:inverse-actions} then $(A/\sim, \circ)$ is a group.
\end{proposition}
\begin{proof}

    Totality is given by world condition \ref{wldcon:unrestricted-actions}.
    Associativity is given by proposition \ref{prp:Asim-associative}.
    Identity element given by proposition \ref{prp:Asim-identity}.
    Inverse element is given by world condition \ref{wldcon:inverse-actions}.
\end{proof}

\begin{proposition}\label{prp:world-conditions-sufficient}
    If the world obeys world conditions \ref{wldcon:unrestricted-actions} and \ref{wldcon:inverse-actions}, then $*: (A/\sim) \times W \to W$ is a left group action.
\end{proposition}
\begin{proof}

    We have already established that $(A/\sim, \circ)$ is a group (proposition \ref{prp:Asim-group}).
    Therefore, to show that $*$ is a left group action we only have to prove the group action conditions of (a) identity and (b) compatibility.
    Consider an arbitrary world state $w \in W$.
    
    (a) $[1]_{\sim} * w = 1 * w = 1_{w} * w = w$.
    
    (b) We need to show that $a' * (a * w) = (a' \circ a) * w$.
    Because actions are unrestricted in $W$, for any $w \in W$, there exists the transitions $d_{1}: w \xrightarrow{a} t(d_{1})$ and $d_{2}: t(d_{1}) \xrightarrow{a} t(d_{2})$.
    $\implies$ there exists the transition $(d_{2} \circ d_{1}): w \xrightarrow{a' \circ a} t(d_{2})$.
    Therefore, $(a' \circ a) * w = t(d_{2})$.
    Using the transitions $d_{1}$ and $ d_{2}$ for the LHS of the condition, $a' * ( a * w) = a' * t(d_{1}) = t(d_{2})$.
\end{proof}

From Proposition \ref{prp:world-conditions-sufficient}, if a world satisfies world conditions \ref{wldcon:unrestricted-actions} and \ref{wldcon:inverse-actions}, then the transformations of that world can be fully described using SBDRL (\textit{i.e.}, $*: (A/\sim) \times W \to W$ is a group action).

\begin{proposition}\label{prp:unrestricted-actions-necessary}
    If $*: (A/\sim) \times W \to W$ is a group action, then world condition \ref{wldcon:unrestricted-actions} is satisfied.
\end{proposition}
\begin{proof}
    Since a group action is a full operation by definition,  world condition \ref{wldcon:unrestricted-actions} is satisfied.
\end{proof}

\begin{proposition}\label{prp:inverse-actions-sufficient}
    If $*: (A/\sim) \times W \to W$ is a group action, then world condition \ref{wldcon:inverse-actions} is satisfied.
\end{proposition}
\begin{proof}
    If $*$ is a group action, 
    then $A/\sim$ is a group.
    If $A/\sim$ is a group, then for each $a \in A/\sim$ there is an inverse element $a^{-1}$ such that (1) $a^{-1} \circ a = 1$ and (2) $a \circ a^{-1} = 1$.

    For an arbitrary state $w \in W$, $a^{-1} \circ (a * w) = a^{-1} \circ (a * w)$, therefore $a^{-1} \circ (a * w) = (a^{-1} \circ a) * w$ from the group action compatibility condition, therefore  $a^{-1} \circ (a * w) = 1 * w$ from (1), therefore (3) $a^{-1} \circ (a * w) = w$ from the group action identity condition.

    Similarly, for an arbitrary state $w \in W$, $a \circ (a^{-1} * w) = a \circ (a^{-1} * w)$, therefore $a \circ (a^{-1} * w) = (a \circ a^{-1}) * w$ from the group action compatibility condition, therefore $a \circ (a^{-1} * w) = 1 * w$ from (2), therefore (4) $a \circ (a^{-1} * w) = w$ from the group action identity condition.

    (3) and (4) together are world condition \ref{wldcon:inverse-actions}.
\end{proof}

\begin{proposition}\label{prp:WC-unrestricted-actions-necessary}
    If a world does not satisfy world condition \ref{wldcon:unrestricted-actions}, then $*$ is not a group action.
\end{proposition}
\begin{proof}
    If a world does not satisfy world condition \ref{wldcon:unrestricted-actions}, then there exists some world state $w \in W$ and some action $a \in A$ such that $a * w$ is undefined.
    Therefore, for $[a] \in A/\sim$, $[a] * w$ is undefined, and so $*: (A/\sim) \times W \to W$ is a partial operation and so not a group action.
\end{proof}

\begin{proposition}\label{prp:WC-inverse-actions-necessary}
    If a world does not satisfy world condition \ref{wldcon:inverse-actions}, then $A/\sim$ is not a group and so $*$ is not a group action.
\end{proposition}
\begin{proof}
    If a world does not satisfy world condition \ref{wldcon:inverse-actions}, then there exists an $a \in A/\sim$ such that there is no $a' \in A/\sim$ for which $a' \circ a * w = w$ or $a \circ a' * w = w$ for all $w \in W$.
    \textbf{Proof by contradiction.}
    Assume $a$ has an inverse $a'' \in A/\sim$.
    Therefore, $a'' \circ a \sim 1$ and $a \circ a'' \sim 1$.
    Since $1 * w = w$ for all $w \in W$, $a'' \circ a * w = w$ and $a \circ a'' * w = w$ for all $w \in W$, which is a contradiction.
\end{proof}

Proposition \ref{prp:world-conditions-sufficient} shows that world conditions \ref{wldcon:unrestricted-actions} and \ref{wldcon:inverse-actions} are sufficient conditions for $*$ to be a group action, while propositions \ref{prp:WC-unrestricted-actions-necessary} and \ref{prp:WC-inverse-actions-necessary} show that world conditions \ref{wldcon:unrestricted-actions} and \ref{wldcon:inverse-actions} are necessary conditions for $*$ to be a group action.
Since world conditions \ref{wldcon:unrestricted-actions} and \ref{wldcon:inverse-actions} are sufficient and necessary conditions for $*$ to be a group action and therefore $A/\sim$ to be a group, these conditions give a characterisation of the worlds with transformations (due to the actions of an agent) that can be fully described using SBDRL; in other words, if the transformations of a world can be fully described using SBDRL then that world satisfies world conditions \ref{wldcon:unrestricted-actions} and \ref{wldcon:inverse-actions}, and if a world satisfies world conditions \ref{wldcon:unrestricted-actions} and \ref{wldcon:inverse-actions} then the transformations of that world can be fully described using SBDRL.

\subsection{Action-homogeneous worlds}

Given that the algebraic structure underlying symmetric representations in SBDRL are groups, the (extra) condition of homogeneity directly applies:

\begin{world_condition}[Action homogeneity]\label{wldcon:action-homogeneity}
    For every pair $(w_{1}, w_{2}) \in W^{2}$, there exists a bijective map $\sigma_{(w_{1},w_{2})}: W \to W$ such that $\sigma_{(w_{1},w_{2})}(w_{1})=w_{2}$ and such that:
    
    \begin{enumerate}
        \item for every $d \in D_{A}$ with $d: s(d) \xrightarrow{a} t(d)$, there exists a $d' \in D_{A}$ with $d': \sigma_{(w_{1}, w_{2})}(s(d)) \xrightarrow{a} \sigma_{(w_{1}, w_{2})}(t(d))$;
        
        \item for every $d \in D_{A}$ with $d: s(d) \xrightarrow{a} t(d)$, there exists a $d' \in D_{A}$ with $d': \sigma^{-1}_{(w_{1}, w_{2})}(s(d)) \xrightarrow{a} \sigma^{-1}_{(w_{1}, w_{2})}(t(d))$.
    \end{enumerate}
\end{world_condition}

World condition \ref{wldcon:action-homogeneity} means that action sequences have the same result for any initial world state.
Essentially, this means that the world looks the same from any world state with respect to the relationships of actions.
We call worlds with world condition \ref{wldcon:action-homogeneity} \textit{action-homogeneous worlds}.

\medskip
 In this section, we have presented a mathematical formalism for describing transitions due to the actions of an agent between world states and an algorithm that generates them. We have then laid out an example world and shown that, after applying an equivalence relation, the algebra of the agent's actions in this world forms a group. We have also characterised the worlds with transformations that can be fully described by SBDRL by giving world conditions that are sufficient and necessary for the algebra of the transformations of that world to be a group. In the next section, we present transformations relevant to reinforcement learning that do not form groups and are thus beyond the expressive power of SBDL but that can be formulated using our framework.

\section{Beyond SBDRs}\label{sec:Beyond SBDRs}

In the previous section, we showed under which conditions, those of forming a (symmetry) group, SBDRL representations can be expressed within our framework, as defined in Section \ref{sec:Mathematical framework for an agent in an environment}. We will now consider worlds that are accounted for using our framework but that do not satisfy either world condition \ref{wldcon:unrestricted-actions} or world condition \ref{wldcon:inverse-actions}, that is, transformations that do not form a group action and for which the assumption of action-homogeneous worlds does not necessarily hold. We have selected examples with features that are common in simple reinforcement learning scenarios to illustrate the potential of our approach in extending SBDRL to symmetries that go beyond group actions.

In order to do so, we use two methods of treating actions that are not allowed to be used in certain world states (e.g., the agent trying to move through a wall or eat a consumable in a state where the agent is not in the same location as the consumable). Method 1 (Sections \ref{sec:identity reversible action-inhomogeneous world} and \ref{sec:identity irreversible inhomogeneous actions}) lets the agent select the actions but any actions that would have been undefined in a state \textit{w}  have the same effect in \textit{w} as the agent performing the identity action 1. Method 2 (Sections \ref{sec:masked reversible action-inhomogeneous world} and \ref{sec:masked irreversible action-inhomogeneous world}) does not let the agent select these actions and so considers them as undefined; this violates world condition \ref{wldcon:unrestricted-actions}. These two treatments of actions are common in reinforcement learning. We employ the computational methods outlined in Section \ref{sec:Algorithmic exploration of world structures} to generate the action algebras of the agent in these worlds and provide evidence for our statements. 

\subsection{Worlds without inverse actions}

In this section, we consider worlds that do not necessarily satisfy world condition \ref{wldcon:inverse-actions} but do satisfy world condition \ref{wldcon:unrestricted-actions}.

\begin{proposition}\label{prp:wc1_gives_monoid_action}
    Consider a world $\mathscr{W}$ with a set $W$ of world states and containing an agent with a set $A$ of actions.
    If $\mathscr{W}$ satisfies world condition \ref{wldcon:unrestricted-actions}, then $*: (A/\sim) \times W \to W'$, where $W' \subseteq W$, is the action of a monoid $A/\sim$ on $W$.
\end{proposition}
\begin{proof}
    (1) Totality of $A/\sim$ is given by world condition \ref{wldcon:unrestricted-actions}.
    (2) Associativity of $A/\sim$ is given by proposition \ref{prp:Asim-associative}.
    (3) Identity element of $A/\sim$ is given by proposition \ref{prp:Asim-identity}.
    Since $A/\sim$ satisfies properties (1), (2), and (3), $A/\sim$ is a monoid.
    
    $*$ is defined for any $a \in A$ and $w \in W$, therefore $*$ is a monoid action.
\end{proof}

\begin{remark}
    If all actions are reversible in $\mathscr{W}$ then $W' = W$.
    If any action is irreversible in $\mathscr{W}$ then $W' \subset W$.
\end{remark}

\subsubsection{Example 1: reversible action-inhomogeneous world}\label{sec:identity reversible action-inhomogeneous world}

To turn the $2 \times 2$ cyclical grid world $\mathscr{W}_{c}$ used previously from a reversible action-homogeneous world to a reversible action-inhomogeneous world we add a wall to the world as shown in Figure \ref{fig:2x2-cyclical-grid-world-wall-states} to give world $\mathscr{W}_{wall}$.

\begin{figure}[H]
  \centering
    \begin{subfigure}[b]{0.45\linewidth}
        \centering
        \includegraphics[width=0.5\linewidth]{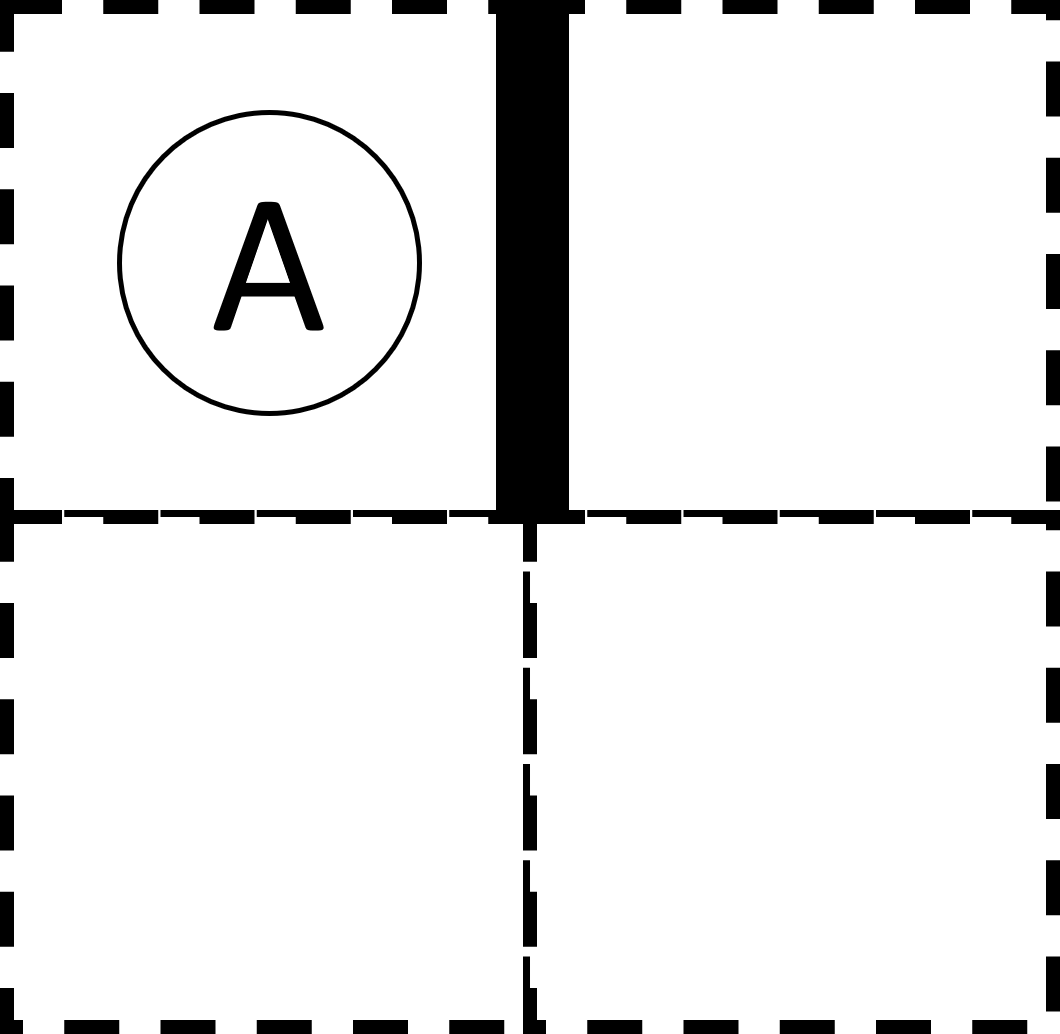}
        \caption{$w_{0}$}
        \vspace{0.25cm}
    \end{subfigure}
    \begin{subfigure}[b]{0.45\linewidth}
        \centering
        \includegraphics[width=0.5\linewidth]{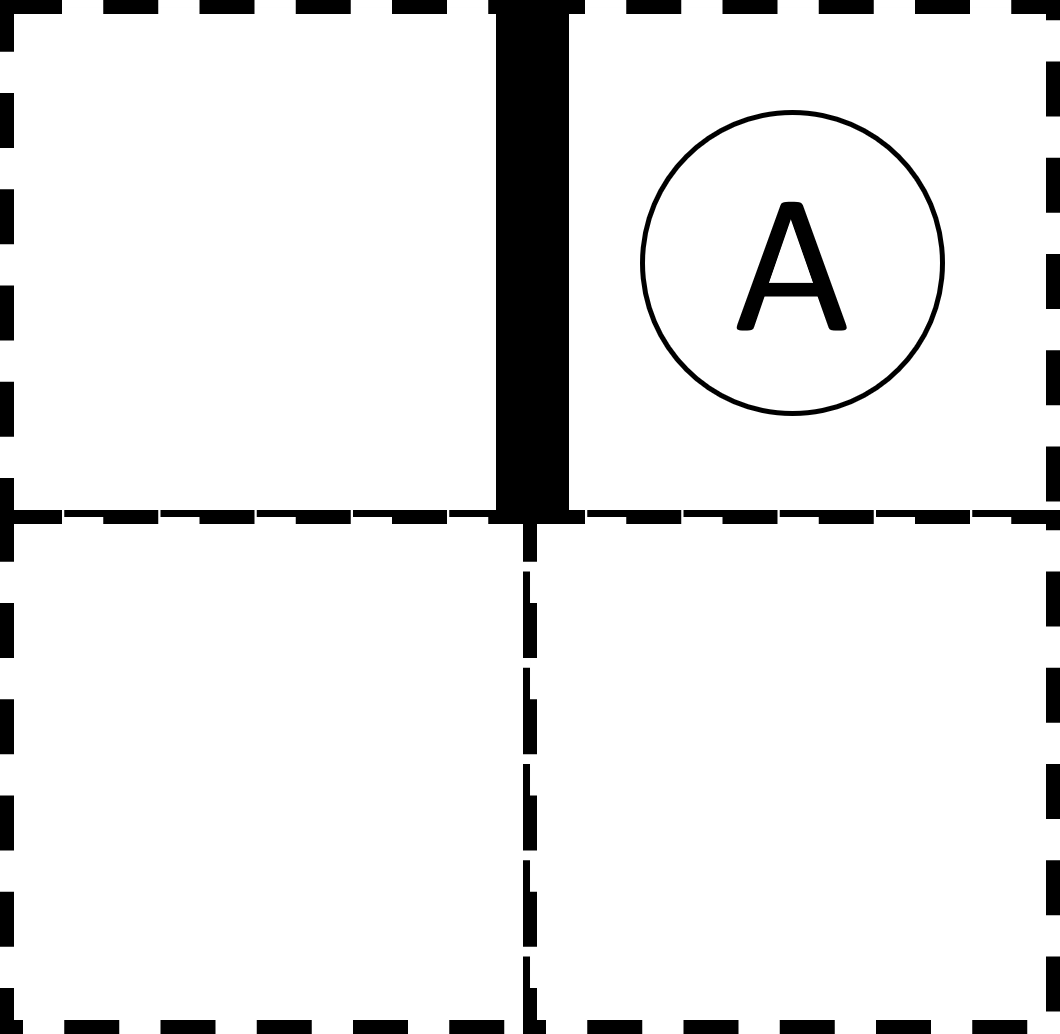}
        \caption{$w_{1}$}
        \vspace{0.25cm}
    \end{subfigure}
    \begin{subfigure}[b]{0.45\linewidth}
        \centering
        \includegraphics[width=0.5\linewidth]{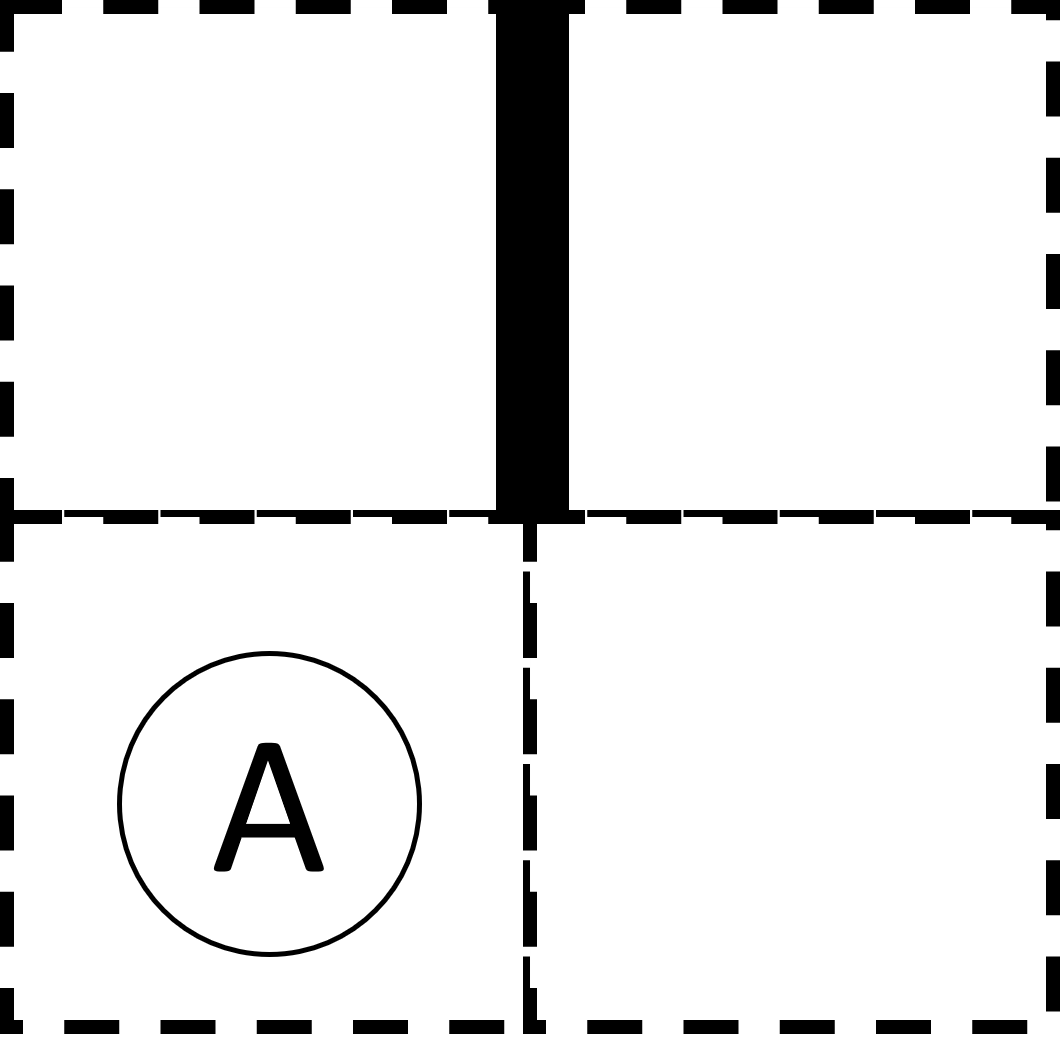}
        \caption{$w_{2}$}
    \end{subfigure}
    \begin{subfigure}[b]{0.45\linewidth}
        \centering
        \includegraphics[width=0.5\linewidth]{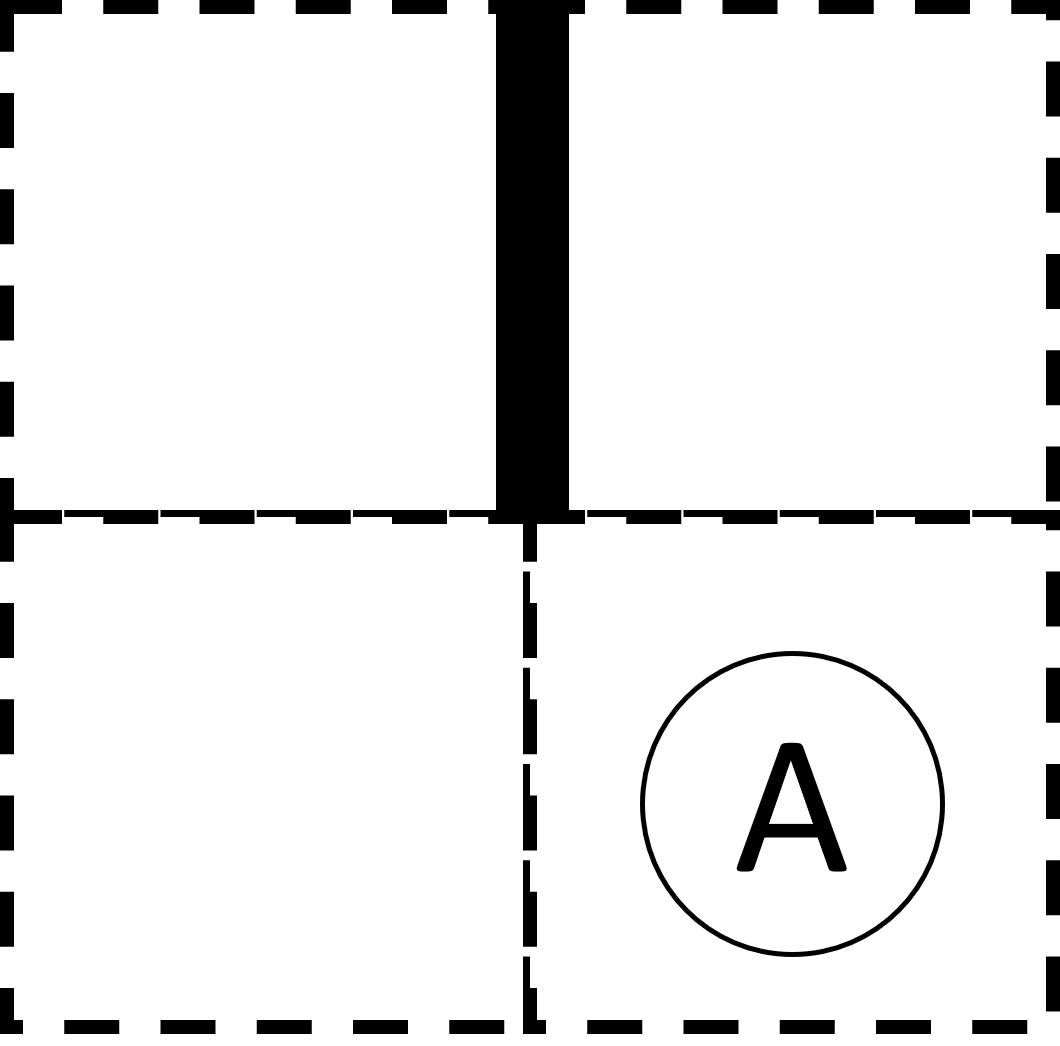}
        \caption{$w_{3}$}
    \end{subfigure}
  \caption{
  The world states of a cyclical $2\times 2$ grid world $\mathscr{W}_{wall}$ with a wall, where changes to the world are due to an agent moving either up, down, left, or right.
  The position of the agent in the world is represented by the position of the circled A.
  The treatment of the wall is explained when needed.
  }
  \label{fig:2x2-cyclical-grid-world-wall-states}
\end{figure}

This wall is said to \textit{restrict} the actions of the agent.
We say that actions restricted by the wall affect the world in the same way as the identity action; for example, if the agent is directly to the left of a wall and performs the `move to the right' action, then this action is treated like the identity action and so the state of the world does not change (see Table \ref{tab:2x2-gridworld-minimum-transitions-wall-identity} and Figure \ref{fig:2x2-cyclical-min-actions-wall-identity}).

\begin{table}[H]
    \centering
    \begin{tabular}{c|c c c c c}
                &  $1$      & $U$       & $D$       & $L$               & $R$\\
         \hline
        $w_{0}$ & $w_{0}$   & $w_{2}$   & $w_{2}$   & $w_{1}$           & \textbf{$w_{0}$}\\
        $w_{1}$ & $w_{1}$   & $w_{3}$   & $w_{3}$   & \textbf{$w_{1}$}  & $w_{0}$\\
        $w_{2}$ & $w_{2}$   & $w_{0}$   & $w_{0}$   & $w_{3}$           & $w_{3}$\\
        $w_{3}$ & $w_{3}$   & $w_{1}$   & $w_{1}$   & $w_{2}$           & $w_{2}$\\
    \end{tabular}
    \caption{
    Each entry in this table shows the outcome state of the agent performing the action given in the column label when in the world state given by the row label.
    }
    \label{tab:2x2-gridworld-minimum-transitions-wall-identity}
\end{table}

\begin{figure}
    \centering
    \includegraphics[width=0.5\linewidth]{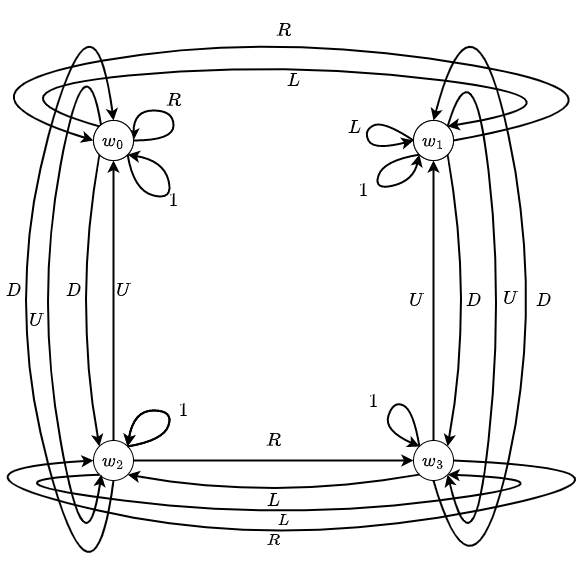}
    \caption{
    A transition diagram of the labelled transitions in Table \ref{tab:2x2-gridworld-minimum-transitions-wall-identity}.
    }
    \label{fig:2x2-cyclical-min-actions-wall-identity}
\end{figure}

\paragraph{Properties and structure of $A/\sim$}
The action Cayley table for this world with the identity treatment of the walls contains 26 elements.
As shown in Table \ref{tab:identity-walls-properties}, $A/\sim$ is a monoid.

\begin{table}[H]
    \centering
    \begin{tabular}{c|c}
        \textbf{Property}   & \textbf{Present?} \\
        \hline
        Totality            & Y\\
        Identity            & Y\\
        Inverse             & N\\
        Associative         & Y\\
        Commutative         & Y
    \end{tabular}
    \caption{Properties of the $A/\sim$ algebra.}
    \label{tab:identity-walls-properties}
\end{table}

Adding a single wall to the world has massively increased the complexity of the transition algebra of the world.
While the transition algebra of $\mathscr{W}_{c}$ has four elements, the transition algebra of $\mathscr{W}_{wall}$ with restricted actions treated as identity actions contains 26 elements.

The restrictiveness of the group inverse condition should be noted.
For $\mathscr{W}_{wall}$ with restricted actions treated as identity actions, every action is reversible from a particular state $w \in W$.
However, the action that takes $a$ back to its starting state is not necessarily the same any starting state $w \in W$.
For the inverse property to be present, the inverse for each element must be the same from any starting state (\textit{i.e.}, the inverse must be independent of the starting state).
$\mathscr{W}_{wall}$ with restricted actions treated as identity actions is proof that it is possible to have a world where all actions are reversible but for some of those actions to not have an inverse action.

\subsubsection{Example 2: Reversible action-inhomogenous world without walls}

Consider a world $\mathscr{W}_{block}$ with the world states in Figure \ref{fig:movable_block_world_states} and with movement along a single 1D cyclical axis with a movable block.
If the agent is in the location directly to the left of the block and moves into the block, the block moves one location in the direction of the agent's movement and the agent moves into the location previously occupied by the block (see Table \ref{tab:4x1-gridworld-minimum-transitions-moveable-block} and Figure \ref{fig:4x1-block-min-actions-wall}).

\begin{figure}[htbp]
  \centering
  \begin{subfigure}{0.48\textwidth}
    \centering
    \includegraphics[width=\textwidth]{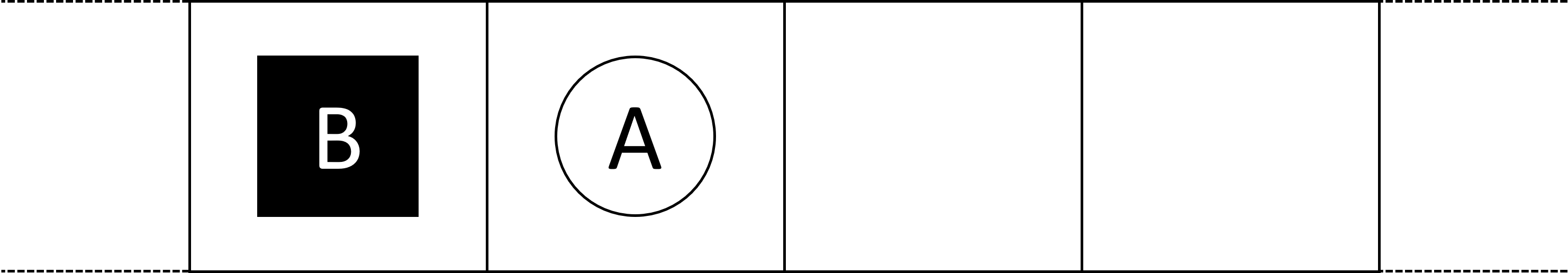}
    \caption{$w_{0}$}
  \end{subfigure}%
  \hfill
  \begin{subfigure}{0.48\textwidth}
    \centering
    \includegraphics[width=\textwidth]{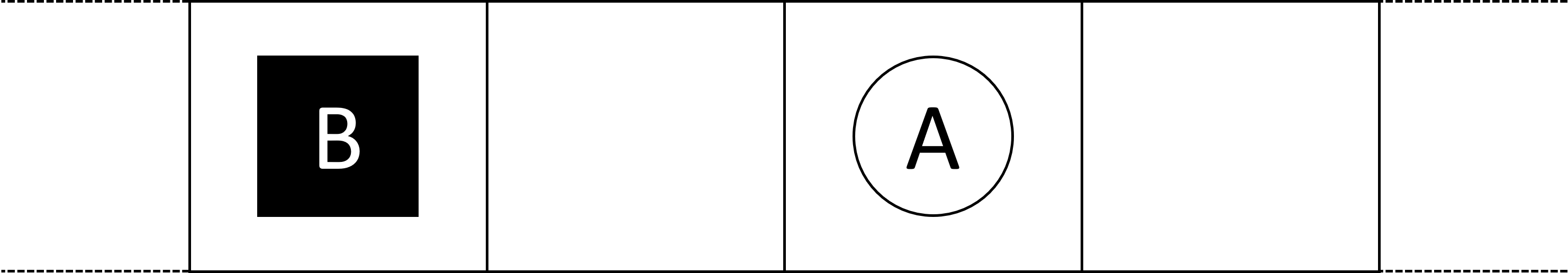}
    \caption{$w_{1}$}
  \end{subfigure}%
  \vspace{0.5cm}
  \begin{subfigure}{0.48\textwidth}
    \centering
    \includegraphics[width=\textwidth]{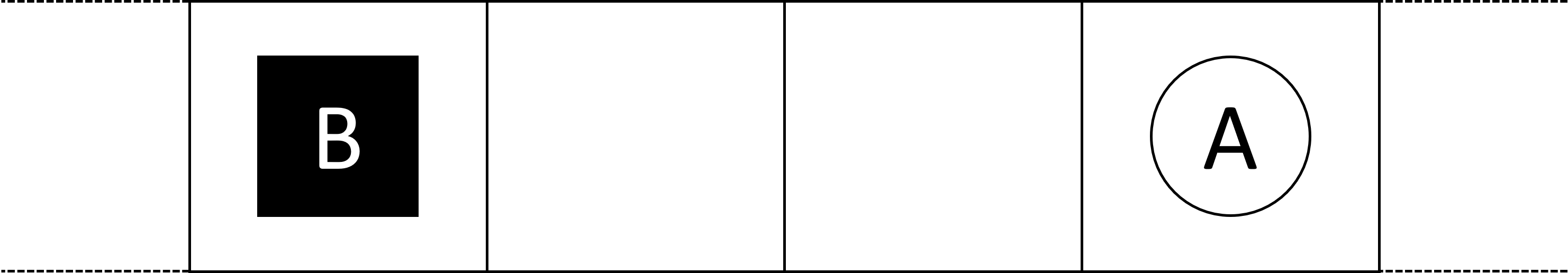}
    \caption{$w_{2}$}
  \end{subfigure}%
  \hfill
  \begin{subfigure}{0.48\textwidth}
    \centering
    \includegraphics[width=\textwidth]{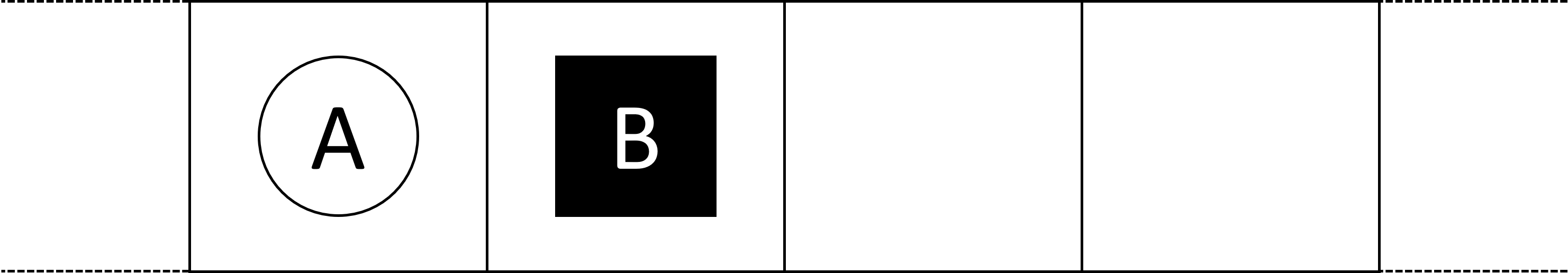}
    \caption{$w_{3}$}
  \end{subfigure}%
  \vspace{0.5cm}
  \begin{subfigure}{0.48\textwidth}
    \centering
    \includegraphics[width=\textwidth]{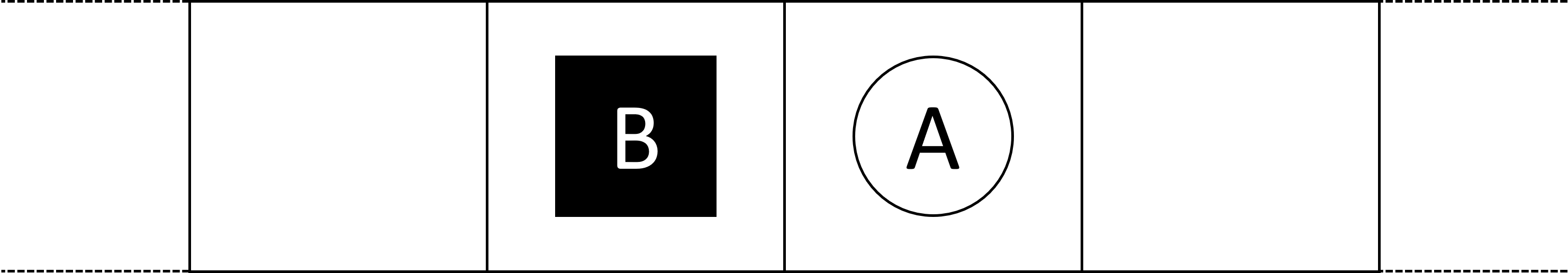}
    \caption{$w_{4}$}
  \end{subfigure}%
  \hfill
  \begin{subfigure}{0.48\textwidth}
    \centering
    \includegraphics[width=\textwidth]{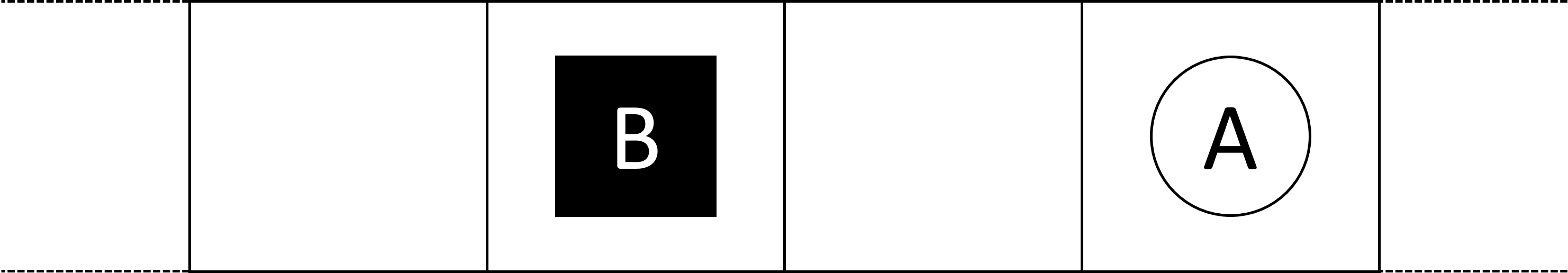}
    \caption{$w_{5}$}
  \end{subfigure}%
  \vspace{0.5cm}
  \begin{subfigure}{0.48\textwidth}
    \centering
    \includegraphics[width=\textwidth]{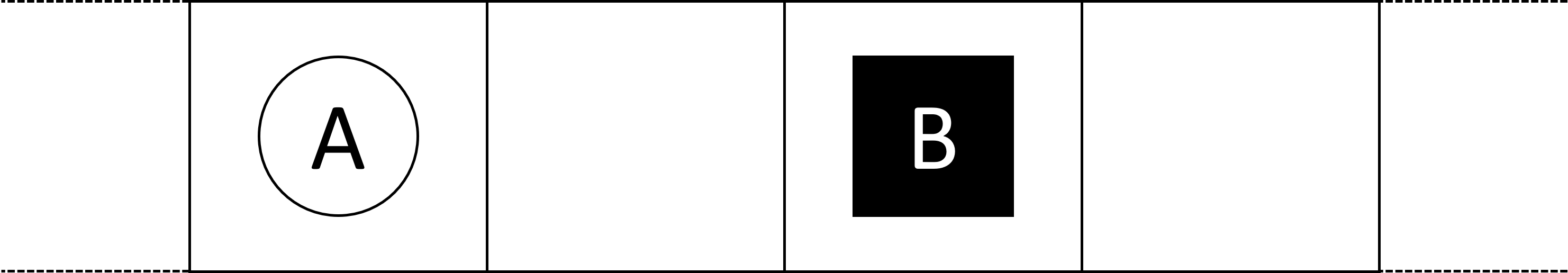}
    \caption{$w_{6}$}
  \end{subfigure}%
  \hfill
  \begin{subfigure}{0.48\textwidth}
    \centering
    \includegraphics[width=\textwidth]{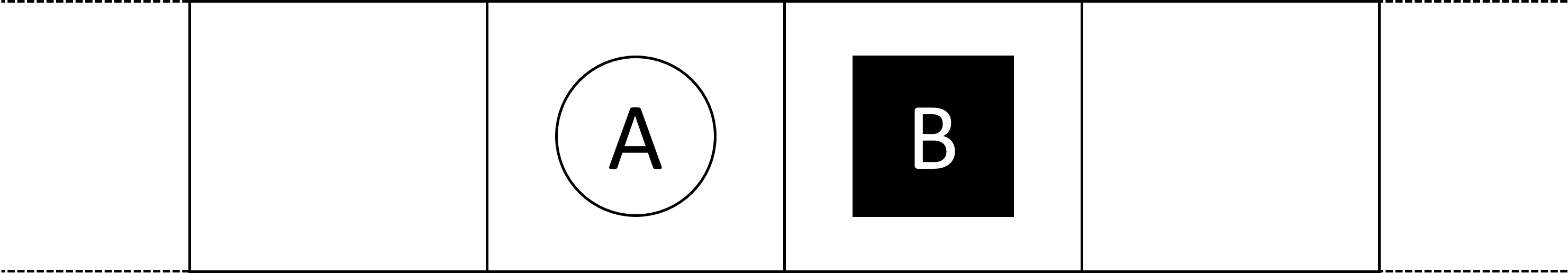}
    \caption{$w_{7}$}
  \end{subfigure}%
  \vspace{0.5cm}
  \begin{subfigure}{0.48\textwidth}
    \centering
    \includegraphics[width=\textwidth]{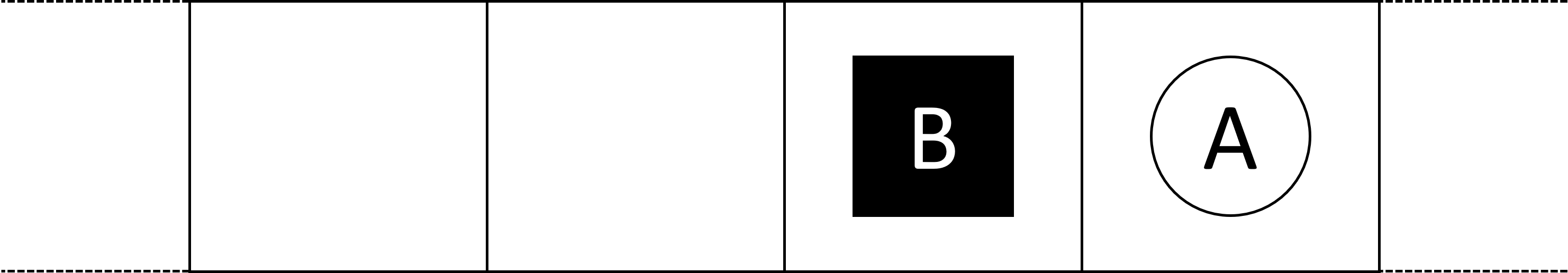}
    \caption{$w_{8}$}
  \end{subfigure}%
  \hfill
  \begin{subfigure}{0.48\textwidth}
    \centering
    \includegraphics[width=\textwidth]{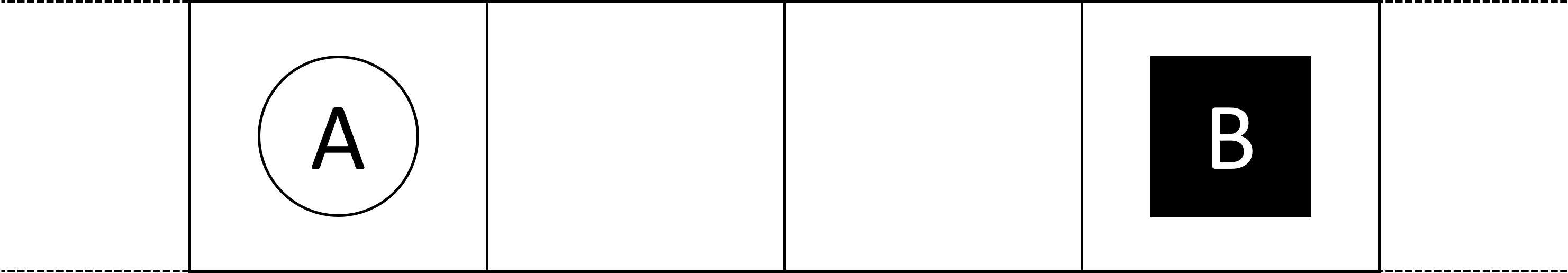}
    \caption{$w_{9}$}
  \end{subfigure}%
    \vspace{0.5cm}
  \begin{subfigure}{0.48\textwidth}
    \centering
    \includegraphics[width=\textwidth]{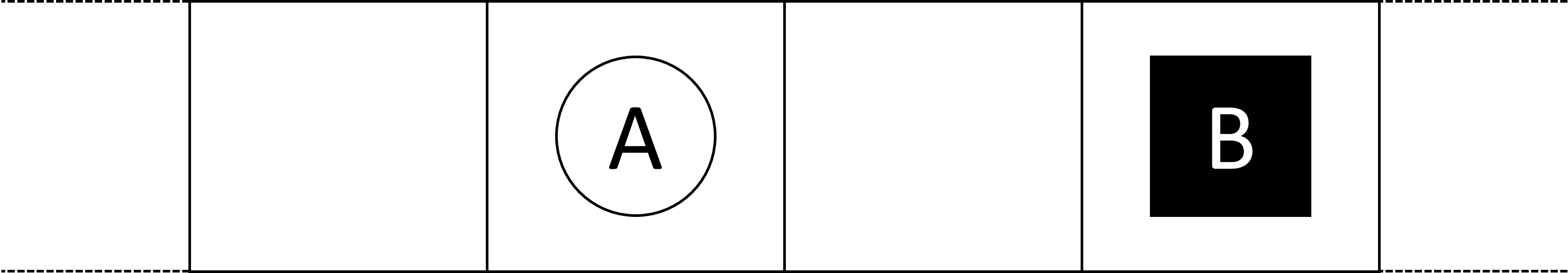}
    \caption{$w_{10}$}
  \end{subfigure}%
  \hfill
  \begin{subfigure}{0.48\textwidth}
    \centering
    \includegraphics[width=\textwidth]{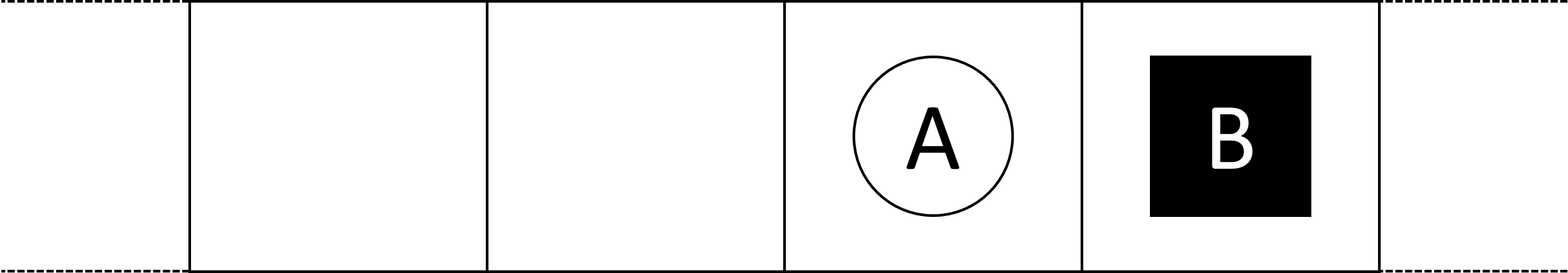}
    \caption{$w_{11}$}
  \end{subfigure}%
  \caption{World states of a world containing an agent and a movable block.}
  \label{fig:movable_block_world_states}
\end{figure}

\begin{table}[H]
    \centering
    \begin{tabular}{c|c c c c c}
                &  $1$      & $L$      & $R$\\
         \hline
        $w_{0}$ & $w_{0}$   & $w_{9}$   & $w_{1}$\\
        $w_{1}$ & $w_{1}$   & $w_{0}$   & $w_{2}$\\
        $w_{2}$ & $w_{2}$   & $w_{1}$   & $w_{3}$\\
        $w_{3}$ & $w_{3}$   & $w_{5}$   & $w_{7}$\\
        $w_{4}$ & $w_{4}$   & $w_{0}$   & $w_{5}$\\
        $w_{5}$ & $w_{5}$   & $w_{4}$   & $w_{3}$\\
        $w_{6}$ & $w_{6}$   & $w_{8}$   & $w_{7}$\\
        $w_{7}$ & $w_{7}$   & $w_{6}$   & $w_{11}$\\
        $w_{8}$ & $w_{8}$   & $w_{4}$   & $w_{6}$\\
        $w_{9}$ & $w_{9}$   & $w_{8}$   & $w_{10}$\\
        $w_{10}$ & $w_{10}$ & $w_{9}$   & $w_{11}$\\
        $w_{11}$ & $w_{11}$ & $w_{10}$  & $w_{2}$\\
    \end{tabular}
    \caption{
    Each entry in this table shows the outcome state of the agent performing the action given in the column label when in the world state given by the row label.
    }
    \label{tab:4x1-gridworld-minimum-transitions-moveable-block}
\end{table}

\begin{figure}[H]
    \centering
    \includegraphics[width=0.7\linewidth]{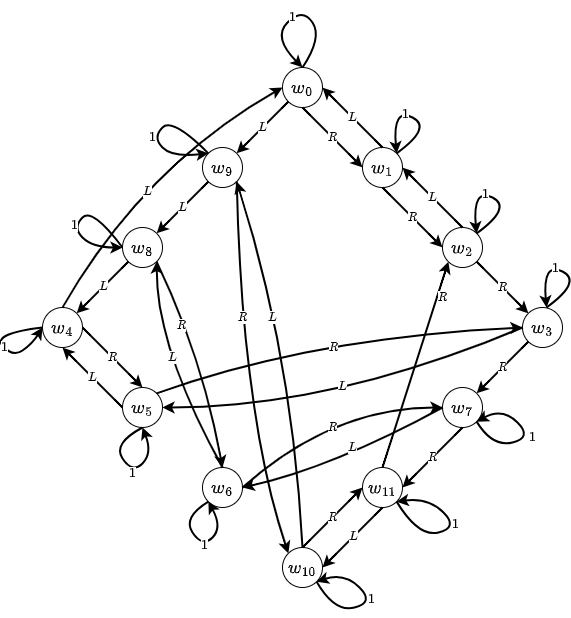}
    \caption{
    A transition diagram of the labelled transitions in Table \ref{tab:4x1-gridworld-minimum-transitions-moveable-block}.
    }
    \label{fig:4x1-block-min-actions-wall}
\end{figure}

The action Cayley table for $\mathscr{W}_{block}$ contains 17 elements.
As shown by Table \ref{tab:block-world-properties}, $A/\sim$ is a monoid.

\begin{table}[H]
    \centering
    \begin{tabular}{c|c}
        \textbf{Property}   & \textbf{Present?} \\
        \hline
        Totality            & Y\\
        Identity            & Y\\
        Inverse             & N\\
        Associative         & Y\\
        Commutative         & N
    \end{tabular}
    \caption{Properties of the $A/\sim$ algebra.}
    \label{tab:block-world-properties}
\end{table}

\begin{remark}
    Adding a wall or a moveable block breaks the symmetry of the original $2 \times 2$ cyclical world $\mathscr{W}_{c}$; this manifests as the action algebra of the world being much more complex.
\end{remark}

\subsubsection{Example 3: irreversible inhomogeneous actions}\label{sec:identity irreversible inhomogeneous actions}

This example explores the transition algebras of worlds that are action-inhomogeneous and contain irreversible actions.

Consider a world $\mathscr{W}_{consumable}$ with movement along a single 1D cyclical axis with a single consumable.
Let this world also contain an agent that can move left, right or consume the consumable if the agent is in the same place as the consumable (see Figure \ref{fig:consumable_world_states}).

\begin{figure}[H]
  \centering
  \begin{subfigure}{0.48\textwidth}
    \centering
    \includegraphics[width=\textwidth]{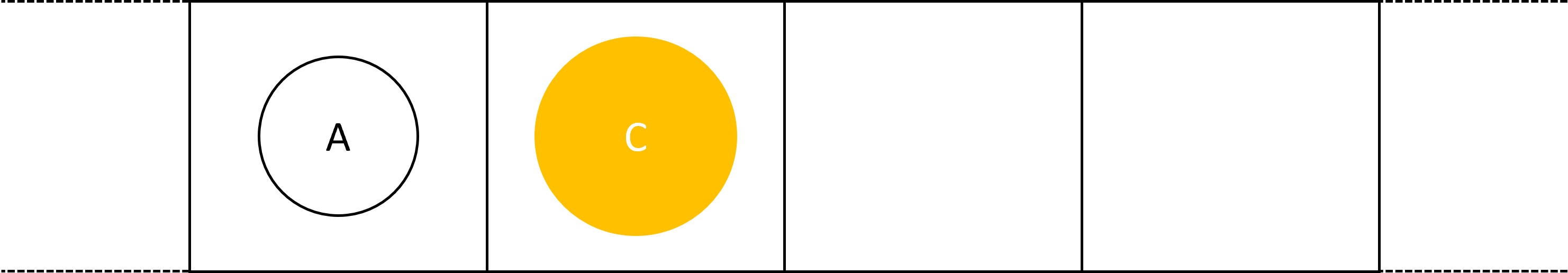}
    \caption{$w_{0}$}
    \label{fig:w0}
  \end{subfigure}%
  \hfill
  \begin{subfigure}{0.48\textwidth}
    \centering
    \includegraphics[width=\textwidth]{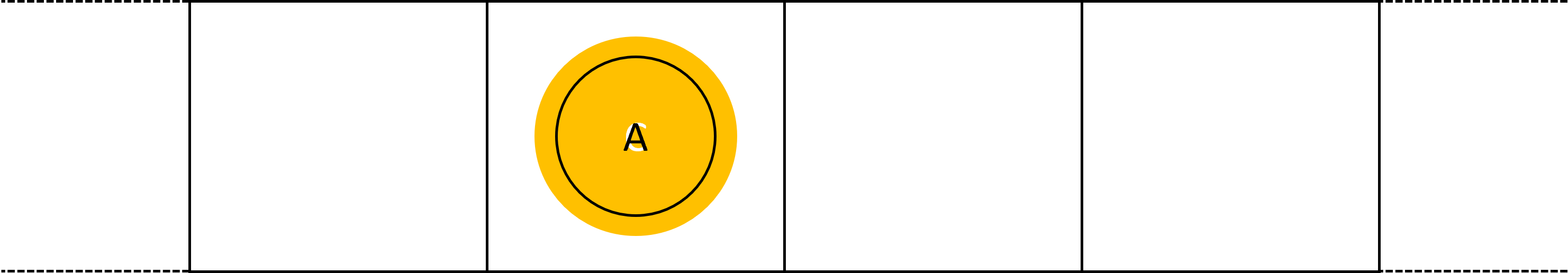}
    \caption{$w_{1}$}
    \label{fig:w1}
  \end{subfigure}%
  \vspace{0.5cm}
  \begin{subfigure}{0.48\textwidth}
    \centering
    \includegraphics[width=\textwidth]{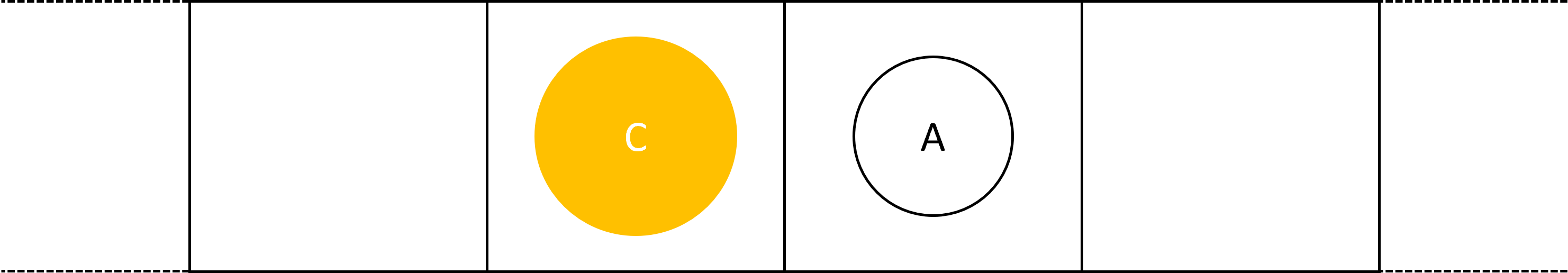}
    \caption{$w_{2}$}
    \label{fig:w2}
  \end{subfigure}%
  \hfill
  \begin{subfigure}{0.48\textwidth}
    \centering
    \includegraphics[width=\textwidth]{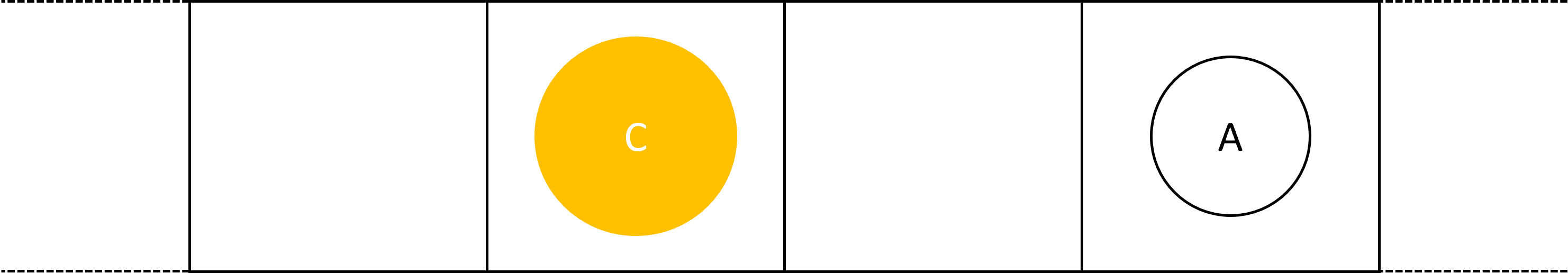}
    \caption{$w_{3}$}
    \label{fig:w3}
  \end{subfigure}%
  \vspace{0.5cm}
  \begin{subfigure}{0.48\textwidth}
    \centering
    \includegraphics[width=\textwidth]{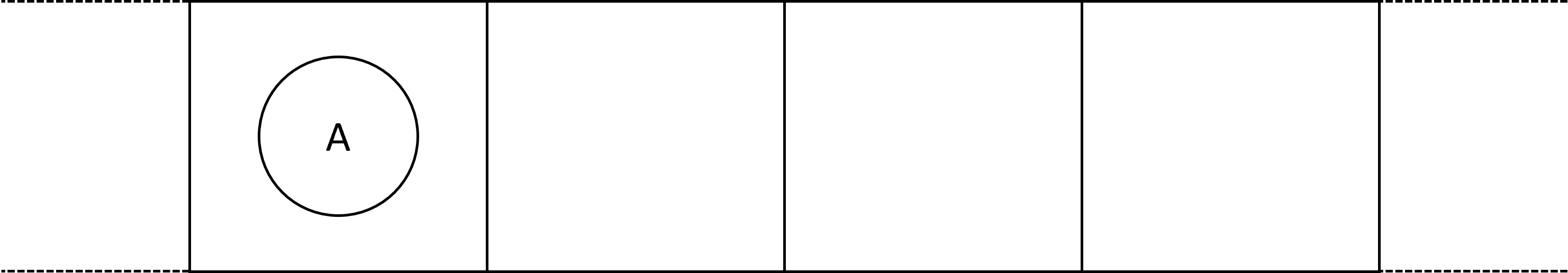}
    \caption{$w_{4}$}
    \label{fig:w4}
  \end{subfigure}%
  \hfill
  \begin{subfigure}{0.48\textwidth}
    \centering
    \includegraphics[width=\textwidth]{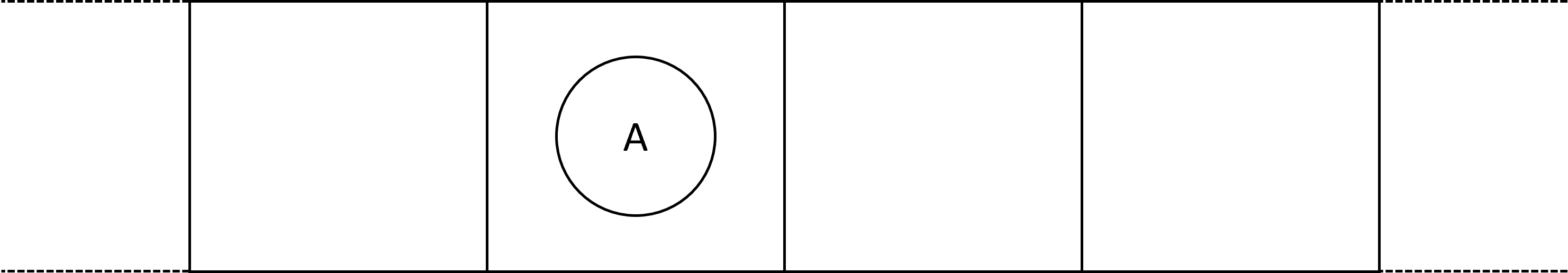}
    \caption{$w_{5}$}
    \label{fig:w5}
  \end{subfigure}%
  \vspace{0.5cm}
  \begin{subfigure}{0.48\textwidth}
    \centering
    \includegraphics[width=\textwidth]{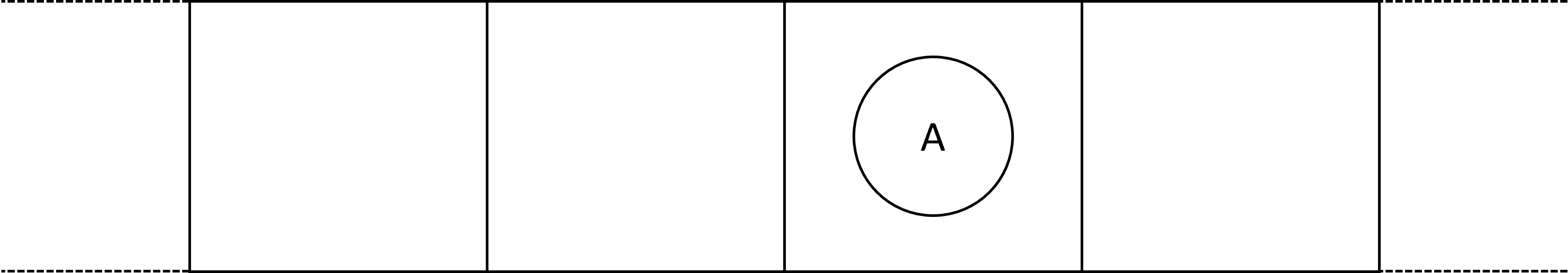}
    \caption{$w_{6}$}
    \label{fig:w6}
  \end{subfigure}%
  \hfill
  \begin{subfigure}{0.48\textwidth}
    \centering
    \includegraphics[width=\textwidth]{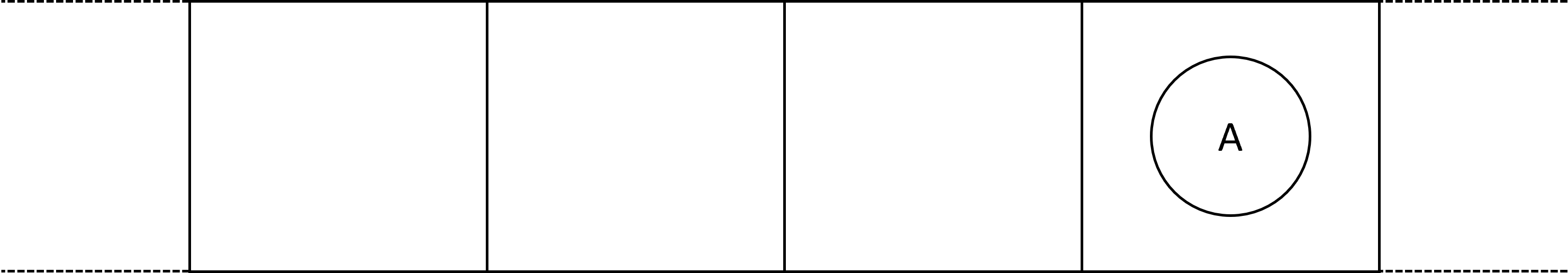}
    \caption{$w_{7}$}
    \label{fig:w7}
  \end{subfigure}%
  
  \caption{World states of a world containing an agent and a consumable.}
  \label{fig:consumable_world_states}
\end{figure}

There is not a consumable in every state, therefore if the agent performs the consume action in any state except $w_{1}$ then the effect will be the same as if the agent had performed the no-op action $1 \in A/\sim$ (see Figure \ref{fig:min-action-net-world-with-consumable-identity}).

\begin{figure}[H]
    \centering
    \includegraphics[width=0.7\linewidth]{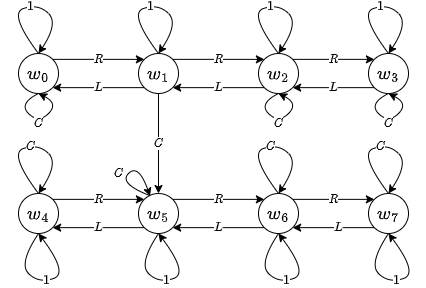}
    \caption{
   Minimum action network for world $\mathscr{W}_{consumable}$.
    }
    \label{fig:min-action-net-world-with-consumable-identity}
\end{figure}

The action Cayley table for world $\mathscr{W}_{consumable}$ with restricted actions treated as identity actions contains 64 elements.
As shown in Table \ref{tab:identity-consumable-properties}, $A/\sim$ is a monoid.

\begin{table}[H]
    \centering
    \begin{tabular}{c|c}
        \textbf{Property}   & \textbf{Present?} \\
        \hline
        Totality            & Y\\
        Identity            & Y\\
        Inverse             & N\\
        Associative         & Y\\
        Commutative         & N
    \end{tabular}
    \caption{Properties of the $A/\sim$ algebra.}
    \label{tab:identity-consumable-properties}
\end{table}

\begin{remark}
    Considering this example, we note that the irreversible action moves from one `reversible plane' to another.
    Perhaps we can treat an irreversible action as having a `reversible action affecting part' and a `world affecting part'; in this example, the reversible action affecting part would be the identity action since the reversible action network is unchanged by the irreversible action.
\end{remark}

\subsection{Worlds without inverse actions or unrestricted actions}\label{sec:Worlds without inverse actions or unrestricted actions}

In this section, we consider worlds that do not necessarily satisfy world condition \ref{wldcon:inverse-actions} or world condition \ref{wldcon:unrestricted-actions}.

\begin{proposition}\label{prp:all_worlds_give_small_category_action}
    Consider a world $\mathscr{W}$ with a set $W$ of world states and containing an agent with a set $A$ of actions.
    $*: (A/\sim) \times W \to W'$, where $W' \subseteq W$, is the action of a small category $A/\sim$ on $W$.
\end{proposition}
\begin{proof}
    (1) Associativity of $A/\sim$ is given by proposition \ref{prp:Asim-associative}.
    (2) Identity element of $A/\sim$ is given by proposition \ref{prp:Asim-identity}.
    Since $A/\sim$ satisfies properties (1) and (2), $A/\sim$ is a small category.
    Therefore $*$ is the action of a small category.
\end{proof}

\begin{remark}
    We can consider two equivalent perspectives of $*$:
    \begin{enumerate}
        \item $A/\sim$ is a small category and $*$ is a full action of $A/\sim$ on $W$.
        \item $A/\sim$ is a monoid and $*$ is a partial action of $A/\sim$ on $W$.
    \end{enumerate}
\end{remark}

\subsubsection{Example 1: reversible action-inhomogeneous world}\label{sec:masked reversible action-inhomogeneous world}

We once again consider the world $\mathscr{W}_{wall}$ (see Figure \ref{fig:2x2-cyclical-grid-world-wall-states} for world states); however, now instead of treating restricted actions like the identity action, we \textit{mask} the restricted actions.
Masking restricted actions involves not allowing the agent to perform the restricted actions - the restricted actions are hidden (masked) from the agent - and so, mathematically, we treat the restricted actions as \textit{undefined}; for example, if $\mathscr{W}_{wall}$ is in state $w_{0}$, then the agent cannot perform the $R$ action because $R * w_{0}$ is undefined (see Table \ref{tab:2x2-gridworld-minimum-transition-wall-masked} and Figure \ref{fig:2x2-cyclical-min-actions-wall-masked}.

\begin{table}[H]
    \centering
    \begin{tabular}{c|c c c c c}
                &  $1$      & $U$       & $D$       & $L$               & $R$\\
         \hline
        $w_{0}$ & $w_{0}$   & $w_{2}$   & $w_{2}$   & $w_{1}$           & undefined\\
        $w_{1}$ & $w_{1}$   & $w_{3}$   & $w_{3}$   & undefined         & $w_{0}$\\
        $w_{2}$ & $w_{2}$   & $w_{0}$   & $w_{0}$   & $w_{3}$           & $w_{3}$\\
        $w_{3}$ & $w_{3}$   & $w_{1}$   & $w_{1}$   & $w_{2}$           & $w_{2}$\\
    \end{tabular}
    \caption{Each entry in this table shows the outcome state of the agent performing the action given in the column label when in the world state given by the row label.}
    \label{tab:2x2-gridworld-minimum-transition-wall-masked}
\end{table}

\begin{figure}[H]
    \centering
    \includegraphics[width=0.5\linewidth]{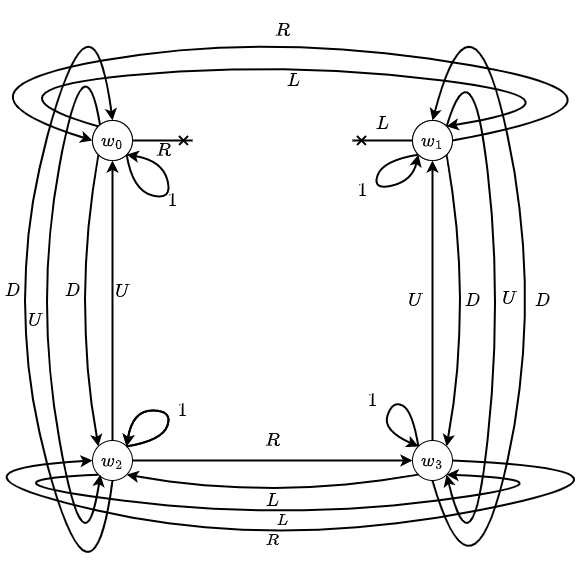}
    \caption{
    A diagram of the labelled transitions in Table \ref{tab:2x2-gridworld-minimum-transition-wall-masked}.
    }
    \label{fig:2x2-cyclical-min-actions-wall-masked}
\end{figure}

The action Cayley table for $\mathscr{W}_{wall}$ with the masked treatment of the walls contains 59 elements.
As shown in Table \ref{tab:masked-walls}, $A/\sim$ is a small category.

\begin{table}[H]
    \centering
    \begin{tabular}{c|c}
        \textbf{Property}   & \textbf{Present?} \\
        \hline
        Totality            & N\\
        Identity            & Y\\
        Inverse             & N\\
        Associative         & Y\\
        Commutative         & N
    \end{tabular}
    \caption{Properties of the $A/\sim$ algebra.}
    \label{tab:masked-walls}
\end{table}

\subsubsection{Example 2: irreversible action-inhomogeneous world}\label{sec:masked irreversible action-inhomogeneous world}

We will now apply the masking treatment of restricted actions to the world $\mathscr{W}_{consumable}$ (see Figure \ref{fig:min-action-net-world-with-consumable-masked} for world states); for example, if $\mathscr{W}_{consumable}$ is not in state $w_{1}$ then the consume action is undefined.

\begin{figure}[H]
    \centering
    \includegraphics[scale=0.5]{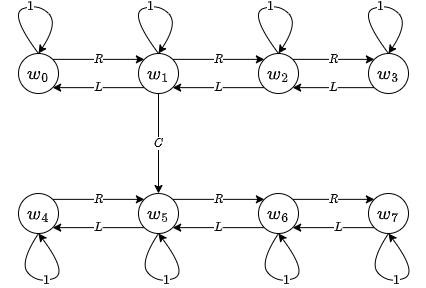}
    \caption{Minimum action network for a world containing an agent and a consumable.}
    \label{fig:min-action-net-world-with-consumable-masked}
\end{figure}

The action Cayley table for $\mathscr{W}_{consumable}$ with the masked treatment of the walls contains 20 elements.
As shown in Table \ref{tab:masked-consumable-properties}, $A/\sim$ is a small category.

\begin{table}[H]
    \centering
    \begin{tabular}{c|c}
        \textbf{Property}   & \textbf{Present?} \\
        \hline
        Totality            & N\\
        Identity            & Y\\
        Inverse             & N\\
        Associative         & Y\\
        Commutative         & N
    \end{tabular}
    \caption{Properties of the $A/\sim$ algebra.}
    \label{tab:masked-consumable-properties}
\end{table}

\begin{remark}
    For any world, once the agent has the action Cayley table for any initial world state it can produce the action Cayley table for any other initial world state by applying an action that transitions from the old initial world state to the new initial world state to every element of the action Cayley table.
\end{remark}

\begin{proposition}\label{prp:algebra_depends_on_action_treatment}
    The algebra $A/\sim$ of the actions of an agent depends on the treatment of the actions of the agent (\textit{e.g.}, masked treatment vs identity treatment) even if the world states remain the same.
\end{proposition}
\begin{proof}
    $A/\sim$ for the identity treatment of $\mathscr{W}_{wall}$ contains 26 elements, while $A/\sim$ for the masked treatment of $\mathscr{W}_{wall}$ contains 59 elements.

    Additionally, $A/\sim$ for the identity treatment of $\mathscr{W}_{consumable}$ contains 64 elements, while $A/\sim$ for the masked treatment of $\mathscr{W}_{consumable}$ contains 20 elements.
\end{proof}

We have also demonstrated that neither the identity treatment nor the masked treatment produces simpler algebra in every world.
For $\mathscr{W}_{wall}$, the identity treatment contains fewer elements than the masked treatment (26 elements vs 59 elements), while for $\mathscr{W}_{consumable}$, the masked treatment contains fewer elements than the identity treatment (20 elements vs 64 elements).

\medskip
In this section, we have selected several standard reinforcement learning scenarios that go beyond the expressive power of the symmetries represented in SBDRL as groups and that can nevertheless be accounted for using our general framework. It was not the purpose of this section to identify which algebraic structures would correspond with such representations, say semi-groups or groupoids for instance, but rather to establish a proof of concept, through examples generated using Algorithms \ref{alg:Generate state Cayley table} and \ref{alg:Generate action Cayley table}, that SBDRL is not robust or generalisable enough and that our framework subsumes and extends it to bring about substantial representations. 

\section{Generalising SBDRL}\label{sec:Generalising SBDRL}

So far, we have introduced a formal framework that allows agents to learn richer representations than those of SBDRL. In order to do so, we have proved in Section \ref{sec:Reproducing SBDRL} that SBDRL is subsumed by our framework, that is, that the symmetric disentangled representations that SBDRL accounts for are also derivable using a generative algorithm of the corresponding Cayley tables within our approach. In Section \ref{sec:Beyond SBDRs}, we have extended our study to simple yet challenging reinforcement learning cases where the conditions for the formation of groups, intrinsic to SBDRL, are not satisfied and that fall nevertheless under the remit of our framework. In other words, we have proved that our framework is more powerful than SBDRL and that it allows the representation of world transitions that can be structured in any algebra, in principle. Although formal, this was nevertheless an exploratory exercise, and, to express it with mathematical rigour we recur to category theory in this section. In order to do so, after some preliminary definitions, we first generalise SBDRL’s equivariance condition categorically for both single and, importantly, multi-object categories, and then the notion of disentanglement for transformations that cannot be described by groups. That is, unlike in Section \ref{sec:Beyond SBDRs}, the purpose of Section \ref{sec:Generalising SBDRL} is not to present further proof cases of the expressive power of our framework nor to revisit the examples in the previous section categorically (it would be impractical to go through the many algebraic structures that might result from our framework within and beyond SBDRL) but rather to provide a solid, formal characterisation that, in turn, broadens the rationale and capabilities of the original framework. 

A fundamental result in category theory is the Yoneda Lemma, which posits that the properties of mathematical objects are completely determined by their relationships to other objects \citep{riehl2017category,barr1990category}. This result is similar to the shift in perspective in Artificial Intelligence representations from studying world states to gaining insight into their structures by considering how they are transformed by the dynamics of the environment they are immersed in, including the actions of agents. Due to the Yoneda Lemma, category theory already incorporates this approach of considering the transformation properties of objects built in, which arguably makes it the natural choice for their representation. More generally, category theory has been extensively used in computer science in the analysis of relational databases and functional programming \citep{Fong_Spivak_2019, Milewski-2018}. 
Let us track back the main argument: generally speaking, symmetries define invariance, that is, impunity to possible alterations. Interestingly, the fact that the parts that are related by means of an equivalence relation correspond to the family of operations transforming the parts into each other while leaving the whole invariant satisfies the conditions for constituting a group. Consequently, it has traditionally been assumed that group theory is the language of symmetries and indeed groups have been customarily used to exploit symmetries in mathematics and physics, from E. Galois’ studies of the structures that underlie the number and form of the solutions for equations of arbitrary degrees, to the formulation of the Standard Model that classifies all elementary particles and their interactions according to their flavour, charge and colour symmetries as the $SU(3)\bigotimes SU(2)\bigotimes U(1)$ group.

Now, a group is simply a category with a single object where every morphism is an isomorphism, that is, the object only relates to itself. Thus, although groups are certainly sufficient to formalise homogeneous structures, there are nonetheless plenty of transformations that exhibit what we clearly recognise as symmetries that are nontrivial (think of a bowling ball; it may not display the symmetries of a perfect sphere, such as a basketball, but it nonetheless poses symmetries that can be identified and exploited). To detect such symmetries, we need to “zoom out”, and this is precisely why adopting a categorical approach in representation learning makes sense: category theory does consider objects, but its focus is on the relationships (morphisms) between those objects, rather than the internal structure of the objects themselves. The axioms of a category do not require objects to have elements or any specific internal structure. Category theory thus becomes the ideal tool for the study of structures that show partial symmetries and symmetries that apply to multiple objects, and those that result in higher-dimensional, increasingly more abstract categories. In other words, the ontology of algebraic groups is suitable to express symmetric structures that can be represented at the object-level in a straightforward way, but does not scale to other types that benefit from viewing the problem from a relation-level. Using the framework of Section \ref{sec:Mathematical framework for an agent in an environment}, in Section \ref{sec:Beyond SBDRs} we have identified several symmetric transformations that cannot be expressed in groups, the structure favoured by SBDRL. In this section, we formalise the conditions that accommodate such structures as categories. Our claim is that the insights provided in this section might be taken as the starting point for a deeper exploration of category theory as the formalism in which to establish and derive symmetries beyond group representations in reinforcement learning and, more generally, Artificial Intelligence \citep{gavranović2024positioncategoricaldeeplearning, gavranović2024fundamentalcomponentsdeeplearning}.


\subsection{Preliminaries}

Technically, a category consists of objects connected by arrows, which represent structure-preserving maps between the objects. Importantly, we can also relate categories via functors, which are ways to transform a category into another category while preserving the relationships between the objects and arrows of the domain category. In turn, natural transforms are ways to replace one functor with another while preserving the structure of the first. That is, we can iteratively construct higher-order categories, where (\textit{n})-categories become objects of (\textit{n+1})-categories. It is worth emphasizing that (a) the fundamental concept is the concept of morphism (transitions in our case) since an object is in fact defined as its collection of identity morphisms; (b) in relating categories these do not need to be of the same class, rather the power of category theory becomes visible when, through a process of categorisation, that is, by forgetting the details of its objects, we learn properties of a category by relating it to other category whose structure is well-known (similarly, we can recover the details of the domain category by a process of decategorisation); (c) as a result, in practice, categories provide a rich ontology to represent the concept of similarity beyond strict equivalence. That is, category theory allows us to compare (and build) a great variety of structures.  

We will now introduce relevant category theory concepts. The reader is referred to \citep{Spivak2014-SPICTF} for a more detailed account of category theory and its uses.

\begin{definition}[Category]\label{def:category}
    A category $\mathcal{C}$ consists of a class of objects, denoted $\textbf{Ob}(\mathcal{C})$, and, for each pair $x$, $y$ of objects, a class of morphisms $\alpha : x \to y$, denoted $C(x,y)$, satisfying the following:

    \begin{itemize}
        \item \textbf{Composition law.} Given two morphisms $\alpha \in \mathcal{C}(x,y)$ and $\beta \in \mathcal{C}(y,z)$ there exists a morphism $\beta \circ \alpha \in \mathcal{C}(x,z)$ called the composition of $\alpha$ and $\beta$.
        
        \item \textbf{Existence of units.} Given an object $x$, there exists a morphism denoted by $1_{x} \in \mathcal{C}(x,x)$ such that for any morphism $\alpha \in \mathcal{C}(x,a)$, $\alpha \circ 1_{x} = \alpha$ and for any morphism $\beta \in \mathcal{C}(b,x)$, $1_{x} \circ \beta = \beta$.
        
        \item \textbf{Associativity.} Given three morphisms $\alpha \in C(x,y)$, $\beta \in C(y,z)$, $\gamma \in C(z,u)$, then the following associative law is satisfied: $\gamma \circ ( \beta \circ \alpha) = (\gamma \circ \beta) \circ \alpha$.
    \end{itemize}
\end{definition}

\begin{definition}[Isomorphism]\label{def:isomorphism}
    A morphism $\alpha : x \to y$ in a category \textbf{C} is an isomorphism if there exists another morphism $\beta : y \to x$ in \textbf{C} such that $\beta \circ \alpha = 1_{x}$ and $\alpha \circ \beta = 1_{y}$.
    This can be denoted by $x \overset{\alpha}{\cong} y$.
\end{definition}

\begin{definition}[Group]
    A group is a category that has a single object and in which every morphism is an isomorphism (\textit{i.e.,} every morphism has an inverse).
\end{definition}

\begin{definition}[Small category]
    A category $\mathcal{C}$ is a \textit{small category} if its collection of objects $\textbf{Ob}(\mathcal{C})$ is a set, and the collection of morphisms $Hom_{\mathcal{C}}(X, Y)$, where $X,Y \in \textbf{Ob}(\mathcal{C})$ is also a set.
\end{definition}

\begin{definition}[Hom-set]
    Given objects $x$ and $y$ in a small category $\mathcal{C}$, the hom-set $Hom_{\mathcal{C}}(x,y)$  is the collection of all morphisms from $x$ to $y$.
    A category is said to be \textit{small} if each of its hom-sets is a set instead of a proper class.
\end{definition}

\begin{definition}[(covariant) Functor]
    A \textit{functor} is a structure-preserving map between two categories.
    For two categories $\mathcal{A}$ and $\mathcal{B}$, a functor $F: \mathcal{A} \to \mathcal{B}$ from $\mathcal{A}$ to $\mathcal{B}$ assigns to each object in $\mathcal{A}$ an object in $\mathcal{B}$ and to each morphism in $\mathcal{A}$ a morphism in $\mathcal{B}$ such that the composition of morphisms and the identity morphisms are preserved.
    A functor transforms objects and morphism from one category to another in a way that preserves the structure of the original category.

    More precisely, a functor $F: \mathcal{A} \to \mathcal{B}$ consists of two maps:
    \begin{enumerate}
        \item A map that assigns an object $F(A)$ in $\mathcal{B}$ to each object $A$ in $\mathcal{A}$.

        \item A map that assigns a morphism $F(f)$ in $\mathcal{B}$ to each morphism $f$ in $\mathcal{A}$ such that:
        \begin{enumerate}
            \item $F$ respects composition for any two composable morphisms $f,g \in \mathcal{A}$, $F(fg) = F(f) F(g)$.

            \item $F$ preserves identities: for any object $A$ in $\mathcal{A}$, $F(id_{A}) = id_{F(A)}$.
        \end{enumerate}
    \end{enumerate}
\end{definition}
    
    \begin{definition}[Functor category]
        For categories $\mathcal{C}$ and $\mathcal{D}$, the functor category denoted $D^{\mathcal{C}}$ or $[\mathcal{C}, \mathcal{D}]$ is the category whose: (1) objects are functors $F: \mathcal{C} \to \mathcal{D}$, and (2) morphisms are natural transforms between these functors.
\end{definition}

\begin{definition}[Natural transform]
    For categories $\mathcal{C}$ and $\mathcal{D}$ and functors $F,G: \mathcal{C} \to \mathcal{D}$, a \textit{natural transform} $\alpha: F \Rightarrow G$ between $F$ and $G$ is an assignment to every object $x$ in $\mathcal{C}$ of a morphism $\alpha_{x}: F(x) \to G(x)$ in $\mathcal{D}$ such that for any morphism $f: x \to y$ in $\mathcal{C}$, the following diagram commutes:

\[\begin{tikzcd}
	{F(x)} && {F(y)} \\
	\\
	{G(x)} && {G(y)}
	\arrow["{F(f)}", from=1-1, to=1-3]
	\arrow["{\alpha_{x}}"', from=1-1, to=3-1]
	\arrow["{\alpha_{y}}", from=1-3, to=3-3]
	\arrow["{G(f)}"', from=3-1, to=3-3]
\end{tikzcd}\]
 \end{definition}

\begin{definition}[Delooped category]
    Given an algebraic structure $A$, we can construct the \textit{delooped category} $\textbf{B}A$ whose morphisms correspond to the elements of $A$ with the relevant composition: $A \xrightarrow{deloop} \textbf{B}A$.
    $\textbf{B}$ id called the \textit{base} of the category and contains the objects.
    The number of objects in $\textbf{B}$ is the number of objects necessary for the morphisms of $\textbf{B}A$ to correspond to the elements of $A$.
\end{definition}

\begin{definition}[Monoid]\label{def:monoid}
    A monoid is a category with a single object.
\end{definition}

\begin{definition}[Categorical product]
    The \textit{categorical product} of two categories $C_{1}$ and $C_{2}$ is the category $C_{1} \times C_{2}$.
    The objects of $C_{1} \times C_{2}$ are the pairs of objects $(B_{1}, B_{2})$, where $B_{1}$ is an object in $C_{1}$ and $B_{2}$ is an object in $C_{2}$.
    The morphisms in $C_{1} \times C_{2}$ are the pairs of morphisms $(f_{1}, f_{2})$, where $f_{1}$ is a morphism in $C_{1}$ and $f_{2}$ is a morphism in $C_{2}$.
    The composition of morphisms is defined component-wise.
\end{definition}

\begin{definition}[Sub-functors of sub-categories]
    We can define a functor $F: C_{1} \times C_{2} \to C$ as follows:
    \begin{enumerate}
        \item For each object $(B_{1}, B_{2})$ in $C_{1} \times C_{2}$, $F(B_{1}, B_{2})$ is the object in $C$ that corresponds to the pair $(B_{1}, B_{2})$.
    
        \item For each morphism $(f_{1}, f_{2}): (B_{1}, B_{2}) \to (B'_{1}, B'_{2})$ in $G_{1} \times G_{2}$, $F(f_{1}, f_{2})$ is the morphism in $G$ that corresponds to the pair $(f_{1}, f_{2})$.
    \end{enumerate}
    
    We can now define sub-functors $F_{1}$ of $G_{1}$ and $F_{2}$ of $G_{2}$ on an object $B$ of $G$ as follows:
    \begin{enumerate}
        \item $F_{1}(g_{1}, B) = F(g_{1}, 1_{G_{2}})(B)$ for all $g_{1} \in G_{1}$.
        \item $F_{2}(g_{2}, B) = F(1_{G_{1}}, g_{2})(B)$ for all $g_{1} \in G_{1}$.
    \end{enumerate}
    
    We can decompose $F$ into the sub-functors $F_{1}$ and $F_{2}$ as $F = F_{1} \times F_{2}$, if there is a decomposition $B = B_{1} \times B_{2}$ of $B$ into two sub-objects $B_{1}$ and $B_{2}$ such that:
    \begin{enumerate}
        \item For all $g_{1} \in G_{1}$ and $B_{2} \in G_{2}$, we have $F_{1}(g_{1}, B_{1} \times B_{2}) = F_{1}(g_{1}, B_{1} \times B_{2}$.
    
        \item For all $g_{2} \in G_{2}$ and $B_{1} \in G_{1}$, we have $F_{2}(g_{2}, B_{1} \times B_{2}) = B_{1} \times F_{2}(g_{2}, B_{2})$.
    \end{enumerate}
\end{definition}

\subsection{The equivariance condition}

In this section, we use category theory to generalise \cite{Higgins2018}'s equivariance condition.
In doing so we show that the equivariance condition can apply to worlds where the actions of an agent cannot be fully described by groups.

\subsubsection{Group equivariance in category theory}\label{sec:Group equivariance in category theory}

We will now convert \cite{Higgins2018}'s group equivariance condition into the language of category theory.

A left group action $G \times S \to S$ is a homomorphism from a group $G$ to the group of bijections of a set $S$ that $G$ is acting upon.
Since $S$ is itself an object in the category $\textbf{Set}$ of sets, and bijections from $S$ to itself are the invertible morphisms in $hom_{\textbf{Set}}(S, S)$, the left group action $G \times S \to S$ is an object of the category $\textbf{Set}^{G}$ of covariant functors from $G$ to $\textbf{Set}$.

Since the objects of the category $\textbf{Set}^{G}$ are the functors from $G$ to $\textbf{Set}$\footnote{The objects of the category $\textbf{Set}^{G}$ are the maps from the morphisms in $G$ to the morphisms in $\textbf{Set}$.}, then the morphisms of $\textbf{Set}^{G}$ are natural transforms between these functors.
For the left action $A_1: G \to \textbf{Set}$ that maps the single object $b$ of $G$ to a set $S_1$ and the left action $A_2: G \to \textbf{Set}$ that maps the single object $b$ of $G$ to a set $S_2$, the natural transform $\eta : A_1 \to A_2$ has a single component $\eta_{b} : S_1 \to S_2$ because there is a single object $b$ in $G$.
Every morphism, which is a group element $g \in G$, must satisfy the naturality condition $\eta_{b}(g \cdot_{S_1} s) = g \cdot_{S_2} \eta_{b}(s)$ for all $s \in S_1$, where $\cdot_{S_1}$ denotes the action of $G$ on set $S_1$ and $\cdot_{S_2}$ denotes the action of $G$ on set $S_2$.

Now consider an agent with a set $A$ of actions and a set $Z$ of representation states in a world $\mathscr{W}_{0}$ that has a set $W$ of world states and that satisfies world conditions \ref{wldcon:unrestricted-actions} and \ref{wldcon:inverse-actions}.
For $\mathscr{W}_{0}$, $A/\sim$ is a group.
Since the set of world states $W$ and the set of representation states $Z$ are both acted on by the same group $A/\sim$, there are two group actions $*_{W}: (A/\sim) \times W \to W$ and $*_{Z}: (A/\sim) \times Z \to Z$, and both group actions give functors from $A/\sim$ to objects ($W$ and $Z$) in the category $\textbf{Set}$.
A structure-preserving map between these functors is a natural transform with the single component $\eta_{b}: W \to Z$ that satisfies $\eta_{b}(a *_{W} w) = a *_{Z} \eta_{b}(w)$ for all $w \in W$ and for all $a \in A/\sim$, where $*_{W}$ denotes the action of $A/\sim$ on set $W$ and $*_{Z}$ denotes the action of $A/\sim$ on set $Z$.
This component $\eta_{b}$ is \cite{Higgins2018}'s equivariant condition ($g \cdot_{Z} f(w) = f(g \cdot_{W} w)$) in the language of category theory.

\subsubsection{Equivariance for single-object categories}

We will now take the category theory argument we used to derive the group equivariance condition for worlds where $*_{W}: (A/\sim) \times W \to W$ is a group action (reproducing \cite{Higgins2018}'s equivariant condition) and generalise our argument to worlds where $*_{W}: (A/\sim) \times W \to W$ is the (full) action of any algebraic structure $A/\sim$ that can be delooped to form a single-object category $\textbf{B}(A/\sim)$\footnote{Formally, in section \ref{sec:Group equivariance in category theory} we follow the argument in this section and deloop the group $A/\sim$ to form the associated category $\textbf{B}(A/\sim)$.}.
Any single-object category is a monoid (definition \ref{def:monoid}), therefore an equivariance condition for single-object categories will hold for any world where $A/\sim$ is a monoid.
The derivation of this condition is trivially the same argument as given for the group action case since we did not require that the morphisms in the single-object category with morphisms giving $A/\sim$ be isomorphisms.

\paragraph{1. Setup}
Now consider an agent with a set $A$ of actions and a set $Z$ of representation states in a world $\mathscr{W}_{1}$ that has a set $W$ of world states and where $A/\sim$ is a monoid with elements $a_{1}, a_{2}, ..., a_{n}$.

\paragraph{2. Category of actions}
Let $\textbf{B}(A/\sim)$ be the delooped category of $A/\sim$ and let $\textbf{B} = \{b\}$ be the set of objects of $\textbf{B}(A/\sim)$.
The set of morphisms of $\textbf{B}(A/\sim)$ is $A/\sim$.

\paragraph{3. Functors for actions on sets}
The action $*_{W}: (A/\sim) \times W \to W$ of $A/\sim$ on the set $W$ gives a functor $\rho: \textbf{B}(A/\sim) \to W$, where $\rho$ encodes the properties of the algebraic structure of $(A/\sim)$.
Similarly, the action $*_{Z}: (A/\sim) \times Z \to Z$ of $A/\sim$ on the set $Z$ gives a functor $\tau: \textbf{B}(A/\sim) \to Z$, where $\tau$ encodes the properties of the algebraic structure of $(A/\sim)$.

\paragraph{4. Structure-preserving morphism (natural transform)}
We want the property that $A/\sim$ acts on $W$ and $Z$ in the same way, so we want to preserve the structure between the functors $\rho$ and $\tau$.
The objects of the functor category $\textbf{Set}^{\textbf{B}(A/\sim)}$ are the functors, including $\rho$ and $\tau$, from $\textbf{B}(A/\sim)$ to $\textbf{Set}$; the morphisms (structure-preserving maps) between the objects of $\textbf{Set}^{\textbf{B}(A/\sim)}$ are natural transforms.
The structure-preserving map between $\rho$ and $\tau$ is the natural transform $\eta: W \to Z$ with the single component $\eta_{b}: W \to Z$ that satisfies $\eta_{b}(a *_{W} w) = a *_{Z} \eta_{b}(w)$ for all $w \in W$ and for all $a \in A/\sim$, where $*_{W}$ denotes the action of $A/\sim$ on set $W$ and $*_{Z}$ denotes the action of $A/\sim$ on set $Z$.

\paragraph{5. Generalised equivariance condition}
The generalised equivariance condition for the single-object case is:
\begin{center}
    $\eta_{b}(a *_{W} w) = a *_{Z} \eta_{b}(w)$ for all $w \in W$ and for all $a \in A/\sim$.
\end{center}
In other words, the diagram
\[\begin{tikzcd}
	w && {a *_{W} w} \\
	\\
	{\eta_{b}(w)} && {\eta_{b}(a *_{W} w) = a *_{Z} \eta_{b}(w)}
	\arrow["{a *_{W}}", from=1-1, to=1-3]
	\arrow["{a *_{Z}}", from=3-1, to=3-3]
	\arrow["{\eta_{b}}"', dashed, from=1-3, to=3-3]
	\arrow["{\eta_{b}}"', dashed, from=1-1, to=3-1]
\end{tikzcd}\]
commutes.

\medskip
We have now generalised \cite{Higgins2018}'s group equivariance condition to worlds where $*_{W}: (A/\sim) \times W \to W$ is a (full) monoid action.
We have shown that the equivariance condition depends only on the number of objects in the delooped category of $A/\sim$, and so the equivariance condition for monoid action $*$ is structurally the same as for group action $*$.
Our derived equivariance condition is valid in some worlds where agents can perform irreversible actions since monoids can have elements with no inverse; the inability of \cite{Higgins2018}'s original formalism to deal with irreversible actions has been explicitly stated \cite[page 4]{caselles2019symmetry}.
In fact, from proposition \ref{prp:wc1_gives_monoid_action}, this equivariance condition for the single-object category case is valid for any world that satisfies world condition \ref{wldcon:unrestricted-actions}.

\subsubsection{Equivariance for multi-object categories}

We will now generalise our argument to derive equivariance conditions for worlds where $*_{W}: (A/\sim) \times W \to W$ is the action of any algebraic structure $A/\sim$ that can be delooped to form any small category $\textbf{B}(A/\sim)$.
Let us adapt our argument for single-object categories from the previous section step-by-step to the multi-object category case:

\paragraph{1. Setup}
Consider an agent with a set $A$ of actions and a set $Z$ of representation states in a world $\mathscr{W}_{2}$ that has a set $W$ of world states and where $A/\sim$ is a small category with elements $a_{1}, a_{2}, ..., a_{n}$.

\paragraph{2. Category of actions}
Let $\textbf{B}(A/\sim)$ be the delooped category of $A/\sim$ and let $\textbf{B} = \{b_{1}, b_{2}, ..., b_{m}\}$ be the set of objects of $\textbf{B}(A/\sim)$.
For the multi-object category case, the set of morphisms $f: b_{i} \to b_{j}$ of $\textbf{B}(A/\sim)$ is not the elements $a_{1}, a_{2}, ..., a_{n}$ of $A/\sim$ - the elements from the algebraic structure $A/\sim$ do not appear directly as components of morphisms; instead they inform how the objects and morphisms within the category interact.
The morphisms $f: b_{i} \to b_{j}$ of $\textbf{B}(A/\sim)$ can be thought of as arrows connecting object $b_{i}$ to object $b_{j}$.
Each morphism is then labelled by an element of $A/\sim$; this labelling indicates how the element from $A/\sim$ maps $b_{i}$ to $b_{j}$ within the categorical framework\footnote{In the single-object category case ($\textbf{B} = b$), all morphisms are between the same object (they are \textit{endomorphisms}) and so there is a one-to-one correspondence between these endomorphisms $f: b \to b$ and the elements of $A/\sim$; this means we can treat the elements of $A/\sim$ as the morphisms of $\textbf{B}(A/\sim)$.}.
This is analogous to the treatment of transitions (analogous to morphisms) and actions (analogous to elements of $A/\sim$) given in Section \ref{sec:Agent actions as labelled transitions}.

\paragraph{3. Functors for actions on sets}
As before, the action $*_{W}: (A/\sim) \times W \to W$ of $A/\sim$ on the set $W$ gives a functor $\rho: \textbf{B}(A/\sim) \to W$, where $\rho$ encodes the properties of the algebraic structure of $(A/\sim)$.
Similarly, the action $*_{Z}: (A/\sim) \times Z \to Z$ of $A/\sim$ on the set $Z$ gives a functor $\tau: \textbf{B}(A/\sim) \to Z$, where $\tau$ encodes the properties of the algebraic structure of $(A/\sim)$.
The functors $\rho$ and $\tau$ now have components, $\rho(b_{i})$ and $\tau(b_{i})$, for each object $b_{i} \in \textbf{B}$ as well as components, $\rho(f)$ and $\tau(f)$, for each morphism $f \in \textbf{B}(A/\sim)$.

\paragraph{4. Structure preserving morphisms (natural transform)}
We again want the property that $A/\sim$ acts on $W$ and $Z$ in the same way, so we want to preserve the structure between the functors $\rho$ and $\tau$.
The objects of the functor category $\textbf{Set}^{\textbf{B}(A/\sim)}$ are the functors, including $\rho$ and $\tau$, from $\textbf{B}(A/\sim)$ to $\textbf{Set}$; the morphisms between the objects of $\textbf{Set}^{\textbf{B}(A/\sim)}$ are natural transforms.
So we again want to preserve the structure between the functors $\rho$ and $\tau$ using a natural transform $\eta: W \to Z$.
However, since there are multiple objects in $\textbf{B}(A/\sim)$, our natural transform $\eta$ has a component for each object in $\textbf{B}$.
For each object $b_{i} \in \textbf{B}$, there is a component $\eta_{b_{i}}: W \to Z$ such that it satisfies $\eta_{b_{i}}(a *_{W} w) = a *_{Z} \eta_{b_{i}}(w)$ for all $w \in W$ and for all $a \in A/\sim$, where $*_{W}$ denotes the action of $A/\sim$ on set $W$ and $*_{Z}$ denotes the action of $A/\sim$ on set $Z$.

\begin{remark}
   There is a different $\eta_{b_{i}}$ for each object $b_{i} \in \textbf{B}$.
\end{remark}

\paragraph{5. Generalised equivariance condition}
 The generalised equivariance condition for the multi-object category case is a collection of equivariance conditions, one for each object $b_{i} \in \textbf{B}$, with the corresponding natural transformation components $\eta_{b_i}$ satisfying the condition:
 \begin{center}
      $\eta_{b_i}(a *_{W} w) = a *_{Z} \eta_{b_i}(w)$ for all $w \in W$ and for all $a \in A/\sim$.
 \end{center}

In other words, the diagram 
\[\begin{tikzcd}
	w && {a *_{W} w} \\
	\\
	{\eta_{b_{i}}(w)} && {\eta_{b_{i}}(a *_{W} w) = a *_{Z} \eta_{b_{i}}(w)}
	\arrow["{\eta_{b_{i}}}"', dashed, from=1-1, to=3-1]
	\arrow["{a *_{W}}", from=1-1, to=1-3]
	\arrow["{a *_{Z}}", from=3-1, to=3-3]
	\arrow["{\eta_{b_{i}}}"', dashed, from=1-3, to=3-3]
\end{tikzcd}\]
 commutes.

This generalised equivariance condition ensures that the action of $A/\sim$ on $W$ and $Z$ is preserved consistently across all objects of $\textbf{B}(A/\sim)$.

\begin{remark}
    For each object $b_i$, we have an associated natural transformation component $\eta_{b_i}: W \to Z$.
    This component represents how the action of $A/\sim$ on the set $W$ relates to the set $Z$ for the specific object $b_i$.
    In other words, $\eta_{b_i}$ encodes how the algebraic structure $A/\sim$ interacts with that particular object $b_i$.
\end{remark}

We have generalised \cite{Higgins2018}'s group equivariance condition to worlds where $*_{W}: (A/\sim) \times W \to W$ is a (full) small category action or, equivalently, a partial monoid action.
Our derived equivariance condition is valid in some worlds where agents have actions that are not defined by some world states.

\subsection{Disentangling}

We will now generalise the definition of disentanglement given by \cite{Higgins2018} to give a definition that works for worlds with transformations that cannot be described by groups.
We prove that the equivalence condition can be disentangled.

\subsubsection{Disentangling of group actions in Category Theory terms}

In Category Theory, the group action $\cdot: G \times X \to X$, where $G$ is a group and $X$ is a set, gives a functor $\alpha: \textbf{B}G \times \textbf{Set} \to \textbf{Set}$, where $\textbf{B}G$ is the (single-object) delooped category of $G$ and $\textbf{Set}$ the category of sets.
$\alpha$ satifies:
\begin{enumerate}
    \item $\alpha_{1} = \text{id}_{X}$, where $1$ is the identity element of $G$.
    \item $\alpha_{g_{1}} \circ \alpha_{g_{2}} = \alpha_{g_{1}g_{2}}$ for all $g_{1}, g_{2} \in G$.
\end{enumerate}
The set $X$ is an object in the category $\textbf{Set}$.

If $G$ decomposes into the direct product of subgroups $G = G_{1} \times ... \times G_{n}$, then $\textbf{B}G$ decomposes into a categorical product of the relevant delooped categories $\textbf{B}_{i}G_{i}$ as $\textbf{B}G = \textbf{B}_{1}G_{1} \times ... \times \textbf{B}_{n}G_{n}$.
The object of $\textbf{B}_{1}G_{1} \times ... \times \textbf{B}_{n}G_{n}$ is the $n$-tuple of objects $(b_{1},...,b_{n})$ where each $b_{i}$ is the (single) object in $\textbf{B}_{i}$.
The morphisms of $\textbf{B}_{1}G_{1} \times ... \times \textbf{B}_{n}G_{n}$ are the $n$-tuples of morphisms $(f_{1}, ..., f_{n})$ where $f_{i}$ is a morphism $f_{i}: \textbf{B}_{i}G_{i} \to \textbf{B}_{i}G_{i}$.
The composition of morphisms is defined component-wise.

We say the functor $\alpha: G \times X \to X$ is \textit{disentangled} with respect to a decomposition $G = G_{1} \times ... G_{n}$ of $G$ if:
\begin{enumerate}
    \item There exists a decomposition $X_{1} \times ... X_{n}$ and an isomorphism $X \cong X_{1} \times ... X_{n}$.

    \item $\alpha$ can be decomposed into sub-functors $\alpha = \alpha_{1} \times \alpha_{2} \times ... \times \alpha_{n}$, where each sub-functor $\alpha_{i}: (\textbf{B}_{i}G_{i}) \times X_{i} \to X_{i}$ corresponds to the action of the subgroup $G_{i}$ on the subspace $X_{i}$ that is invariant under all other subgroups: $\alpha(g, x) = \alpha(g_{1} \times ... \times g_{n}, x_{1} \times ... \times x_{n}) = \alpha_{1}(g_{1}, x_{1}) \times ... \times \alpha_{n}(g_{n}, x_{n})$.
\end{enumerate}

If $G$ has additional structure such as a vector space or a topological space, then we can consider a subcategory  consisting of objects with that structure, and require the functors $\alpha$ and $\alpha_{i}$ to preserve that structure.

We will now extend the definition of disentangling given by \cite{Higgins2018} to worlds with transition algebras that are described by multi-object categories.
First, we convert \cite{Higgins2018}'s definition of disentangling into category theory terms.
We then extend this to any single-object category, and finally to any multi-object category.

The definition of a disentangled functor given previously extends to the case of any algebra, including those expressed as multi-object categories, by letting $G$ be the algebra $A/\sim$ and letting $\textbf{B}$ contain one or more objects.
These changes do not change the definition of a disentangled functor given previously.

\subsubsection{Disentangling and the equivariance condition}

We now want to see how disentangling affects the equivariance condition. Specifically, we want to see if the natural transform $\eta$ disentangles.

Let $\rho: \textbf{B}(A/\sim) \to W$ and $\tau: \textbf{B}(A/\sim) \to Z$ be disentangled functors of the actions $*_{W}: (A/\sim) \times W \to W$ and $*_{Z}: (A/\sim) \times Z \to Z$ respectively.

By definition, there are decompositions $W = \prod_{i=1}^{n} W_{i}$ and $Z = \prod_{i=1}^{n} Z_{i}$ where each $W_{i}$ (respectively $Z_{i}$) is fixed by the sub-functor $\rho_{j\neq i}$ (respectively $\tau_{j\neq i}$) and is only affected by the functor $\rho_{i}$ (respectively $\tau_{i}$).
Now let $\eta: W \to Z$ be a natural transform between $\rho$ and $\tau$ (\textit{i.e.}, for each $a, w \in (A/\sim) \times W$, we have $\eta(\rho(a, w)) = \tau(a, \eta(w))$).
We want to show that $\eta$ is itself disentangled with respect to the decomposition of $\textbf{B}(A/\sim)$.

Let $w = (w_{1},..., w_{n}) \in W$ be a point in the decomposition of $W$, let $a = (a_{1},...,a_{n}) \in A/\sim$ be a point of decomposition of $\textbf{B}(A/\sim)$, and let $a_{i} \in (A/\sim)_{i}$ be an element in the $i$th factor of $A/\sim$.
Say we have a decomposition of $\eta$ such that $\eta = (\eta_{1}, ..., \eta_{n})$ where $\eta_{i}$ is the $i$th component of $\eta$.
Therefore, we have $\eta(w) = (\eta_{1}(w_{1}), ..., \eta_{n}(w_{n}))$ where $\eta_{i}(w) \in Z_{i}$.
Consider

\begin{align*}
    \eta(\rho(a_{i}, w)) &= \eta((\rho_{1}(a_{i}, w_{1}),..., \rho_{i}(a_{i}, w_{i}), ..., \rho_{n}(a_{i}, w_{n}))) \\
    &= (\eta_{1}(\rho_{1}(a_{i}, w_{1})), ..., \eta_{i}(\rho_{i}(a_{i}, w_{i})), ... \eta_{n}(\rho_{n}(a_{i}, w_{n})))
\end{align*}

Since $W_{i}$ is fixed by the action of $(A/\sim)_{j \neq i}$, we have $\rho_{j}(a_{i}, w_{j}) = w_{j}$ for $i \neq j$, and so:
\begin{align}
    \eta(\rho(a_{i}, w)) = (\eta_{1}(w_{1}), ..., \eta(\rho_{i}(a_{i}, w_{i})), ..., \eta_{n}(w_{n}))
    \label{eqn:disentangling-equivalence-W}
\end{align}

Now consider
\begin{align*}
    \tau(a_{i}, \eta(w)) &= (\tau_{1}(a_{i}, \eta_{1}(w_{1})),..., \tau_{i}(a_{i}, \eta_{i}(w_{i})), ..., \tau_{n}(a_{i}, \eta_{n}(w_{n})))
\end{align*}

Since $Z_{i}$ is fixed by the action of $(A/\sim)_{j \neq i}$, we have $\tau_{j}(a_{i}, \eta(w_{j})) = \eta(w_{j})$ for $i \neq j$, and so:
\begin{align}
    \tau(a_{i}, \eta(w)) = (\eta_{1}(w_{1}),..., \tau_{i}(a_{i}, \eta_{i}(w_{i})), ..., \eta_{n}(w_{n}))
    \label{eqn:disentangling-equivalence-Z}
\end{align}

We can see from combining equations \ref{eqn:disentangling-equivalence-W} and \ref{eqn:disentangling-equivalence-Z} that if $\eta_{i}(\rho_{i}(a_{i}, w_{i})) = \tau_{i}(a_{i}, \eta_{i}(w_{i}))$ (\textit{i.e.}, the relevant diagrams commute for each component $\eta_{i}$ of $\eta$), then $\eta_{i}$ is disentangled with respect to the action of $(A/\sim)_{i}$ on $W_{i}$ and $Z_{i}$, and therefore $\eta$ is itself disentangled with respect to the decomposition of $A/\sim$.

We have shown that $\eta$ can be decomposed into sub-natural transforms $\eta_{1}, ..., \eta_{n}$, where each $\eta_{i}$ is a natural transform between sub-functions $\rho_{i}$ and $\tau_{i}$.
Each $\eta_{i}: W_{i} \to Z_{i}$ preserves the structure of the corresponding subcategory of $W$ and $Z$; therefore, if $W_{i}$ and $Z_{i}$ have additional structure, then the sub-natural transform $\eta_{i}$ must preserve that structure.
This opens up the possibility of a decomposition of $A/\sim$ and therefore $W$, $Z$, $\rho$, $\tau$, and $\eta$ such that some additional structure is confined to a single set of components $(A/\sim)_{i}$, $W_{i}$, $Z_{i}$, $\rho_{i}$, $\tau_{i}$, and $\eta_{i}$.

\paragraph{Components of $\eta$ when $\textbf{B}(A/\sim)$ is a multi-object category}

If $\textbf{B}(A/\sim)$ is a multi-object category, then the components of the sub-natural transform $\eta_{i}$ between functors $\rho_{i}$ and $\tau_{i}$ are ${\eta_{i}}_{b_{j}}: \rho_{i}(b_{j}) \to \tau_{i}(b_{j})$ for each object $b_{j} \in \textbf{B}$.

\medskip
 In this section, we first generalised \cite{Higgins2018}'s equivariance condition using category theory; to do so, we converted our symmetry-based representation argument into category theory - this gave us insight into the relationship between the algebraic form and the categorical form.
Category theory naturally generalises the group-equivariance condition given by \citep{Higgins2018} to any algebraic structure that has a categorical interpretation.
We then showed that a form of the equivariance condition exists for a world with transformations that form \textit{any} algebra.
We also demonstrated that the equivariance condition is, in fact, a fundamental feature of category theory: the natural transform.

Next, we converted \cite{Higgins2018}'s definition of disentangling into category theory terms; category theory then provided a natural generalisation of the definition of disentangling to worlds with transformations that form any algebra.
Finally, we explored the interplay between the generalised equivariance condition and the generalised definition of disentangling.
We concluded that disentangled sub-algebras can each have their own individual equivariance sub-conditions, and thus, the learning of these sub-algebras, as well as their applications, can be treated independently.
This result has important implications for learning algorithms; for example, since each disentangled subspace has its own individual equivariance condition, the learning of each subspace is also independent, and so different learning algorithms could be used on each disentangled subspace.

\section{Discussion $\&$ conclusion}\label{sec:discussion and conclusion}

The question of what features of the world should be present in a ‘good’ representation that improves the performance of an artificial agent for a variety of tasks is key in representation learning.  \citep{Higgins2018} proposes that the symmetries of the world are important structures that should be present in an agent’s representation of that world. They formalise their proposal using group theory as SBDRL, which is made up of a disentangling condition that defines transformations as commutative subgroups and a group equivariant condition.

In this paper, we have taken that programme one step further and posit that the relationships of transformations of the world due to the actions of the agent should be included in the agent’s representation of the world and not just the actions that form group symmetries. We then used this framework to derive and identify limitations in SBDRL. This approach reports two benefits: (1) it shows that the framework we put forward encompasses SBDRL, and (2) it identifies worlds where SBDRL cannot describe the relationships between the transformations of a world due to the actions of an agent. We use algorithmic methods, newly designed by us for this work, to extract the (not necessarily group) algebras of the actions of an agent in example worlds with transformations that cannot be fully described by SBDRL. We decided to use worlds that exhibit features commonly found in reinforcement learning scenarios because representation learning methods have been shown to improve the learning efficiency, robustness, and generalisation in such contexts. Finally, we use category theory to generalise core ideas of SBDRs, specifically the equivariance condition and the disentangling definition, to a much larger class of worlds than previously done, with more complex action algebras, including those with irreversible actions. We also propose that category theory appears to be a natural choice for the study of transformations of a world because it focuses on the transformation properties of objects per se, and the perspective that the properties of objects are completely determined by their relationship to other objects is a key result of category theory (the Yoneda Lemma).

The framework we have set out and its results have much room for expansion in future work, including the following: (1) How would we deal with transformations of the world that are not due to the actions of an agent? (2) How would partial observability affect the agent’s representations? (3) What effect would the use of continuous actions have? (4) What algebraic structures would be given by different equivalence relations? (5) Under what conditions can we disentangle reversible and irreversible actions? (6) Could our category theory generalisation of the SBDRL equivariance condition also be used to describe other uses of equivariance conditions in Artificial Intelligence, such as unifying the different equivariance conditions given by \citep{Bronstein2021} through natural transforms? (7) How can our framework be used to develop better representation learning algorithms?

Despite such valid follow-up questions, the main contributions of our proposal stand clear: for agents to interact with their environment efficiently, it is paramount that they learn “good” representations, representations that take advantage of the symmetries of the transformations brought in by their own actions to reduce the task of data processing. For instance, if we learn that a square shows rotation and reflection symmetries, we can, in principle, operate directly in the world under such assumptions, rather than exploring the outcomes of every possible action. Whereas SBDRL and similar approaches have paved the way in formalising such transformations in the case of homogeneous groups, they fall short in representing more flexible symmetric structures. That is, while the study of symmetries is becoming more prominent in representation learning, in this work we have sought to take some key results of mathematical frameworks based on symmetries and generalise them to encompass all transformations of the world due to the actions of an agent: we have presented a formal framework that provides AI developers with the right tools to generate rich representations that show symmetries as groups (as SBDRL does) and beyond. The range of applications of AI models built with this underlying framework and mechanism may greatly extend the state of the art –for instance, reinforcement learning algorithms can directly incorporate knowledge of unknown until now symmetries in so-called world models, minimizing convergence speed; in natural language processing and computer vision tasks, for both LLMs and multimodal foundation models that rely on training vast amounts of data, information and its structured representation in various types of symmetries may be incredibly useful; even in relatively simple classification tasks executed on convolutional layers one could expect an improvement in performance if such convolutions take into account the symmetries formalised in our framework; and the same can be hypothesised about generative AI models such as GANs and transformers that depend on encoding and decoding embedded representations to and from latent spaces. We hope the work presented here will stimulate further work in this direction.

In addition, we also believe that a general framework for exploring the algebra of the transformations of worlds containing an agent, as proposed in this paper, has the potential to be used as a tool in the field of explainable AI since it may enable us to predict which algebraic structures should appear in the agent’s representation of the world at the end of the learning process.

\section*{Declaration of Generative AI and AI-assisted technologies in the writing process}

During the preparation of this work, the authors used ChatGPT (GPT-3.5) in order to improve the language and readability of their work. After using this tool/service, the author(s) reviewed and edited the content as needed and take(s) full responsibility for the content of the publication.

\bibliographystyle{elsarticle-num} 
\bibliography{references}

\begin{thebibliography}{10}
\expandafter\ifx\csname url\endcsname\relax
  \def\url#1{\texttt{#1}}\fi
\expandafter\ifx\csname urlprefix\endcsname\relax\def\urlprefix{URL }\fi
\expandafter\ifx\csname href\endcsname\relax
  \def\href#1#2{#2} \def\path#1{#1}\fi

\bibitem{Higgins2018}
I.~Higgins, D.~Amos, D.~Pfau, S.~Racaniere, L.~Matthey, D.~Rezende, A.~Lerchner, Towards a definition of disentangled representations, arXiv preprint arXiv:1812.02230 (2018).

\bibitem{Amodei2018}
D.~Amodei, D.~Hernandez, G.~Sastry, J.~Clark, G.~Brockman, I.~Sutskever, \href{openai.com/blog/ai-and-compute/}{{AI and Compute}} (2018).
\newline\urlprefix\url{openai.com/blog/ai-and-compute/}

\bibitem{thompson2020computational}
N.~C. Thompson, K.~Greenewald, K.~Lee, G.~F. Manso, The computational limits of deep learning, arXiv preprint arXiv:2007.05558 (2020).

\bibitem{flesch2022orthogonal}
T.~Flesch, K.~Juechems, T.~Dumbalska, A.~Saxe, C.~Summerfield, Orthogonal representations for robust context-dependent task performance in brains and neural networks, Neuron 110~(7) (2022) 1258--1270.

\bibitem{Bernardi2020}
S.~Bernardi, M.~K. Benna, M.~Rigotti, J.~Munuera, S.~Fusi, C.~D. Salzman, The geometry of abstraction in the hippocampus and prefrontal cortex, Cell 183~(4) (2020) 954--967.

\bibitem{ito2022compositional}
T.~Ito, T.~Klinger, D.~Schultz, J.~Murray, M.~Cole, M.~Rigotti, Compositional generalization through abstract representations in human and artificial neural networks, Advances in Neural Information Processing Systems 35 (2022) 32225--32239.

\bibitem{momennejad2020learning}
I.~Momennejad, Learning structures: predictive representations, replay, and generalization, Current Opinion in Behavioral Sciences 32 (2020) 155--166.

\bibitem{lehnert2020reward}
L.~Lehnert, M.~L. Littman, M.~J. Frank, Reward-predictive representations generalize across tasks in reinforcement learning, PLoS computational biology 16~(10) (2020) e1008317.

\bibitem{alonso2013associative}
E.~Alonso, E.~Mondrag{\'o}n, Associative reinforcement learning-a proposal to build truly adaptive agents and multi-agent systems, in: International Conference on Agents and Artificial Intelligence, Vol.~2, SCITEPRESS, 2013, pp. 141--146.

\bibitem{kokkola2019double}
N.~H. Kokkola, E.~Mondrag{\'o}n, E.~Alonso, A double error dynamic asymptote model of associative learning., Psychological review 126~(4) (2019) 506.

\bibitem{Niv2019}
Y.~Niv, Learning task-state representations, Nature Neuroscience 22 (2019) 1544--1553.
\newblock \href {https://doi.org/10.1038/s41593-019-0470-8} {\path{doi:10.1038/s41593-019-0470-8}}.

\bibitem{mack2020ventromedial}
M.~L. Mack, A.~R. Preston, B.~C. Love, Ventromedial prefrontal cortex compression during concept learning, Nature communications 11~(1) (2020) 46.

\bibitem{jha2023extracting}
A.~Jha, J.~C. Peterson, T.~L. Griffiths, Extracting low-dimensional psychological representations from convolutional neural networks, Cognitive science 47~(1) (2023) e13226.

\bibitem{op2001inferotemporal}
H.~Op~de Beeck, J.~Wagemans, R.~Vogels, Inferotemporal neurons represent low-dimensional configurations of parameterized shapes, Nature neuroscience 4~(12) (2001) 1244--1252.

\bibitem{shepard1987toward}
R.~N. Shepard, Toward a universal law of generalization for psychological science, Science 237~(4820) (1987) 1317--1323.

\bibitem{edelman1997learning}
S.~Edelman, N.~Intrator, Learning as extraction of low-dimensional representations, in: Psychology of learning and motivation, Vol.~36, Elsevier, 1997, pp. 353--380.

\bibitem{sutton2018reinforcement}
R.~S. Sutton, A.~G. Barto, Reinforcement learning: An introduction, MIT press, 2018.

\bibitem{li2017deep}
Y.~Li, Deep reinforcement learning: An overview, arXiv preprint arXiv:1701.07274 (2017).

\bibitem{arulkumaran2017deep}
K.~Arulkumaran, M.~P. Deisenroth, M.~Brundage, A.~A. Bharath, Deep reinforcement learning: A brief survey, IEEE Signal Processing Magazine 34~(6) (2017) 26--38.

\bibitem{nian2020review}
R.~Nian, J.~Liu, B.~Huang, A review on reinforcement learning: Introduction and applications in industrial process control, Computers \& Chemical Engineering 139 (2020) 106886.

\bibitem{Higgins2022}
I.~Higgins, S.~Racanière, D.~Rezende, Symmetry-based representations for artificial and biological general intelligence (2022).
\newblock \href {http://arxiv.org/abs/2203.09250} {\path{arXiv:2203.09250}}.

\bibitem{LeCun1995}
Y.~LeCun, Y.~Bengio, et~al., Convolutional networks for images, speech, and time series, The handbook of brain theory and neural networks 3361~(10) (1995) 1995.

\bibitem{Dai2021}
Z.~Dai, H.~Liu, Q.~V. Le, M.~Tan, Coatnet: Marrying convolution and attention for all data sizes, Advances in neural information processing systems 34 (2021) 3965--3977.

\bibitem{Battaglia2018}
P.~W. Battaglia, J.~B. Hamrick, V.~Bapst, A.~Sanchez-Gonzalez, V.~Zambaldi, M.~Malinowski, A.~Tacchetti, D.~Raposo, A.~Santoro, R.~Faulkner, et~al., Relational inductive biases, deep learning, and graph networks, arXiv preprint arXiv:1806.01261 (2018).

\bibitem{Bronstein2021}
M.~M. Bronstein, J.~Bruna, T.~Cohen, P.~Veli{\v{c}}kovi{\'c}, Geometric deep learning: Grids, groups, graphs, geodesics, and gauges, arXiv preprint arXiv:2104.13478 (2021).

\bibitem{Baek2021}
M.~Baek, F.~DiMaio, I.~Anishchenko, J.~Dauparas, S.~Ovchinnikov, G.~R. Lee, J.~Wang, Q.~Cong, L.~N. Kinch, R.~D. Schaeffer, et~al., Accurate prediction of protein structures and interactions using a three-track neural network, Science 373~(6557) (2021) 871--876.

\bibitem{Batzner2022}
S.~Batzner, A.~Musaelian, L.~Sun, M.~Geiger, J.~P. Mailoa, M.~Kornbluth, N.~Molinari, T.~E. Smidt, B.~Kozinsky, E (3)-equivariant graph neural networks for data-efficient and accurate interatomic potentials, Nature communications 13~(1) (2022) 2453.

\bibitem{Chen2020}
T.~Chen, S.~Kornblith, M.~Norouzi, G.~Hinton, A simple framework for contrastive learning of visual representations, in: International conference on machine learning, PMLR, 2020, pp. 1597--1607.

\bibitem{Kohler2020}
J.~K{\"o}hler, L.~Klein, F.~No{\'e}, Equivariant flows: exact likelihood generative learning for symmetric densities, in: International conference on machine learning, PMLR, 2020, pp. 5361--5370.

\bibitem{burgess2018understanding}
C.~P. Burgess, I.~Higgins, A.~Pal, L.~Matthey, N.~Watters, G.~Desjardins, A.~Lerchner, Understanding disentangling in $\beta$-vae, arXiv preprint arXiv:1804.03599 (2018).

\bibitem{Jaderberg2016}
M.~Jaderberg, V.~Mnih, W.~M. Czarnecki, T.~Schaul, J.~Z. Leibo, D.~Silver, K.~Kavukcuoglu, Reinforcement learning with unsupervised auxiliary tasks, arXiv preprint arXiv:1611.05397 (2016).

\bibitem{Krizhevsky2012}
A.~Krizhevsky, I.~Sutskever, G.~E. Hinton, Imagenet classification with deep convolutional neural networks, Advances in neural information processing systems 25 (2012).

\bibitem{Hu2018}
J.~Hu, L.~Shen, G.~Sun, Squeeze-and-excitation networks, in: Proceedings of the IEEE conference on computer vision and pattern recognition, 2018, pp. 7132--7141.

\bibitem{Silver2016}
D.~Silver, A.~Huang, C.~J. Maddison, A.~Guez, L.~Sifre, G.~Van Den~Driessche, J.~Schrittwieser, I.~Antonoglou, V.~Panneershelvam, M.~Lanctot, et~al., Mastering the game of go with deep neural networks and tree search, nature 529~(7587) (2016) 484--489.

\bibitem{Espeholt2018}
L.~Espeholt, H.~Soyer, R.~Munos, K.~Simonyan, V.~Mnih, T.~Ward, Y.~Doron, V.~Firoiu, T.~Harley, I.~Dunning, et~al., Impala: Scalable distributed deep-rl with importance weighted actor-learner architectures, in: International conference on machine learning, PMLR, 2018, pp. 1407--1416.

\bibitem{marcus2018deep}
G.~Marcus, Deep learning: A critical appraisal, arXiv preprint arXiv:1801.00631 (2018).

\bibitem{Cobbe2019}
K.~Cobbe, O.~Klimov, C.~Hesse, T.~Kim, J.~Schulman, Quantifying generalization in reinforcement learning, in: International Conference on Machine Learning, PMLR, 2019, pp. 1282--1289.

\bibitem{Bengio2013}
Y.~Bengio, A.~Courville, P.~Vincent, {Representation learning: A review and new perspectives}, IEEE Transactions on Pattern Analysis and Machine Intelligence 35~(8) (2014) 1798--1828.
\newblock \href {http://arxiv.org/abs/1206.5538} {\path{arXiv:1206.5538}}, \href {https://doi.org/10.1109/TPAMI.2013.50} {\path{doi:10.1109/TPAMI.2013.50}}.

\bibitem{raffin2019decoupling}
A.~Raffin, A.~Hill, R.~Traor{\'e}, T.~Lesort, N.~D{\'\i}az-Rodr{\'\i}guez, D.~Filliat, Decoupling feature extraction from policy learning: assessing benefits of state representation learning in goal based robotics, arXiv preprint arXiv:1901.08651 (2019).

\bibitem{wang2022disentangled}
X.~Wang, H.~Chen, S.~Tang, Z.~Wu, W.~Zhu, Disentangled representation learning, arXiv preprint arXiv:2211.11695 (2022).

\bibitem{Park2022learning}
J.~Y. Park, O.~Biza, L.~Zhao, J.~W. van~de Meent, R.~Walters, Learning symmetric embeddings for equivariant world models, arXiv preprint arXiv:2204.11371 (2022).

\bibitem{Quessard2020learning}
R.~Quessard, T.~Barrett, W.~Clements, Learning disentangled representations and group structure of dynamical environments, Advances in Neural Information Processing Systems 33 (2020) 19727--19737.

\bibitem{Miyato2022unsupervised}
T.~Miyato, M.~Koyama, K.~Fukumizu, Unsupervised learning of equivariant structure from sequences, Advances in Neural Information Processing Systems 35 (2022) 768--781.

\bibitem{Wang2022surprising}
D.~Wang, J.~Y. Park, N.~Sortur, L.~L. Wong, R.~Walters, R.~Platt, The surprising effectiveness of equivariant models in domains with latent symmetry, arXiv preprint arXiv:2211.09231 (2022).

\bibitem{Keurti2023homomorphism}
H.~Keurti, H.-R. Pan, M.~Besserve, B.~F. Grewe, B.~Sch{\"o}lkopf, Homomorphism autoencoder--learning group structured representations from observed transitions, in: International Conference on Machine Learning, PMLR, 2023, pp. 16190--16215.

\bibitem{Zhu2021commutative}
X.~Zhu, C.~Xu, D.~Tao, Commutative lie group vae for disentanglement learning, in: International Conference on Machine Learning, PMLR, 2021, pp. 12924--12934.

\bibitem{Wang2021self}
T.~Wang, Z.~Yue, J.~Huang, Q.~Sun, H.~Zhang, Self-supervised learning disentangled group representation as feature, Advances in Neural Information Processing Systems 34 (2021) 18225--18240.

\bibitem{Pfau2020disentangling}
D.~Pfau, I.~Higgins, A.~Botev, S.~Racani{\`e}re, Disentangling by subspace diffusion, Advances in Neural Information Processing Systems 33 (2020) 17403--17415.

\bibitem{Mercatali2022}
G.~Mercatali, A.~Freitas, V.~Garg, Symmetry-induced disentanglement on graphs, in: S.~Koyejo, S.~Mohamed, A.~Agarwal, D.~Belgrave, K.~Cho, A.~Oh (Eds.), Advances in Neural Information Processing Systems, Vol.~35, Curran Associates, Inc., 2022, pp. 31497--31511.

\bibitem{Marchetti2023}
G.~L. Marchetti, G.~Tegn\'er, A.~Varava, D.~Kragic, Equivariant representation learning via class-pose decomposition, in: F.~Ruiz, J.~Dy, J.-W. van~de Meent (Eds.), Proceedings of The 26th International Conference on Artificial Intelligence and Statistics, Vol. 206 of Proceedings of Machine Learning Research, PMLR, 2023, pp. 4745--4756.

\bibitem{caselles2019symmetry}
H.~Caselles-Dupr{\'e}, M.~Garcia~Ortiz, D.~Filliat, Symmetry-based disentangled representation learning requires interaction with environments, Advances in Neural Information Processing Systems 32 (2019).

\bibitem{caselles2020sensory}
H.~Caselles-Dupr{\'e}, M.~Garcia-Ortiz, D.~Filliat, On the sensory commutativity of action sequences for embodied agents, arXiv preprint arXiv:2002.05630 (2020).

\bibitem{riehl2017category}
E.~Riehl, Category theory in context, Courier Dover Publications, 2017.

\bibitem{barr1990category}
M.~Barr, C.~Wells, Category theory for computing science, Vol.~1, Prentice Hall New York, 1990.

\bibitem{Fong_Spivak_2019}
B.~Fong, D.~I. Spivak, An Invitation to Applied Category Theory: Seven Sketches in Compositionality, Cambridge University Press, 2019.

\bibitem{Milewski-2018}
B.~Milewski, \href{https://unglueit-files.s3.amazonaws.com/ebf/e90890f0a6ea420c9825657d6f3a851d.pdf}{Category Theory for Programmers}, Blurb, 2018.
\newline\urlprefix\url{https://unglueit-files.s3.amazonaws.com/ebf/e90890f0a6ea420c9825657d6f3a851d.pdf}

\bibitem{gavranović2024positioncategoricaldeeplearning}
B.~Gavranović, P.~Lessard, A.~Dudzik, T.~von Glehn, J.~G.~M. Araújo, P.~Veličković, \href{https://arxiv.org/abs/2402.15332}{Position: Categorical deep learning is an algebraic theory of all architectures} (2024).
\newblock \href {http://arxiv.org/abs/2402.15332} {\path{arXiv:2402.15332}}.
\newline\urlprefix\url{https://arxiv.org/abs/2402.15332}

\bibitem{gavranović2024fundamentalcomponentsdeeplearning}
B.~Gavranović, \href{https://arxiv.org/abs/2403.13001}{Fundamental components of deep learning: A category-theoretic approach} (2024).
\newblock \href {http://arxiv.org/abs/2403.13001} {\path{arXiv:2403.13001}}.
\newline\urlprefix\url{https://arxiv.org/abs/2403.13001}

\bibitem{Spivak2014-SPICTF}
D.~I. Spivak, Category Theory for the Sciences, MIT Press, Cambridge, Massachusetts, 2014.

\end{thebibliography}
\section*{Appendix}\label{sec:Appendix}
\begin{algorithm}[H]
\caption{AddElementToStateCayleyTable: Fill state Cayley table row and column for element $a$.}\label{alg:AddElementToStateCayleyTable}
\begin{algorithmic}[1]
    \Require $state\_cayley\_table$, $w$: initial world state, $a$.
    \State Add new row and new column labelled by $a$ to $state\_cayley\_table$.
     \For{$column\_label$ in $state\_cayley\_table$}
        \State $state\_cayley\_table[a][column\_label] \gets column\_label * (a * w)$.
    \EndFor

    \For{$row\_label$ in $state\_cayley\_table$}
        \State $state\_cayley\_table[row\_label][a] \gets a * (row\_label * w)$.
    \EndFor
\end{algorithmic}
\end{algorithm}

\begin{algorithm}[H]
\caption{SearchForNewCandidates: Search for new candidate elements in state Cayley table.}\label{alg:SearchForNewCandidates}
\begin{algorithmic}[1]
    \Require $state\_cayley\_table$, $w$: initial world state, $candidate\_cayley\_table\_elements$.
    \For{$row\_label$ in $state\_cayley\_table$}
        \For{$column\_label$ in $state\_cayley\_table$}
            \State $a_{C} \gets column\_label \circ row\_label$.
            \State $equivalents\_found \gets$ SearchForEquivalents($state\_cayley\_table$, $w$, $a_{C}$). \Comment{See Algorithm \ref{alg:SearchForEquivalents}.}
            \If{$len(equivalents\_found) \neq 0$}
                \State Add $a_{C}$ to relevant equivalence class.
            \Else
                \State Add $a_{C}$ to $candidate\_cayley\_table\_elements$.
            \EndIf
        \EndFor
    \EndFor
    \State \textbf{return} $candidate\_cayley\_table\_elements$.
\end{algorithmic}
\end{algorithm}

\begin{algorithm}[H]
\caption{SearchForEquivalents: Search for elements in Cayley table that are equivalent to $a$.}\label{alg:SearchForEquivalents}
\begin{algorithmic}[1]
    \Require $state\_cayley\_table$, $w$: initial world state, $a$.
    \State $equivalents\_found \gets$ empty list.

    \State $a\_row \gets$ empty list.\Comment{Generate state Cayley row for $a$.}
    \For{$column\_label$ in $state\_cayley\_table$}
        \If{$column\_label ==  a$}
            \State Continue to the next iteration of the for loop.
        \EndIf
        \State Append $column\_label * (a * w)$ to $a\_row$. 
    \EndFor
    
    \State $a\_column \gets$ empty list.\Comment{Generate state Cayley column for $a$.}
    \For{$row\_label$ in $state\_cayley\_table$}
        \If{$row\_label ==  a$}
            \State Continue to the next iteration of the for loop.
        \EndIf
        \State Append $a * (row\_label * w)$ to $a\_column$. 
    \EndFor

    \For{$row\_label$ in $state\_cayley\_table$}
        \If{$(a\_row\text{, }a\_column)$ == $(row\_label\text{ row, } row\_label\text{ column})$}
            \State Append $row\_label$ to $equivalents\_found$.
        \EndIf
    \EndFor
    \State \textbf{return} $equivalents\_found$.
\end{algorithmic}
\end{algorithm}

\begin{algorithm}[H]
\caption{SearchForBrokenEquivalenceClasses: Find equivalence classes that are broken by $a_{C}$.}\label{alg:SearchForBrokenEquivalenceClasses}
\begin{algorithmic}[1]
    \Require $state\_cayley\_table$, $w$: initial world state, $a_{C}$.
    \For{$row\_label$ in $state\_cayley\_table$}
        \State $ec\_label\_outcome \gets row\_label * (a_{C} * w)$.
        \For{$ec\_element$ in equivalence class labelled by $row\_label$}
            \If{$ec\_label\_outcome \neq ec\_element * (a_{C} * w)$}
                \State Create a new equivalence class labelled by $ec\_element$.
                \State Remove $ec\_element$ from the equivalence class labelled by $row\_label$.
            \EndIf
        \EndFor
    \EndFor
    \State \textbf{return:} New equivalence classes.
\end{algorithmic}
\end{algorithm}

\end{document}